\pdfoutput=1
\documentclass[runningheads,a4paper]{llncs}
\usepackage{amsmath,amsfonts,graphicx,amssymb,multicol,psfrag,enumerate}

\newcommand{\norm}[1]{\left|#1\right|}

\newtheorem{thm}{Theorem}

\let\oldsqrt\sqrt
\def\sqrt{\mathpalette\DHLhksqrt}
\def\DHLhksqrt#1#2{%
\setbox0=\hbox{$#1\oldsqrt{#2\,}$}\dimen0=\ht0
\advance\dimen0-0.2\ht0
\setbox2=\hbox{\vrule height\ht0 depth -\dimen0}%
{\box0\lower0.4pt\box2}}

\begin{document}
\frontmatter          
\pagestyle{headings} 
\mainmatter   
\title{Quantum Control Experiments as a Testbed for Evolutionary Multi-Objective Algorithms}
\titlerunning{Quantum Control Experiments as a Testbed for EMOA}
\author{Ofer M. Shir\inst{1} \and Jonathan Roslund\inst{1} \and Zaki Leghtas\inst{2} \and Herschel Rabitz\inst{1}}
\authorrunning{Ofer M. Shir et al.}  
\institute{Department of Chemistry, Princeton University\\Princeton NJ 08544, USA
\medskip \and
Ecole des Mines ParisTech\\Paris 75006, France\\ \medskip
\email{oshir@Princeton.EDU}
}
\maketitle

\begin{abstract}
Experimental multi-objective Quantum Control is an emerging topic within the
broad physics and chemistry applications domain of controlling quantum
phenomena. This realm offers cutting edge ultrafast laser laboratory
applications, which pose multiple objectives, noise, and possibly constraints
on the high-dimensional search. In this study we introduce the topic of
Multi-Observable Quantum Control (MOQC), and consider specific systems to be
Pareto optimized subject to uncertainty, either \emph{experimentally} or by
means of simulated systems. The latter include a family of mathematical
test-functions with a practical link to MOQC experiments, which are introduced
here for the first time. We investigate the behavior of the multi-objective
version of the Covariance Matrix Adaptation Evolution Strategy (MO-CMA-ES) and
assess its performance on computer simulations as well as on laboratory
closed-loop experiments. Overall, we propose a comprehensive study on
\emph{experimental} evolutionary Pareto optimization in high-dimensional
continuous domains, draw some \emph{practical} conclusions concerning the
impact of fitness disturbance on algorithmic behavior, and raise several
\emph{theoretical} issues in the broad evolutionary multi-objective context.
\end{abstract}

\section{Introduction}
Quantum Control (QC) \cite{Hersch93,Gerber07}, sometimes referred
to as Optimal Control or Coherent Control, aims at altering the course of
quantum dynamics phenomena for specific target realizations. There are two main
threads within QC, \emph{theoretical} and \emph{experimental} control, as
typically encountered in physics. Interest in the subject has rapidly increased
during the past 10 years, in parallel with the technological developments of
ultrafast laser pulse shaping capabilities \cite{Weiner00} that made it
possible to bring the dream into experimental fruition.

Quantum Control Theory (QCT) \cite{Hersch88} aims at manipulating the quantum
dynamics of a \emph{simulated system} by means of an external control field,
which typically corresponds to a temporal electromagnetic field arising from a
laser source. 
Quantum Control Experiments (QCE) \cite{Hersch92} consider the realization of
QC in the laboratory, generally executed by applying evolutionary
learning-loops for altering the course of quantum dynamics phenomena. Here, the
yield, or success-rate, is assessed by a physical measurement. The nature of
the optimization is fundamentally different than in QCT, due to practical
laboratory constraints:
limited bandwidth, limited fluence, control resolution, proper control basis, etc. \\
The optimization of QC systems in the laboratory typically poses many algorithmic challenges, such as operating with high-dimensionality, noise,
control constraints, and most importantly in this context, a {\bf potentially large number of simultaneous objectives}.
Attractive features of QCE are the extremely short duration and low cost of an experiment, in comparison to other real-world experimental systems:
the duration of a typical QC measurement is 1msec, allowing a well-averaged single experiment to be recorded in the order of a single second.

{\bf Evolutionary Algorithms} (EAs) \cite{Baeck-book} are the most commonly
employed routines for optimization of QCE systems. This can mostly be
attributed to their high success-rate in addressing the aforementioned
challenges, as reported also in other domains of \emph{experimental}
many-parameter systems (see, e.g., \cite{Schwefel70}). In particular, they
efficiently treat noisy problems, likely due to the employment of large
populations as well as to the fact that they do not require any explicit
gradient determination. Furthermore, EAs possess several features which are
very effective in solving multi-objective (MO) problems, such as being
population-based algorithms, having diversity generation and preservation
mechanisms, etc. Evolutionary Multi-Objective Algorithms (EMOA) (see, e.g.,
\cite{Deb01,CLV07,KCD08}) constitute popular Pareto optimizers that have been
highly successful in treating MO problems.

The list of successful quantum systems controlled in the laboratory by means of
EAs in physics and chemistry is growing rapidly \cite{Gerber07}, but the vast
majority address \emph{de facto} single-objective optimization problems. The
topic of multi-objective QC, also referred to as Multi-Observable Quantum
Control (MOQC), considers multiple distinct physical observables, referring to
mutually competing physical processes. One scenario is a single type of quantum
system, where the competition may be driven by ratios of controlled ionization
or fragmentation of the same molecule \cite{Weber}, versus other scenarios
involving several independent quantum systems, e.g., fluorescence signals in
Optimal Dynamic Discrimination (ODD) of similar molecules
\cite{MO-JPWolf,roth:253001}. MOQC has been addressed in various experimental
systems, predominantly by means of \emph{tailored single-objective scalar
functions} (see, e.g., \cite{Bartelt}). Treating MOQC as a Pareto optimization
problem has been reported only recently, and there is currently a limited
number of studies on this topic: see \cite{RajParetoQC} for QCT and
\cite{MO-JPWolf} for QCE. While the former constituted the first theoretical
study of Pareto fronts in QC, even without involving a MO algorithmic approach,
the latter study is the first reported experimental QC work by means of an
EMOA, namely the NSGA-II \cite{Deb01}.

This study considers several MOQC systems, both \emph{experimental} systems in
the laboratory as well as simulated systems subject to noisy environments. This
work aims to present a pioneering study on {\bf experimental Pareto
optimization in high-dimensional continuous domains} (at least $n=80$ decision
parameters).
Following the successful application of the Covariance Matrix Adaptation
Evolution Strategy (CMA-ES) \cite{hansencmamultimodal} to single-objective QC
systems \cite{QCE_GECCO08,BartelsCMA}, the current study focuses on the
multi-objective version of the CMA-ES (referred to in our notation as MO-CMA)
\cite{CMA-MO} as the algorithmic tool. We investigate its performance upon
treating optimization tasks of both noisy model landscapes (e.g., Multi-Sphere)
as well as real-world MOQC systems.

The manuscript is organized as follows.
Section \ref{sec:EMOA-Noise} will provide some background on the study of EMOA under noise, and outline the specific
characteristics of QCE systems in this context. This will be followed by the description of our algorithmic scheme in
Section \ref{sec:algorithm}, where we shall also discuss the topic of single-parent elitist ES behavior in the presence of noise.
Section \ref{sec:physics} will introduce the systems under study.
We will report on our practical observations in Section \ref{sec:experiments}, and conclude in Section \ref{sec:discussion}.

\section{Uncertain Environments (Noise)}
\label{sec:EMOA-Noise} The presence of uncertainties in environments subject to
optimization by EAs has been widely studied in recent years. The traditional
classes of investigated uncertainties typically include noisy objective
functions \cite{Arnold}, approximation error in the objective function
\cite{OZL06}, the search for robust solutions \cite{BS07}, and dynamic
environments \cite{Branke}. Optimization subject to noisy environments is
typically defined within the topic of \emph{Robustness}. While the research on
single-objective EAs under uncertain environments in general, and under noisy
objective functions in particular, has been widely studied (see, e.g.,
\cite{Arnold,BS06}), there is a limited number of reported EMOA studies to
date. The vast majority of the existing studies consider the scenario of
fitness functions subject to noise, and propose techniques to efficiently
handle this particular uncertainty. Such studies typically make the assumption
that the fitness values are subject to additive Gaussian noise, denoted by
$\mathcal{N}$, with zero mean and finite variance,
\begin{equation}
\label{eq:fnoise}
 \tilde{f}_{i}\left(\vec{x}\right) = f_{i}\left(\vec{x}\right) + \mathcal{N}\left(0,\epsilon_{f}^2\right),
\end{equation}
where the \emph{perceived} $i^{th}$ fitness is $\tilde{f}_{i}$ and the
\emph{ideal} fitness is $f_{i}$. The variance of the normal disturbance,
$\epsilon_f^2$, is referred to as the \emph{noise strength}, and is assumed to
either remain fixed during a run (i.e., additive noise), or to be a
multiplicative factor of the fitness measurement, i.e., $\epsilon_{f_i}^2 \sim
f_i$. Also, the so-called \emph{degree of overvaluation} usually refers to the
difference between the perceived fitness and the ideal fitness:
$\tilde{f}_{i}-f_{i}$. Other types of noisy models, such as consideration of
uncertainty in the decision parameters to be optimized, have received scarce
attention \cite{BOS04,Deb_RobustMO}. This type of noise, which corresponds to
the precision of the optimized design and may represent manufacturing error, is
of {\bf particular interest to this study}. The fitness values are then modeled
as
\begin{equation}
\label{eq:xnoise}
 \tilde{f}_{i}\left(\vec{x}\right) = f_{i}\left(\vec{x} + \mathcal{N}\left(\vec{0},\epsilon_{x}^2 \mathbf{I} \right) \right).
\end{equation}
Here, since the decision parameters are systematically disturbed, each one of them can be controlled only up to a certain degree of accuracy.
Moreover, the fitness values in this case may be either enhanced or deteriorated,
depending exclusively upon the nature of the objective function and
the manner in which the noise propagates through it.
Thus, the expected fitness \emph{overvaluation} or \emph{undervaluation} may be estimated only if the propagation of the noise can be derived.
We choose to refer here to the difference between the perceived and the ideal fitness values stemming from noisy decision parameters
as the \emph{fitness disturbance}, i.e., $\left|\tilde{f}_{i}-f_{i}\right|$.

Regardless of the differences in the modeling, the system still retains inherent underlying uncertainty,
explicitly revealed by two successive evaluations of the same recorded input variables returning two different sets of output values.

\subsection{EMOA in Noisy Environments: Robustness}
Early EMOA work on treatment of \emph{noisy objective functions} includes the
probabilistic Pareto ranking approach (similar concepts by \cite{Teich-EMO} and
\cite{Hughes-EMO}), which introduces a modified selection criterion accounting
for the stochasticity of the objective function. The concepts of domination
dependent lifetime and re-sampling of archived solutions was introduced by
B\"uche et al.\ in \cite{Buche02}.
Moreover, recent studies (see, e.g., \cite{GohTan07}) proposed noise-handling
features, as additions to existing EMOA, and considered a suite of synthetic
bi-criteria landscapes as a testbed. In a recent study, Bader and Zitzler
\cite{bz2010a} provided an important overview on robustness in multi-objective
optimization. In general terms, multi-objective noise-treatment and
robustness-accounting are carried out by one of the following schemes
\cite{bz2010a}:
\begin{enumerate}
 \item Replacement of the objective function value by a measure reflecting uncertainty, e.g., statistical mean, or signal averaging \cite{Mulvey95}
\item Introduction of an additional robustness criterion to the search \cite{Deb_RobustMO,Egorov2002,JS03}
\item Consideration of a tailored robustness constraint, imposing candidate solutions to satisfy statistical criteria \cite{Deb_RobustMO,Gunawan05}
 \end{enumerate}

In what follows, we refer to two specific studies that are directly linked to our work.
\subsubsection*{Simulated Robustness in Multi-Objective Optimization}
Deb and Gupta\ \cite{Deb_RobustMO}, in a pioneering work, introduced \emph{systematic disturbance}
to decision parameters in Pareto optimization and posed the demand for attaining robust solutions. The study shifted the
focus from searching for global best Pareto fronts to robust Pareto fronts, whose pre-images are solutions that
are robust to variable perturbations.
However, as the authors concluded, the proposed schemes were prone to being impractical in real-world scenarios,
as they increased the total number of evaluations by factors of $\sim$ 50--100.

\subsubsection*{Multi-Objective Experimental Optimization}
The first reported campaign of \emph{experimental} Pareto optimization was
carried out by Knowles and co-workers within biological experimental platforms
(e.g., \cite{KnowlesMOExp07}, and see \cite{KnowlesCampaignOverview} for an
overview). In addition to the successful results on multiple experimental
systems, this campaign led to the subsequent development of the ParEGO, an EMOA
specializing in Pareto optimization subject to an extremely small budget of
measurements (see, e.g., \cite{Knowles2006parego}). This promising search
heuristic was designed for specific demanding \emph{experimental} conditions,
amongst which are
\begin{itemize}
\item low noise levels, i.e., individual experiments practically need not be repeated,
\item locally smooth search landscapes,
\item low-dimensional search spaces (less than $10$ decision parameters).
\end{itemize}

\subsubsection*{Note on Elitism versus Robustness}
It has been pointed out in previous studies that \emph{elitist selection} is an essential component for
efficient multi-objective optimization (see, e.g., \cite{MO_VL00,zlb2004a}). A common argument is the need to preserve
the current population's information in the global selection phases of Pareto domination followed by secondary criteria.
Elitism, at the same time, dictates a unique dynamic that when exposed to uncertain environments
has the potential to deteriorate the quality of the run, suffer from systematic overvaluation, and lead to periods of stagnation.
The currently employed EMOA, namely the MO-CMA, employs an elitist strategy as its algorithmic kernel.
Due to its nature, and due to the nature of experimental frameworks,
we shall also explore theoretical studies from the realm of single-objective Evolution Strategies related
to this study, as outlined in Section \ref{sec:algorithm}.

\subsection{QC Systems: Sources of Noise and Uncertainty}
Uncertainty in QCE stems from various sources, and exists at several levels.
We attribute it to three main factors, in decreasing importance, as we shall explain in detail in what follows
(compare to \cite{BS07} as a generic reference):
\begin{enumerate}[(A)]
\item Spectral phase noise: uncertainty concerning the decision (input) parameters;
the error in realizing the prescribed parameters in the experimental setup
\item Observation noise: uncertainty concerning the measurement (output) values, originating from detector noise (also known as Johnson-Nyquist noise)
\item Environmental drift: Systematic slow deviation in the system values over the time span of the entire experiment, e.g., minutes to hours
\end{enumerate}
\begin{figure}
\centering \includegraphics[scale=0.5]{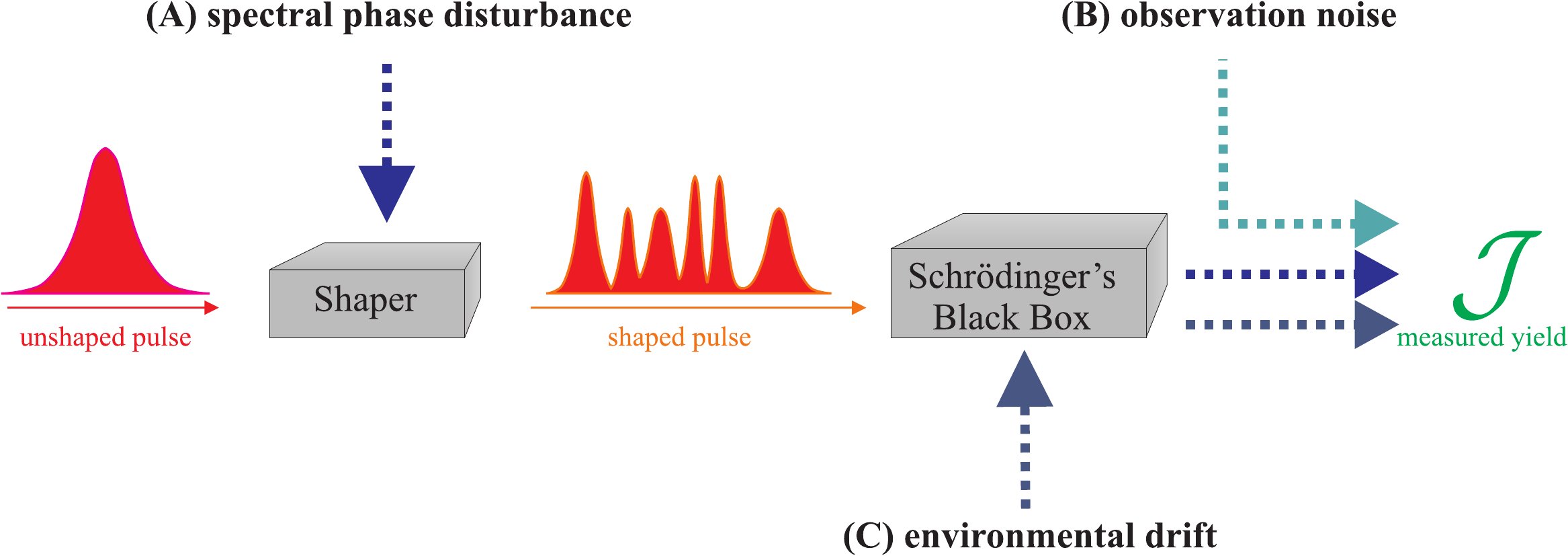}
 \caption{Summary of the three main sources of noise in a typical Quantum Control experiment.
Compare to \cite{BS07} as a generic reference. \label{fig:QC_noisesources}}
\end{figure}
{\bf (A)} The primary component in the current experimental learning loop generating uncertainty with highest impact is the process responsible
for the construction of the laser pulse, which is carried out with a pulse shaper.
Unlike standard modeling in the literature regarding noisy environments, the current framework is
modeled as subject to additive Gaussian noise on the control variables (i.e., the decision parameters to be optimized, or the input),
which propagates typically in a highly nonlinear manner to the measured values (i.e., the objective functions, or the outputs).
More explicitly, the control function with spectral modulation consists of the spectral amplitude $A(\omega)$ and phase $\phi(\omega)$ functions,
which together construct the electric field:
\begin{equation}
\label{eq:efield}
E(t)=\mathbb{R} \left\{\int A(\omega)\exp(i\phi(\omega))
\exp(-i\omega t) \ d\omega \right\}.
\end{equation}
Most QC processes are highly sensitive to the phase, and phase-only shaping is typically sufficient for attaining optimal
control. Our experiments only include phase modulation, where the spectral function $A(\omega)$ is fixed. The latter is
well approximated by a Gaussian and determines the bandwidth, or the minimal pulse duration.
Note that shaping the field with phase-only modulation guarantees conservation of the pulse energy.

The spectral phase $\phi(\omega)$ is defined at $n$ frequencies
$\{\omega_j\}_{j=1}^n$ that are equally distributed across the
bandwidth of the spectrum. These $n$ values,
$\{\phi(\omega_j)\}_{j=1}^n$, correspond to the $n$ pixels of the
pulse shaper and are the decision parameters to be optimized in the
experimental learning loop:
\begin{equation}
\label{eq:phase}
\phi\left(\omega \right) =
\left(\phi(\omega_1),\phi(\omega_2),...,\phi(\omega_n)\right).
\end{equation}
The laser field, as defined in Eq.\ \ref{eq:efield}, completely determines the dynamics of any controlled quantum process,
subject to the associated wavefunction $\psi(t)$, satisfying the Schr\"odinger equation:
\begin{equation}
\label{eq:schroedinger}
\begin{array}{l}
\displaystyle i\frac{\partial \psi}{\partial t}=(H_0+V)\psi(t)\\
\displaystyle V=-\mu E(t)\cos(\omega_0 t)
\end{array}
\end{equation}
where $H_0$ is the field-free Hamiltonian and $\mu$ is the electric dipole moment.
The modeling of noise on the shaper is equivalent to Eq.\ \ref{eq:xnoise},
assuming additive Gaussian noise on each pixel (independent Gaussian sampling):
\begin{equation}
\label{eq:phinoise}
\tilde{\phi}(\omega) = \left(\phi(\omega_1)+\mathcal{N}_1\left(0,\epsilon_{S}^2 \right),\ldots,\phi(\omega_n)+\mathcal{N}_n\left(0,\epsilon_{S}^2 \right)\right),
\end{equation}
where $\tilde{\phi}(\omega)$ and $\phi(\omega)$ are the \emph{perceived} and the \emph{ideal} pixel values, respectively,
and each pixel is subject to a noise level of $\epsilon_{S}^2$; the latter is assumed to remain fixed during the course of the whole experiment.
Since this type of uncertainty stems from physical disturbances -- such as dust or convection currents that are responsible
for variable refraction indices, and therefore can be modeled as some continuous function --
the independently sampled Gaussian disturbance is thus an approximation.
The correlations between disturbances on adjacent pixels may be considered in further studies.

The variations in the input propagates into the output in a highly nonlinear
manner, due to the complex transformations involved in the process (Eqs.\
\ref{eq:efield} and \ref{eq:schroedinger}), and yields non-additive deviations
with an unknown form.

{\bf (B)} Given a quantum observable operator, $\mathcal{O}_i$, and given the
propagated wavefunction $\psi$ solving Eq.\ \ref{eq:schroedinger}, a quantum
observation is then defined as
 $\mathcal{J}_i = \left< \psi  \left|\mathcal{O}_i \right| \psi \right>$ .
The measurement value is assumed to be subject to \emph{observation noise}, corresponding to electronic or thermal fluctuations in the detector (Johnson-Nyquist noise),
which typically possesses very low noise strength $\epsilon_{\mathcal{J}}^2$ and is modeled as additive Gaussian deviations,
equivalent to Eq.\ \ref{eq:fnoise}.\\
The high duty cycle of QC experiments (typically 1kHz) permits increased signal averaging, which reduces the influence of additive noise sources,
such as measurement noise, by virtue of the central limit theorem.
Thus, given $k$ independent, single-shot measurements, the mean and variance of the observation in the presence of
measurement noise, $\tilde{\mathcal{J}}_i$, may be described as follows:
\begin{equation}
\displaystyle \left< \tilde{\mathcal{J}}_i \right> =
\mathcal{J}_i,~~~\textrm{VAR} \left[ \tilde{\mathcal{J}}_i \right] =
\frac{\epsilon_{\mathcal{J}}^2}{k},
\end{equation}
and given sufficient signal averaging, its contribution is effectively removed.
While such signal averaging always increases the precision of the QC measurement, the contribution of non-additive noise
sources, such as the propagation of $\tilde{\phi}(\omega)$ (Eq.\ \ref{eq:phinoise}),
may not be removed, and is of particular interest in this study.

{\bf (C)} The third source of uncertainty, with the least impact, is general system drift which occurs in a time span of the entire experiment
(minutes to hours). The observation is then disturbed by some temporal function $\xi(t)$:
\begin{equation}
\hat{\mathcal{J}}_i(t) = \tilde{\mathcal{J}}_i + \xi(t).
\end{equation}
Fig.\ \ref{fig:QC_noisesources} summarizes the sources of noise in a typical QC experiment.

\section{The Algorithmic Approach: Multi-Objective CMA-ES}
\label{sec:algorithm}
Following the broad success of the Covariance Matrix Adaptation Evolution Strategy (CMA-ES)
in single-objective continuous optimization, a multi-objective version has been released \cite{CMA-MO}.
In short, the CMA is a derandomized ES variant that has been successful in treating correlations among decision parameters by efficiently
learning optimal mutation distributions.
The MO-CMA relies on the elitist $\left(1+\lambda\right)$-CMA kernel \cite{CMAPLUS}
(typically with $\lambda=1$), which had been originally designed for it,
likely due to the aforementioned
studies indicating that \emph{elitism} is essential for efficient multi-objective optimization \cite{MO_VL00,zlb2004a}.
The elitist CMA combines the classical concepts of the $\left(1+1\right)$-ES, and especially the \emph{success probability} and
the \emph{success rule} components (see, e.g., \cite{Baeck-book}),
with the Covariance Matrix Adaptation concept.

Explicitly, the set of evolving individuals comprises $\mu$ search points,
which correspond to $\mu$ independently evolving single-parent CMA mechanisms.
Given the $i^{th}$ search point in generation $g$, $\vec{x}^{(g)}_i$, an offspring is generated
by means of a Gaussian variation:
\begin{equation}
\label{eq:cma_gen}
\vec{x}^{(g+1)}_i\sim\mathcal{N}\left(\vec{x}^{(g)}_i,\sigma^{(g)^{2}}_i\mathbf{C}^{(g)}_i\right).
\end{equation}
The covariance matrices, $\left\{\mathbf{C}^{(g)}_i\right\}_{i=1}^{\mu}$, are
initialized as \emph{unit matrices} and are learned during the course of
evolution, based on cumulative information of successful past mutations. The
step-sizes, $\left\{\sigma^{(g)}_i\right\}_{i=1}^{\mu}$, are updated according
to the so-called \emph{success rule based step-size control}. The set of
parents and offspring undergoes two MO evaluation phases, corresponding to two
selection criteria: the first criterion is Pareto domination ranking, followed
by the hypervolume contribution criterion. Fig.\ \ref{fig:mo-cma-es}
illustrates the operation of the MO-CMA algorithm. For more details we refer
the reader to \cite{CMA-MO}.

\begin{figure}
\centering \includegraphics[scale=0.58]{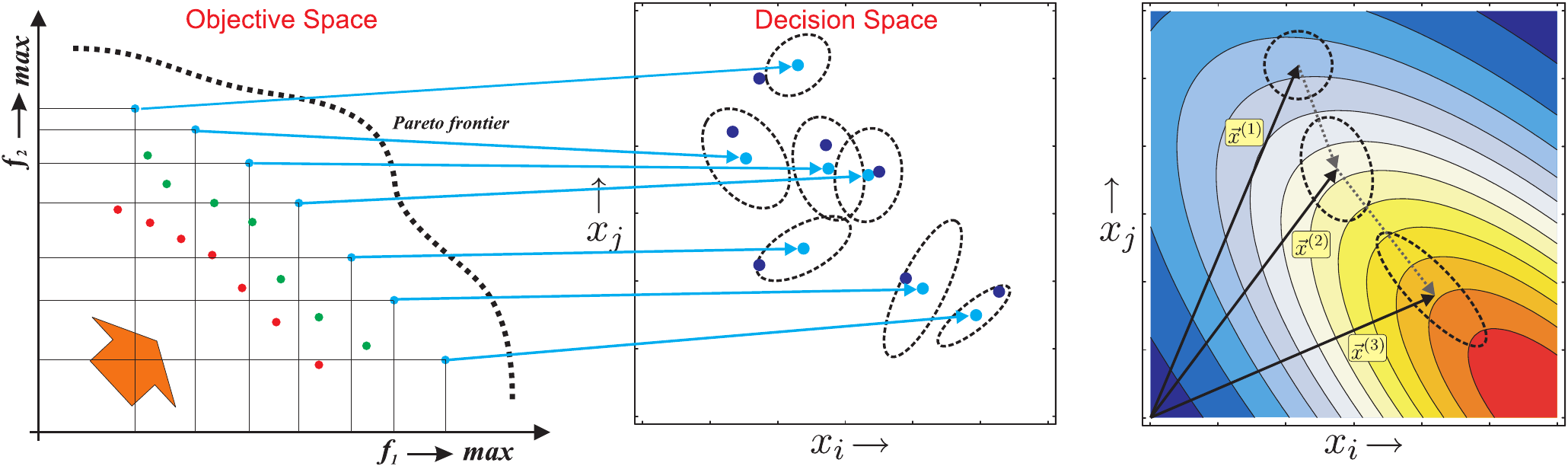}
 \caption{Cartoons illustrating the MO-CMA mechanism: [LEFT] The objective space, where selection is subject to two criteria:
Pareto domination ranking and hypervolume contribution \cite{CMA-MO}. [MIDDLE]
The decision (search) space, where the pre-images of the evolving Pareto front
are depicted, and simultaneously update as independent (1+1)-CMA kernels.
[RIGHT] A solitary CMA kernel evolving in the decision space of an elliptic
{\bf single-objective} model landscape. \label{fig:mo-cma-es}}
\end{figure}

\subsection{Introduction of Noise}
The application of the MO-CMA to MOQC in general, and to the systems under investigation in the current study,
introduces new aspects to Pareto optimization at different levels that have to be addressed.
The current framework differs from previously studied MO noisy systems in two main aspects:
\begin{itemize}
 \item The recorded objective function values (signal measurements) cannot be assumed to follow a specific distribution;
the degree to which the noise on the decision parameters propagates into the objective function values is generally unknown,
and in any case the latter is not additive.
 \item Due to the nature of the MO-CMA learning rules, any manipulation or replacement of archived solutions is not recommended.
This is a common rule of thumb for the family of derandomized ES, which rely on cumulative information gained from previously selected search
points.
\end{itemize}
Furthermore, the introduction of noise to the MO-CMA is expected to raise additional issues:
\begin{itemize}
\item Single-parent strategies experience difficulties in handling noisy landscapes, in comparison to multi-parent strategies:
the application of recombination in the latter case proved highly efficient in
treating excessive noise \cite{ArnoldNoise}.
More specifically, in the context of QC experimental optimization, the
single-objective CMA was observed in \cite{QCE_GECCO08} to fail without
recombination, and to perform extremely well otherwise, as expected from theory
\cite{ArnoldNoise}.
\item Elitist strategies support the survival of parents, and are likely to encounter scenarios in which highly overvaluated perceived
fitness values are kept for long periods, causing stagnation (see, e.g., \cite{Beyer93}).
The issue of \emph{fitness disturbance} is expected to become
a problem for the MO-CMA, should its implementation follow the original algorithm and avoid parental fitness re-evaluation.
\end{itemize}
Arnold and Beyer \cite{ArnoldB02} considered the aforementioned effects and studied theoretically
the local performance of the single-objective $(1+1)$-ES in a noisy environment.
Here are some of the relevant conclusions of that study:
\begin{enumerate}
\item Failure to reevaluate the parental fitness leads to systematic overvaluation.
\item Overvaluation is responsible for the different behavior of the elitist single-parent strategy, in comparison to other strategies,
and may lead to long periods of stagnation.
\item Overvaluation may, nevertheless, be beneficial for the specific homogeneous environment of the quadratic sphere in the limit of infinite dimensions.
\item \label{arg:occasional}
Occasional parental fitness re-evaluation seems to be superior with respect to no re-evaluation at all and to re-evaluation in every generation.
\item Overvaluation has the potential to render useless \emph{success-probability based} step-size mechanisms.
\end{enumerate}
It should be stressed that 
disturbance of objective function values in \emph{experimental optimization} 
typically cannot be tolerated, and is primarily perceived as a source of
deception that deteriorates the reliability of the attained results. Also, the
main focus of the current study is on the attained set of solutions, and on the
ability to reproduce the perceived fitness values as reported in the
algorithm's output. In particular, in the MO context, the research goal is to
investigate the nature of the attained {\bf Pareto optimal set}, in light of
its \emph{a posteriori} re-evaluation.
\subsection{A Proposed Scheme}
\label{sec:proposedscheme}
Given the conclusions concerning the $(1+1)$-ES outlined in the previous section, we would like to propose a \emph{modus operandi} for our
experimental optimization, subject to noise, with the MO-CMA.
In particular, three different empirical scenarios are considered:
\begin{enumerate}
 \item Default MO-CMA ('D')
 \item Parental fitness re-evaluation every generation ('E')
 \item Occasional parental fitness re-evaluation at every \emph{epoch} ('O')
\end{enumerate}
The last scenario aims at achieving a trade-off between low fitness disturbance
during the run (reliability) versus keeping the number of experimental
evaluations to a minimum. It can also be considered as an attempt to
corroborate the theoretical results discussed earlier (see the summary of
\cite{ArnoldB02} in the previous section, and particularly point
\ref{arg:occasional}), upon transferring them to the multi-objective framework.

We set the re-evaluation interval to 10 generations, inspired by a recommended rule of thumb for the evaluation interval of the step-size in the
$(1+1)$-ES (see \cite{Baeck-book} p.\ 84).
\section{Systems under Investigation}
\label{sec:physics}
We present here our selected models for the evaluation of the MO-CMA, which comprise model landscapes, a simulated QC system,
and two QC laboratory experiments.
\subsection{Model Landscapes}
Here we briefly introduce the model landscapes to be Pareto optimized. They
include the basic Multi-Sphere model, which is considered to be an elementary
multi-objective test-case, along with a quantum-oriented model landscape,
referred to as the \emph{Diffraction Grating} problem. The latter, which is
introduced here for the first time as a multi-objective test-problem for the
optimization community, shares many characteristics with QC problems, such as
the nature of the decision parameters and some properties of the objective
function. At the same time, it possesses a quite simple form, requires an
extremely short CPU evaluation time, and offers a complete mathematical
formulation (e.g., the propagation of systematic noise may be analytically
derived). Thus, it as a particularly attractive test-case for this study, and
potentially for other future studies, as it offers a {\bf practical link to
experimental optimization} with a very low computational cost.

The landscapes will be optimized subject to a search space dimensionality of
$n=\left\{10,30,80\right\}$, while we choose to expose the search to noise solely
on the decision parameters, corresponding to Eq.\ \ref{eq:xnoise},
with the following values:
\begin{equation}
\label{eq:noisestrength}
 \epsilon_x^2 = \left\{0.001,0.005,0.01,0.02,0.05 \right\}
\end{equation}

\subsubsection{The Multi-Sphere Model}
We consider the $m$-objective quadratic multi-sphere as our model landscape to be Pareto optimized in an $n$-dimensional search-space (see, e.g., \cite{LRS2001a}):
\begin{equation}
\label{eq:multisphere}
 \vec{f}\left(\vec{x}\right) = \begin{pmatrix}
                                \left( \vec{x} - \vec{c_1} \right)^T \cdot \left( \vec{x} - \vec{c_1} \right) \\
                \left( \vec{x} - \vec{c_2} \right)^T \cdot \left( \vec{x} - \vec{c_2} \right) \\
                \vdots \\
                \left( \vec{x} - \vec{c_m} \right)^T \cdot \left( \vec{x} - \vec{c_m} \right)
                               \end{pmatrix}\longrightarrow \min,
                               ~~~
                               \vec{c_1} = \begin{pmatrix}
              1\\0\\0\\ \vdots \\ 0
             \end{pmatrix},\ldots,
\vec{c_m} = \begin{pmatrix}
              0\\0\\0\\ \vdots \\ 1
             \end{pmatrix}.
\end{equation}

The shape of the Pareto front is convex, and it is explicitly described for $m=2$ as follows (see, e.g., \cite{EmmerichDeutzEMO}):
\begin{equation}
\label{eq:2d-sphere-front}
f_2 = 2 \left(1- \left(\frac{f_1}{2} \right)^{1/2} \right)^2,~~~ f_1 \in \left[0,2\right]
\end{equation}
Upon consideration of {\bf noise} on the decision variables, the mean of the perceived fitness reads
\begin{equation}
\label{eq:spheres_mean}
\displaystyle \left< \tilde{f}_i \left(\vec{x}\right)\right> =  f_i\left(\vec{x}\right)  + n \epsilon_{x}^2,
\end{equation}
and its variance is described as follows (for the derivation see, e.g., \cite{BOS04}):
\begin{equation}
\label{eq:spheres_var}
\displaystyle \textrm{VAR} \left[ \tilde{f}_i \left(\vec{x}\right) \right] = 4\epsilon_{x}^2 \left(f_i\left(\vec{x}\right)+ \frac{n}{2}\epsilon_{x}^2 \right)
\end{equation}
\subsubsection{The Diffraction Grating Problem}
\begin{figure}
 \centering
\includegraphics[scale=0.8]{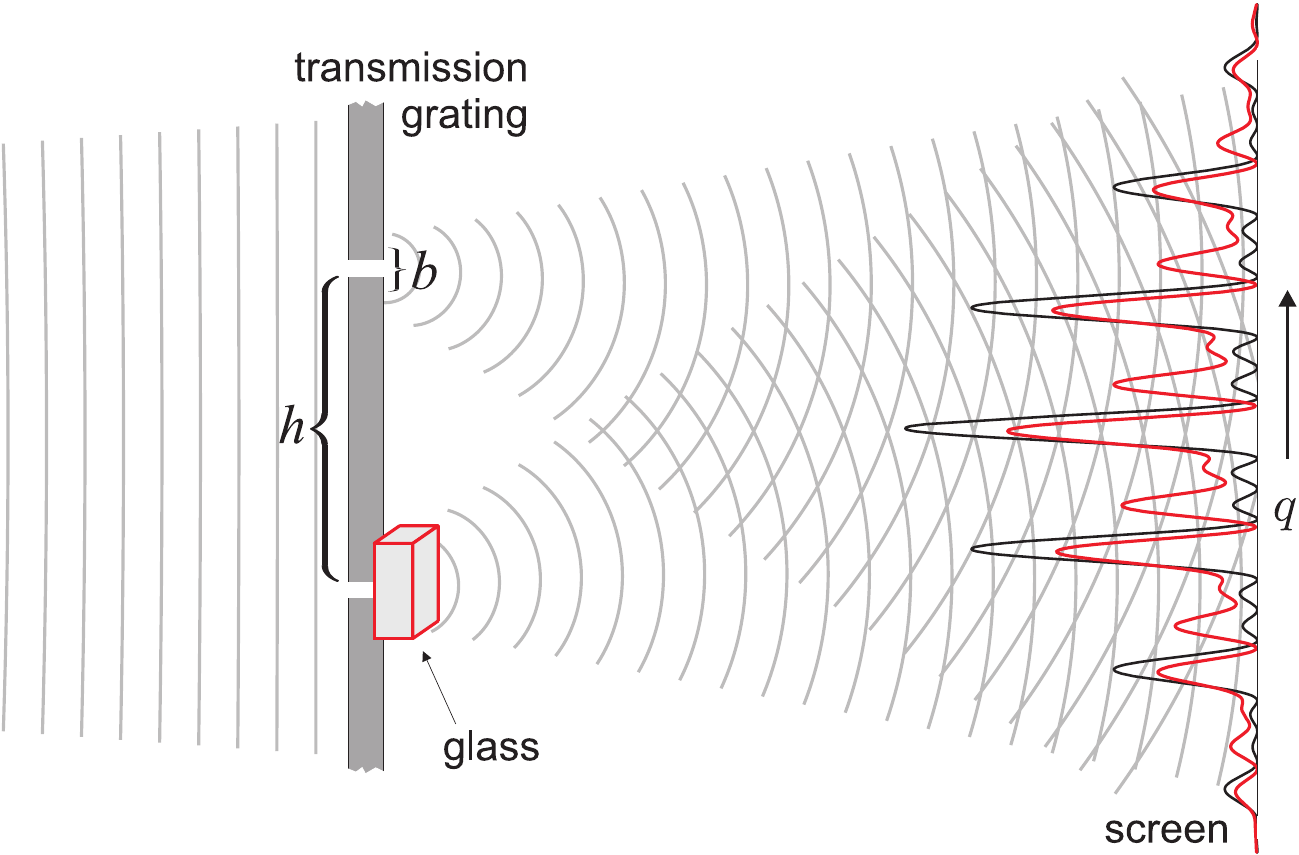} 
\caption{Graphic illustration of the Diffraction Grating problem setup with 2 slits. Incident
light propagates through the slits -- which along with the glass play the role
of a phase function $\vec{\varphi}$ -- and shines on the screen. The intensity
$I_{DG}$ as a function of the position $q$ is then recorded, to become a
position-based objective function. \label{fig:DG_cartoon}}
\end{figure}

The Diffraction Grating family of functions introduces a basic set of optical test-problems
for Pareto optimization, scalable in dimension and subject to a collection of defining parameters
for setting the Pareto front's curvature.

Given a diffraction grating optical setup of $n$ slits, defined by the width of each slit $b$ and the space between adjacent slits $h$,
and given a spatially uniform electromagnetic plane wave illuminating the slits with corresponding phases $\vec{\varphi} \in \left[0,2\pi \right]^n$,
the intensity on a screen point in the Fraunhofer regime (i.e., far field) positioned at $q$ reads:
\begin{equation}
 \label{eq:IntensityDG}
\begin{array}{l}
\medskip
\displaystyle I_{DG}\left(q, \vec{\varphi} \right) = \frac{1}{n^2}\textrm{sinc}^2\left(\frac{q b}{2} \right) \cdot
\left| \sum_{k=0}^{n-1} \exp\left(iq hk\right)\cdot \exp\left(i \varphi_k\right) \right|^2 \\
\displaystyle = \frac{1}{n^2}\textrm{sinc}^2\left(\frac{q b}{2} \right) \cdot \left\{n+ 2\cdot \sum_{k=0}^{n-1}\sum_{\ell > k}^{n-1}
\cos\left[q h \left(\ell-k\right)+\Delta\varphi_{\ell k} \right] \right\},
\end{array}
\end{equation}
where $\vec{\varphi}=\left(\varphi_0,\varphi_1,\ldots,\varphi_{n-1} \right)^T$ and $\Delta\varphi_{\ell k}\equiv \varphi_{\ell}-\varphi_{k}$.
Fig.\ \ref{fig:DG_cartoon} provides an illustration for the Diffraction Grating setup.

Given a set of $m$ competing points on the screen, described by a corresponding \emph{position vector} $\vec{q}\in \mathbb{R}^m$,
the $m$-objective Diffraction Grating problem to be Pareto optimized is defined as follows:
\begin{equation}
\label{eq:fDG}
\vec{f}\left(\vec{q},\vec{\varphi}\right) = \begin{pmatrix}
                \displaystyle I_{DG}\left(q_1, \vec{\varphi} \right)\\
                \displaystyle I_{DG}\left(q_2, \vec{\varphi} \right)\\
                \vdots \\
                \displaystyle I_{DG}\left(q_m, \vec{\varphi} \right)
                               \end{pmatrix} \longrightarrow \max
\end{equation}
The shape of the Pareto front is determined by the positions of the points on the screen, and may furthermore be controlled
by means of the parameters $b$ and $h$. This problem offers a rich variety of complexity levels, and can
easily be extended to many different forms, such as multiple wavelengths, consideration of controllable amplitudes, nonlinear screens,
2-dimensional screens, etc.

Let us consider a setup with
$b=1,~h=4$.
The intensity values on the screen due to optical interferences follow a {\bf period},
$T_q = \frac{2\pi}{h}$,
and it is thus convenient to consider positions in terms of this period $T_q$.
In our calculations we shall consider the maximization of the intensity at position zero,
$q_0=0$, competing with the maximization of the intensity at the following positions:
$q=\left\{0.1\cdot T_q,~~ 0.25\cdot T_q,~~ 0.5\cdot T_q \right\}$.

\begin{figure}
\centering \includegraphics[scale=0.6]{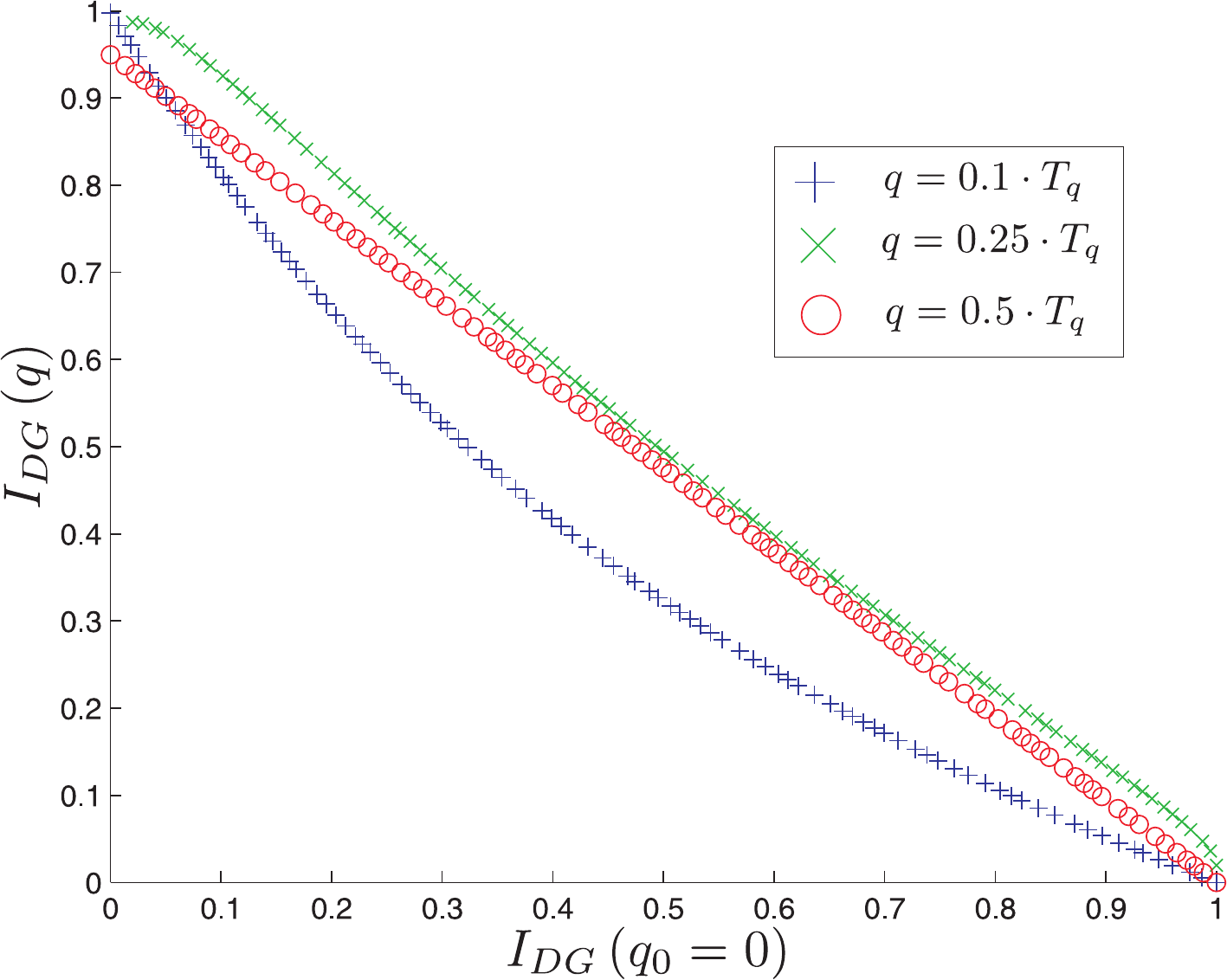}
\caption{Approximate Pareto fronts, attained by the MO-CMA, of the competition
between the intensity at $q_0=0$ to the intensity at each one of the positions
$q=\left\{0.1\cdot T_q,~~ 0.25\cdot T_q,~~ 0.5\cdot T_q \right\}$ -- formalized
as bi-criteria problems following Eq.\ \ref{eq:fDG} with $n=10$ phase points.
Given the fixed optical setup of the problem ($b=1,~h=4$), the positions of the
competing points on the screen dictate the curvature of the Pareto
front.\label{fig:diffraction2Q}}
\end{figure}
For illustration, approximate Pareto fronts (attained by the MO-CMA) of the
competition between the intensity at $q_0=0$ to the intensity at each one of
the positions $q=\left\{0.1\cdot T_q,~~ 0.25\cdot T_q,~~ 0.5\cdot T_q \right\}$
-- formalized as bi-criteria problems following Eq.\ \ref{eq:fDG} with $n=10$
phase points -- are depicted in Fig.\ \ref{fig:diffraction2Q}. In addition, an
approximate Pareto surface, obtained by the steady-state MO-CMA, presenting the
competition between intensities of points positioned at $q_0=0$, $q_1=0.25\cdot
T_q$, and $q_2=0.5\cdot T_q$ -- formalized as a tri-objective problem (Eq.\
\ref{eq:fDG}) with $n=10$ phase points -- is depicted in Fig.\
\ref{fig:diffraction3Q}.

In what follows, this study will focus on the bi-criteria case of $q_1=0.5\cdot
T_q=\frac{\pi}{4}$, i.e.,
\begin{equation}
\label{eq:I05Tq}
\begin{array}{l}
\medskip
\displaystyle f_1= I_{DG}\left( 0,\vec{\varphi} \right) \longrightarrow \max\\
\displaystyle f_2= I_{DG}\left( \frac{\pi}{4},\vec{\varphi} \right)\longrightarrow \max
\end{array}
\end{equation}
{\bf This test-case has a linear Pareto front; see Appendix \ref{app:ParetoProof} for the proof.}
Noise will be modeled here as with a QC phase function (Eq.\ \ref{eq:phinoise}),
i.e., additive Gaussian variations on each phase coordinate:
\begin{equation}
\label{eq:varphinoise}
\tilde{\vec{\varphi}} = \vec{\varphi} + \mathcal{N}\left(\vec{0},\epsilon_{S}^2 \mathbf{I} \right).
\end{equation}
Upon consideration of the noise propagation, the mean and the variance of the perceived fitness can be analytically derived
(see Appendix \ref{app:perceivedcalculations}). The \emph{mean} may be presented in a compact form,
\begin{equation}
\label{eq:DG_mean}
\begin{array}{l}
 \medskip
\displaystyle \left< \tilde{f}_i \left(q_i,\vec{\varphi}\right)\right> = \exp\left(-\epsilon_{S}^2\right) \cdot f_i \left(q_i,\vec{\varphi}\right) +\\
\displaystyle +\textrm{sinc}^2\left(\frac{q_i b}{2} \right) \cdot \left(\frac{1-\exp\left(-\epsilon_{S}^2\right)}{n}\right),
\end{array}
\end{equation}
revealing both additive as well as multiplicative components to the disturbed
objective function values. The \emph{variance}, although possessing a closed
analytical form, cannot be presented in a compact form, but rather in terms of
explicit summation (Eqs.\ \ref{eq:gratingvarconclusion} and
\ref{eq:gratingvarfinal} are given in Appendix
\ref{app:perceivedcalculations}).
\begin{figure}
 \centering \includegraphics[scale=0.55]{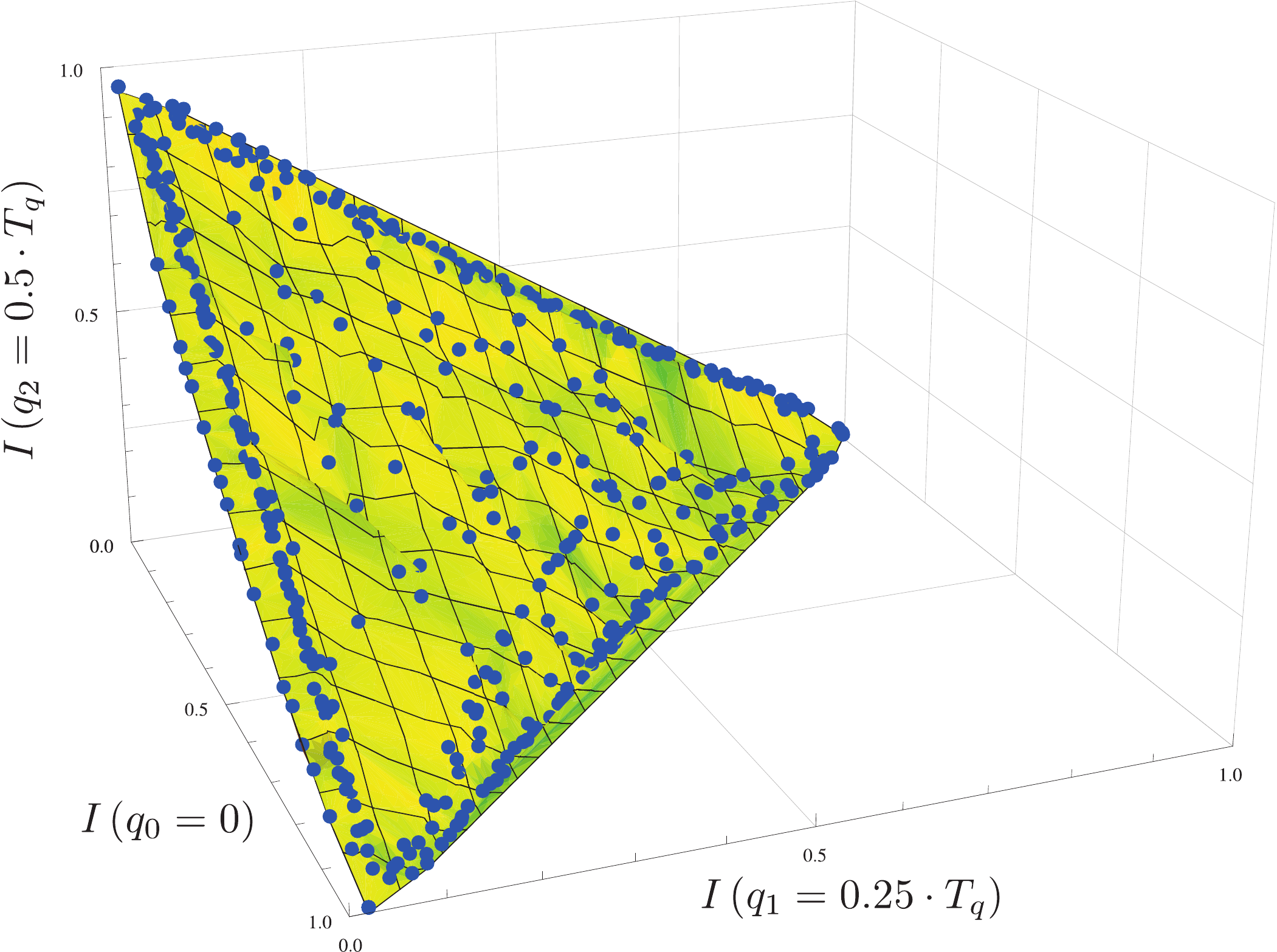}
\caption{The tri-criteria Diffraction Grating problem: approximate Pareto surface, attained by the steady-state MO-CMA,
of the competition between the intensities at $q_0=0$, $q_1=0.25\cdot T_q$, and $q_2=0.5\cdot T_q$ --
following Eq.\ \ref{eq:fDG} with $n=10$ phase points.\label{fig:diffraction3Q}}
\end{figure}
\subsection{Simulated Quantum Control System: Molecular Alignment}
We consider the QC application to \emph{dynamic molecular alignment},
which has been widely investigated in the past by means of noise-free simulations optimized by EAs (see, e.g., \cite{SHIR-DynamicAlign,Shir-JPhysB}).
The time evolution of heteromolecular diatomic alignment is quantum mechanically computed
with the system starting either in the ground rotational level (i.e., at zero temperature), or in a Boltzmann distribution of initial states.
The primary objective is maximization of molecular alignment, quantified by the cosine-squared observable,
$\mathcal{O}_1 = \cos^2(\theta)$,
which considers the angle $\theta$ of the molecular axis with respect to the laser polarization axis.
Fig.\ \ref{fig:box} provides an illustrative overview of the numerical process.
This single-objective form was extended to a bi-criteria framework \cite{SHIR-AlignMO,Klinkenberg},
considering additionally the demand for low-intensity pulses, satisfied by minimizing \emph{second harmonic generation} (SHG).
The bi-criteria formulation is thus posed as obtaining the Pareto front given the following objectives:
\begin{equation}
\label{eq:alignment}
\begin{array}{l}
\medskip
\displaystyle f_1= \left<\cos^2(\theta)\right> \longrightarrow \max\\
\displaystyle f_2 = SHG\left(E(t)\right) = \int_{-\infty}^{\infty} |E(t)|^4 dt \longrightarrow \min .
\end{array}
\end{equation}
For the explicit definition of the cosine-squared observable in $f_1$ we refer the reader to \cite{Vrakking},
while the electric field dependence in $f_2$ follows the formulation in Eq.\ \ref{eq:efield}. The values
of both $f_1$ and $f_2$ are normalized to lie on the interval $[0,1]$.
This bi-criteria molecular alignment problem was previously investigated only for the variant considering
a distribution of initial rotational states \cite{SHIR-AlignMO,Klinkenberg}.
We shall study the problem variant starting in the ground state \cite{Shir-JPhysB},
which constitutes a simulation with a duration of 5sec per single evaluation.
Even upon parallelization of the MO-CMA, we are still facing computationally expensive calculations, which
will practically limit the employment of various strategies and in repeating runs to a certain degree, as will be described.
We consider a discretization of $n=80$ points for the phase function.

We consider the introduction of noise to the phase pixels (Eq.\ \ref{eq:phinoise}), and incorporate
it into the simulation.
In order to evaluate the effect of this noise on the objective values $f_1$ and $f_2$,
the Gaussian variation has to be explicitly propagated
through the Fourier transform and the Schr\"odinger equation.
Such an analytical evaluation is highly complex (especially for $f_1$), generally unknown, and exceeds the scope of this study.

It should be noted that the bi-criteria alignment problem was Pareto optimized
by different variants of the NSGA-II \cite{SHIR-AlignMO} and of the SMS-EMOA \cite{Klinkenberg},
and will be introduced here to the MO-CMA algorithm.
\begin{figure}
\centering \includegraphics[scale=0.55]{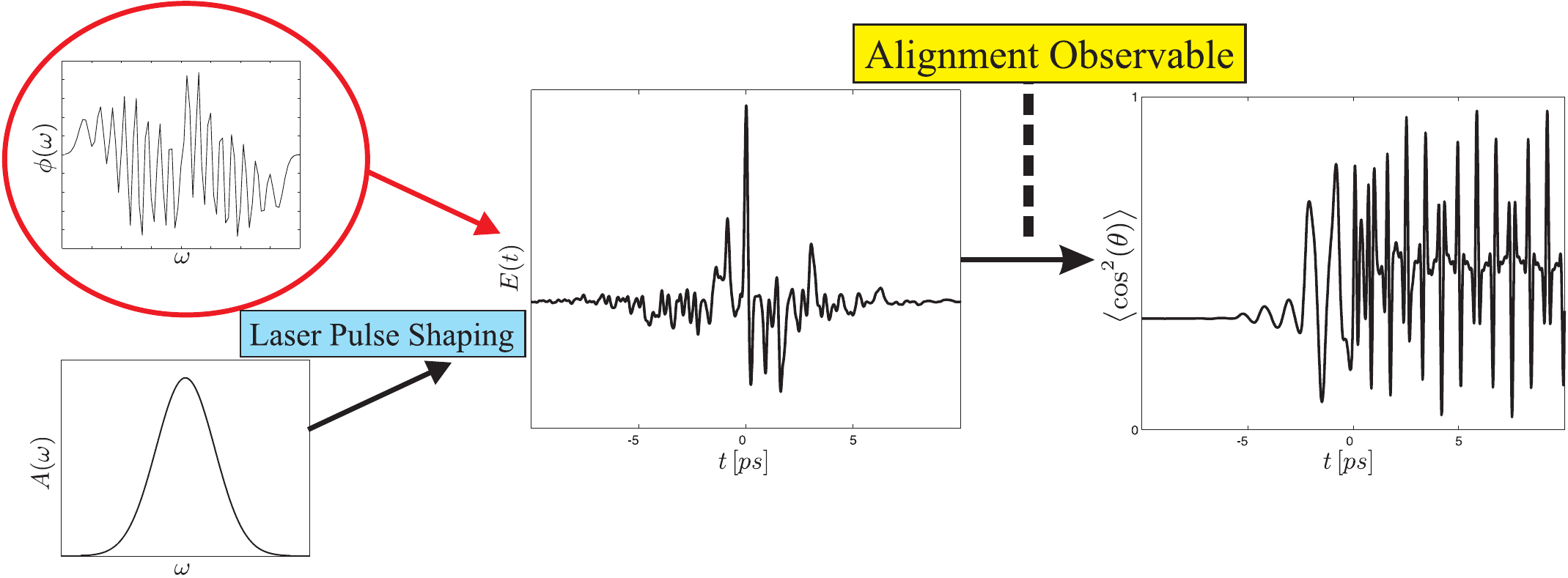} 
\caption{An overview of the numerical modeling of molecular alignment. The control function is the
spectral phase (circled, top left), the amplitude function is fixed and
approximated by a Gaussian (bottom left). The shaping process generates the
electric field, $E(t)$ (center), corresponding to Eq. \ref{eq:efield}. The
``Schr\"odinger Box'' of the alignment observable represents the numerical
calculation of the interaction between the electric field with the molecules,
based on the specified quantum dynamics equations. The revival structure
(right) is the observed simulated behavior of the molecules, upon which the
yield value is based. \label{fig:box}}
\end{figure}

\subsection{Experimental QC System I: Molecular Ion Generation}
\label{sec:labI} We consider the Pareto optimization of a QC
\emph{experimental} system in order to examine the conflict between two
competing quantum mechanical observables. Total ion signal $\mathcal{J}_{Ion}$
resulting from multi-photon ionization of \emph{nitromethane} with shaped,
femtosecond pulses is examined with the goal of discovering a unique set of
ionizing pulses. However, due to the high photon numbers ($\simeq 8$ photons at
800nm) required for single pulse ionization, ion generation is predominantly
dictated by pulse intensity, which obfuscates sensitivity to detailed temporal
control field structure. This inherent intensity dependence is removed by
additionally considering $SHG^{\alpha}$, where $\alpha = 2.5$ in the present
circumstance, as shown later in the inset of Fig.\ \ref{fig:ionmo_front} for
the unshaped reference pulse. Towards this end, we seek to maximize the ion
signal with low-intensity pulses, which naturally results in a conflict between
$\mathcal{J}_{Ion}$ and $SHG^{\alpha}$:
\begin{equation}
\label{eq:expTotalIon}
\begin{array}{l}
\medskip
\displaystyle f_1 =  \mathcal{J}_{Ion} \longrightarrow \max\\
\displaystyle f_2 = SHG^{\alpha} \longrightarrow \min .
\end{array}
\end{equation}
The search is carried out by means of $n=80$ independent phase pixels (see Eq.\ \ref{eq:phase}),
while $\mathcal{J}_{Ion}$ is recorded with a mass spectrometer and SHG is monitored with a two-photon diode.

\subsection{Experimental QC System II: Molecular Plasma Generation}
\label{sec:labII}
As an extension of the \emph{molecular ion generation system}, and as an application of the aforementioned Optimal Dynamic Discrimination concept,
we consider here an equivalent conflict between competing plasma channels.
Total free electron number $\mathcal{J}_{Plasma}$ resulting from multi-photon ionization of \emph{nitromethane} with shaped, femtosecond pulses is diagnosed with radar scattering.
Shaping is performed with the goal of discovering a unique set of ionizing pulses which discriminate against background plasma generation.
Here, also, due to the high photon numbers required for single pulse ionization,
electron generation is predominantly dictated by pulse intensity.
Equivalently, we seek to explore the conflict between $\mathcal{J}_{Plasma}$ maximization and $SHG$ minimization, in an effort to discover
unique, non-intensity dependent ionizing pulses:
\begin{equation}
\label{eq:expRODD}
\begin{array}{l}
\medskip
\displaystyle f_1 =  \mathcal{J}_{Plasma} \longrightarrow \max\\
\displaystyle f_2 = SHG \longrightarrow \min .
\end{array}
\end{equation}
The search is carried out by means of $n=80$ independent phase pixels (see Eq.\ \ref{eq:phase}),
while $\mathcal{J}_{Plasma}$ is recorded with a microwave transmitter/receiver and SHG is monitored with a two-photon diode.

The reader should keep in mind that despite some similarities in the two aforementioned laboratory systems -- i.e.,
Molecular Ion Generation (Section \ref{sec:labI}) versus Molecular Plasma Generation (Section \ref{sec:labII}) --
they possess very different experimental designs, and most importantly, they are subject to fundamentally different underlying physics.
Table \ref{tab:PROBLEMSUMMARY} summarizes the problems investigated in
this study.
\begin{table}
\caption{ Summary of Systems under Investigation\label{tab:PROBLEMSUMMARY}}
{\bf Simulations: Model Landscapes}\\
\begin{tabular}{l l l l} \hline
\emph{Problem Name} & \emph{Formulation} & \emph{Dimensionality} & \emph{Noise Levels}\\
Multi-Sphere & Eq.\ \ref{eq:multisphere} & $n=\left\{10,30,80\right\}$ & $\epsilon_x^2 = \left\{0.001,0.005,0.01,0.02,0.05 \right\}$\\
Diffraction Grating & Eqs.\ \ref{eq:IntensityDG}, \ref{eq:I05Tq} & $n=\left\{10,30,80\right\}$ & $\epsilon_S^2 = \left\{0.001,0.005,0.01,0.02,0.05 \right\}$\\
\hline
\end{tabular}
\\
{\bf Real-World Simulator} \\
\begin{tabular}{l l l l} \hline
\emph{Problem Name} & \emph{Description} & \emph{Dimensionality} & \emph{Noise Levels}\\
Molecular Alignment & Eq.\ \ref{eq:alignment} & $n=80$ & $\epsilon_{S}^2 = \left\{0.001,0.005,0.01,0.02,0.05 \right\}$\\
\hline
\end{tabular}
\\  {\bf Laboratory Experiments} \\
\begin{tabular}{l l l l} \hline
\emph{Problem Name} & \emph{Description} & \emph{Dimensionality} & \emph{Measured Noise Level}\\
Total-Ion Generation & Eq.\ \ref{eq:expTotalIon}  & $n=80$ & $\epsilon_{S}^2 \approx 0.01$\\
Plasma Generation & Eq.\ \ref{eq:expRODD} & $n=80$ & $\epsilon_{S}^2 \approx 0.01$\\
\hline
\end{tabular}
\end{table}
\normalsize

\section{Practical Observations}
\label{sec:experiments}
We describe here our observations of the three frameworks specified in the previous section: Model landscapes, QC simulations, and QC experiments.
Towards this end, we adhere to the structured reporting scheme suggested by Preuss \cite{Pre07}, starting by posing the scientific question to answer.
Each framework is treated by means of relevant methodologies, which depend upon the research question as well as upon the practical constraints
(computational resources, experimental considerations, etc.).
Section \ref{sec:obs_spheres} focuses on the performance of the MO-CMA on the Multi-Sphere landscape subject to noise.
Section \ref{sec:obs_diffraction} considers the performance of several EMOA on the Diffraction Grating problem.
Section \ref{sec:obs_align} reports on results of the simulated Molecular Alignment problem,
and finally, sections \ref{sec:obs_lab} and \ref{sec:obs_lab2} present laboratory results of the Molecular Ion Generation
and Molecular Plasma Generation problems, respectively.

\noindent\textbf{Pre-Experimental Planning.} The MO-CMA code relies on the
Shark Library release 2.2.1\footnote{http://shark-project.sourceforge.net/}
\cite{Shark}. The simulated systems\footnote{A software package of the
Diffraction Grating problem will be provided by the authors upon request.} are
optimized by means of an extended MPI-based parallel implementation to the
Shark code, while the laboratory employs an extended LabView version, which
relies on Shark DLL's. The default parameters are kept, with a total population
size of either $\mu_S=\lambda_S=100$ search points for the simulations, or
$\mu_L=\lambda_L=50$ search points in the laboratory. Random initialization of
search points is carried out uniformly in the interval $\left[-10,10\right]^n$
for the Multi-Sphere cases, and in $\left[0,2\pi\right]^n$ otherwise. The
initialization in the experimental systems also relies on \emph{seed} search
points, which were obtained in single-objective CMA-ES runs addressing a
tailored ratio objective function.

The presentation of the results will include the archived \emph{perceived} fronts attained by the MO-CMA
for all frameworks under investigation. For the two simulated frameworks, we are in a privileged position to reevaluate
archived solutions with noise-free objective functions, and thus we shall present also the \emph{ideal} fronts,
which are calculated \emph{a posteriori}.

We would like to stress the fact that the perceived fronts, due to the elitist strategy in use, are expected to represent
the tail of the disturbance distribution, as projected on the archived solutions.
It is important to consider to what extent the attained perceived front may be reconstructed \emph{de facto}
given the archived solutions. Therefore, we will generate statistical samples of each archived solution, subject to the same
noise conditions, and present additionally the nature of the obtained distributions. We consider this a direct indication of the usefulness
of the archived solutions.

\subsection{Preliminary: MO-CMA on the Multi-Sphere Landscape}
\label{sec:obs_spheres}
\noindent\textbf{Research Question.} How does noise on the decision parameters affect the MO-CMA performance, if at all, and do any
of the considered schemes of three parental re-evaluation scenarios (Section \ref{sec:proposedscheme}) handle noise better?

\noindent\textbf{Performance Criteria.}
In order to assess the quality of the obtained Pareto fronts in the different noisy test-cases,
we shall consider two performance criteria.
Given the attained \emph{hypervolume indicator} values, $V_i$ \cite{zt1998b,ztlf2003a}
(also known as 'S-Metric' \cite{zitz1999a} or 'Lebesgue Measure' \cite{LRS1999}),
the first criterion is their \emph{relative deterioration} with respect to the hypervolume of the
Pareto front obtained in noise-free conditions, $V_{\epsilon_x=0}$.
This criterion will be assessed numerically, for which we set up and test a corresponding quantifier:
\begin{equation}
\label{eq:normalizedperformance}
 \Delta_V^{(i)} = \frac{V_{\epsilon_x=0}-V_i}{V_{\epsilon_x=0}}.
\end{equation}
The second criterion is the spatial distribution of the attained front,
for which we set up and test a corresponding quantifier. In particular, given a final population of size
$\mu$, $\left\{\vec{f}_k^{(i)} \right\}_{k=1}^{\mu}$, sorted by means of partial order,
let us consider its $\chi^2$ value with respect to a reference noise-free population, $\left\{\vec{p}_k \right\}_{k=1}^{\mu}$,
which toward this end represents a desired distribution of points along the front:
\begin{equation}
\label{eq:objectivediversity}
\Delta_D^{(i)} = \sum_{k=1}^{\mu} \frac{\| \vec{f}_k^{(i)} - \vec{p}_k \| ^2}{\| \vec{p}_k \|} \sim \chi^2\left(\mu \right).
\end{equation}
In essence, values given by Eqs.\ \ref{eq:normalizedperformance} and \ref{eq:objectivediversity}
reflect the degrees of deterioration in the hypervolume and the spatial diversity, respectively,
with respect to the noise-free simulations.

\subsubsection{Numerical Results}
\noindent\textbf{Setup.} We consider here the numerical results of the various
simulations on the Bi-Sphere model landscape. While the number of function
evaluations per scheme varied, due to the parental re-evaluation procedure, the
number of total iterations was fixed per search space dimensionality:
$num_{iter}=\left\{10^4,2\cdot10^4,5\cdot10^4 \right\}$ for
$n=\left\{10,30,80\right\}$, respectively. Those values were set based on
preliminary runs, in which the MO-CMA converged to a highly-satisfying front,
with minimal error from the true Pareto front, and with a uniform distribution
of points. For the hypervolume calculations, a reference point at $[2,2]$ is
considered.

\noindent\textbf{Experimentation/Visualization.} We focus on presenting
statistical analyses of specific test-cases, comparing the 3 different MO-CMA
schemes. Overall, taking into account the \emph{a posteriori} calculation, we
shall have two sets of results per procedure. Fig.\ \ref{fig:boxplots} depicts
the statistical box-plots for $\Delta_V$ values of the Multi-Sphere landscape,
taking into account only converged points in the box $[0,2]^2$ in the objective
space, for three test-cases: $n=10$ with $\epsilon_x^2=0.05$ (top), $n=30$ with
$\epsilon_x^2=0.02$ (middle), and $n=80$ with $\epsilon_x^2=0.01$ (bottom).
Fig.\ \ref{fig:ODboxplots} depicts the equivalent box-plots for the $\Delta_D$
calculations (considering all 30 runs per case).
\begin{figure*}
\begin{multicols}{2}
\centering \includegraphics[scale=0.4]{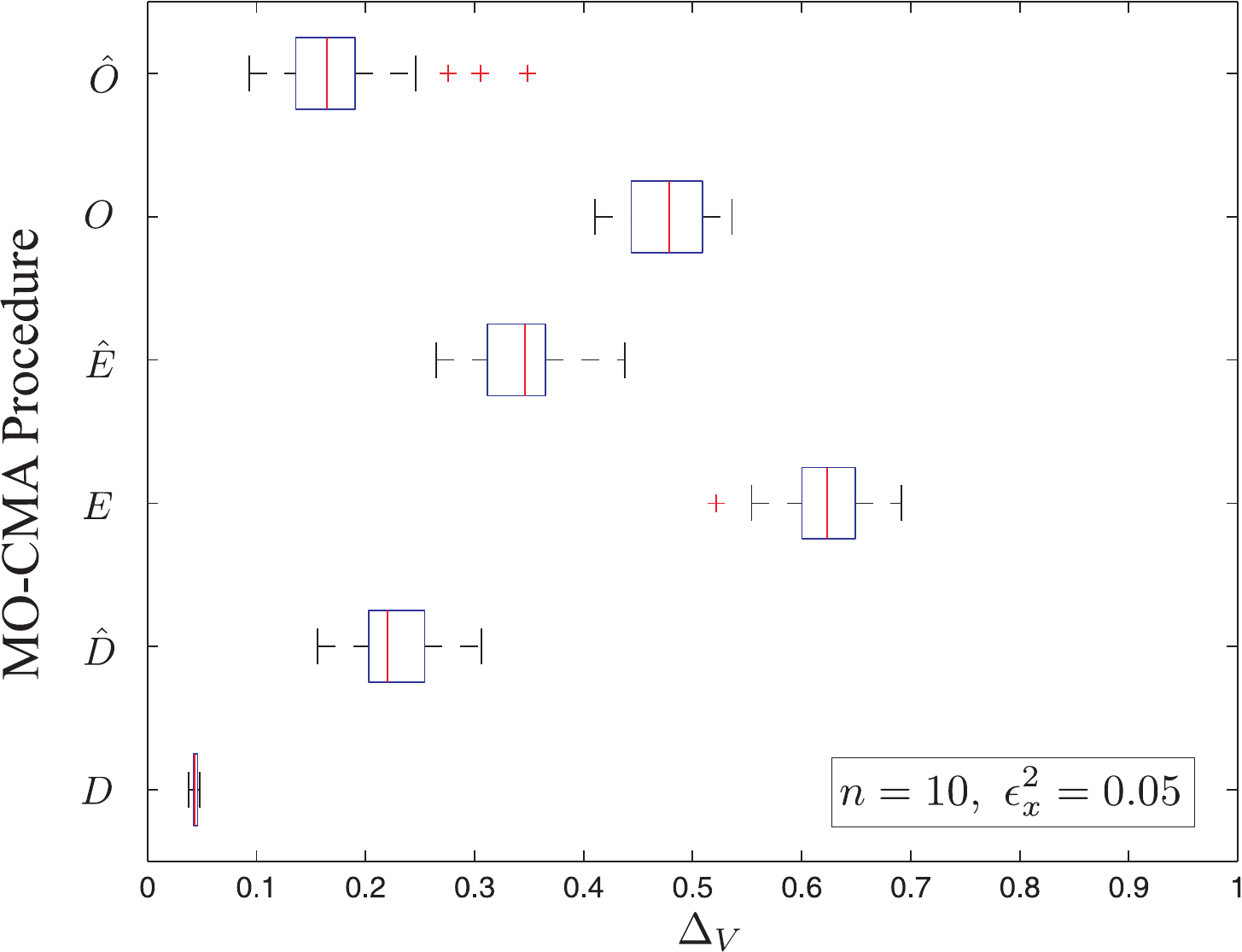}
\centering \includegraphics[scale=0.4]{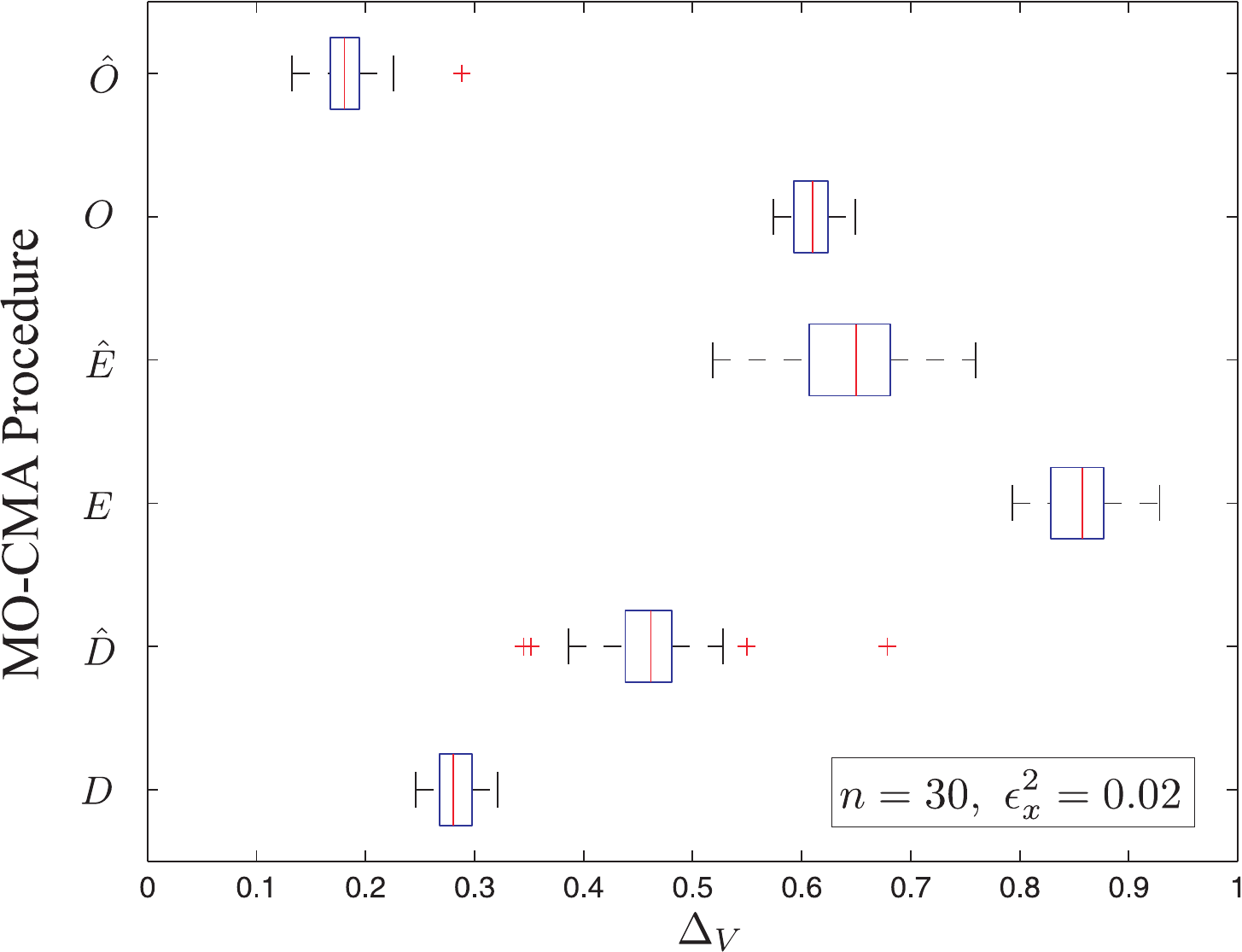}
\centering \includegraphics[scale=0.4]{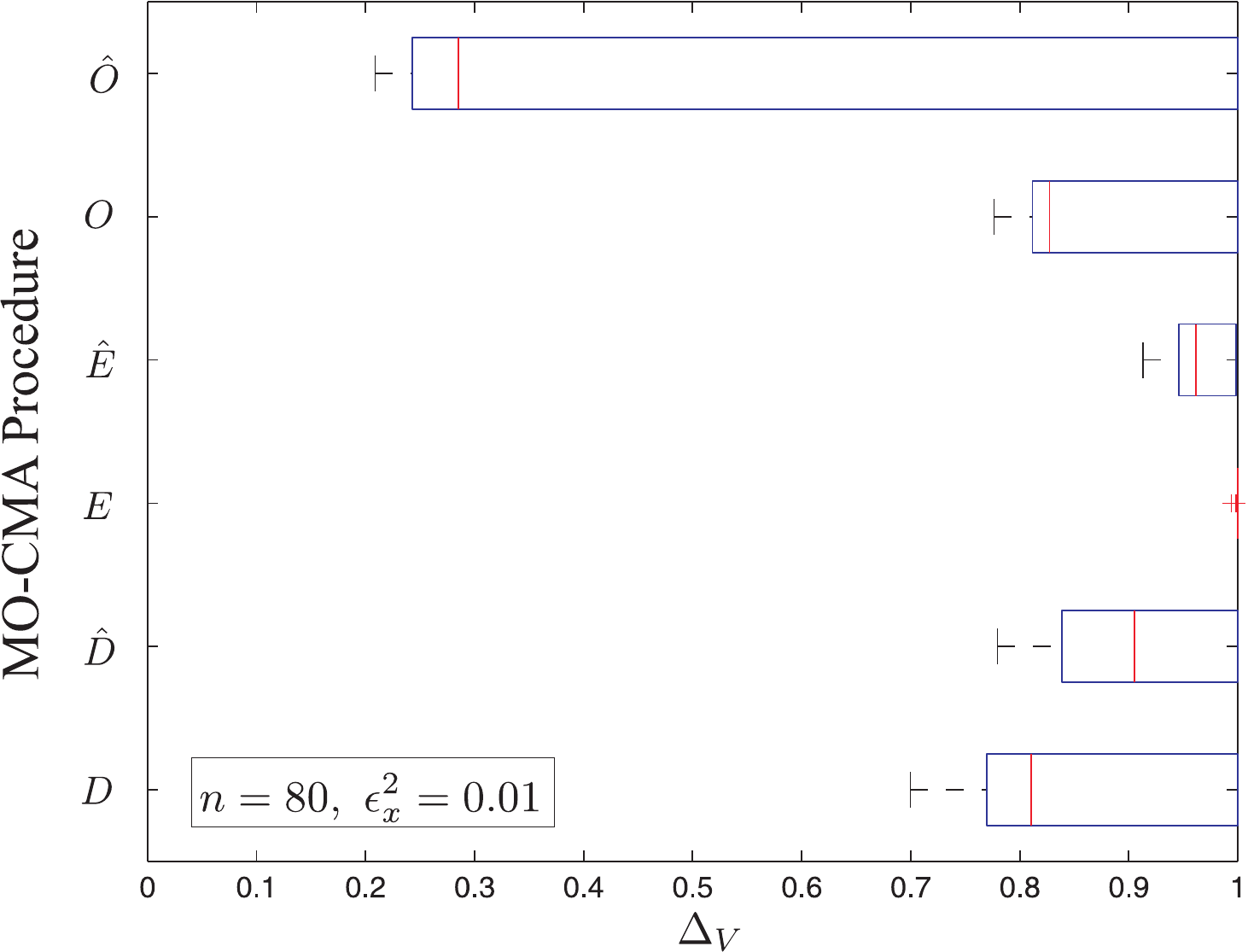}
\caption{Box-plots of $\Delta_V$ values (Eq.\ \ref{eq:normalizedperformance}) over all converged Multi-Sphere
runs, of three test-cases:
$n=10$ with $\epsilon_x^2=0.05$ [top], $n=30$ with $\epsilon_x^2=0.02$ [middle], and $n=80$ with $\epsilon_x^2=0.01$ [bottom].
The \emph{perceived fronts} of the three optimization procedures, corresponding to the
three parental re-evaluation scenarios (Section \ref{sec:proposedscheme}), are
noted as $\left\{D,~E,~O\right\}$.
The \emph{ideal fronts} (noise-free evaluation of the Pareto sets) are
noted as $\left\{\hat{D},~\hat{E},~\hat{O}\right\}$. Each case consists of 30 runs.
\label{fig:boxplots}}
\newpage
\centering \includegraphics[scale=0.4]{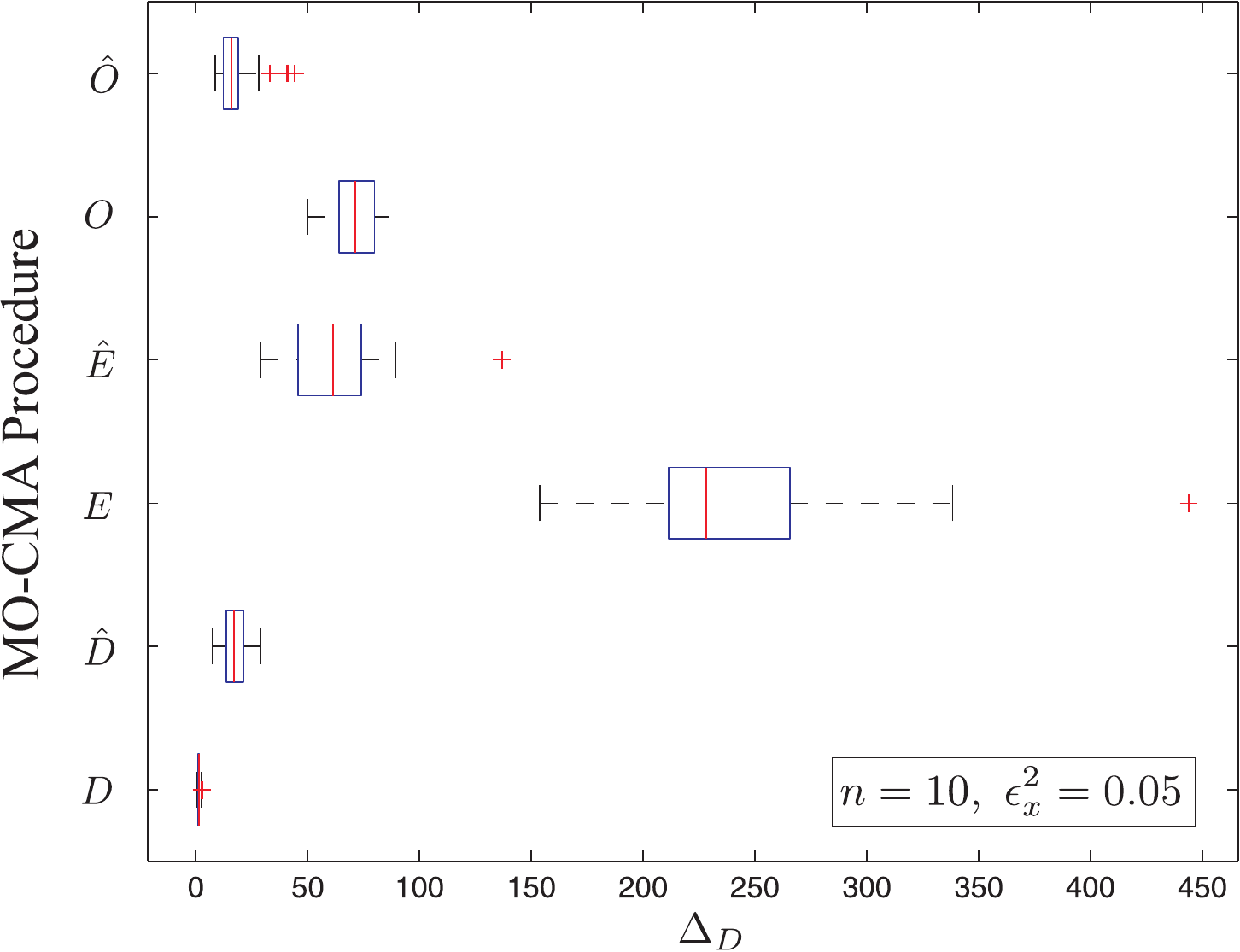}
\centering \includegraphics[scale=0.4]{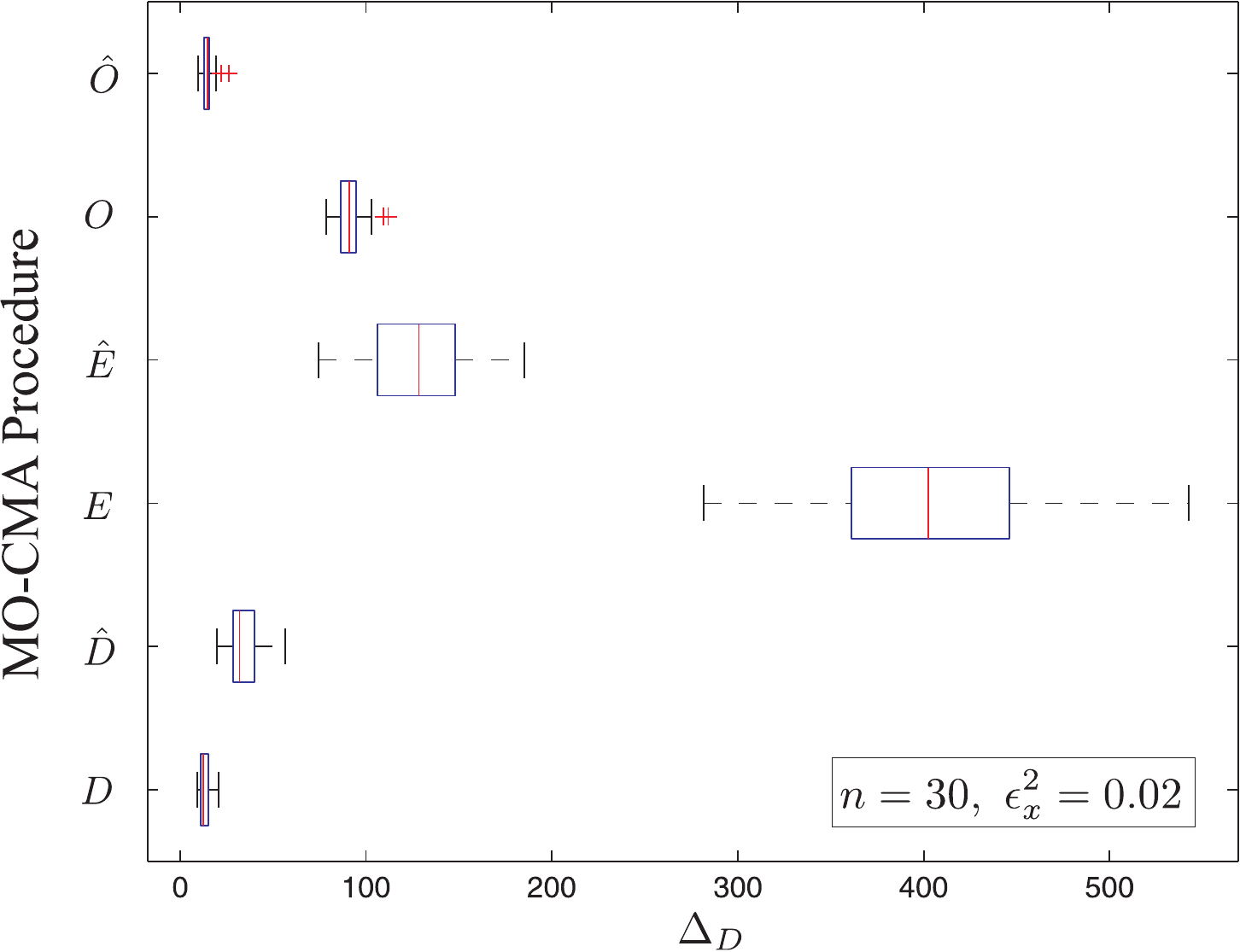}
\centering \includegraphics[scale=0.4]{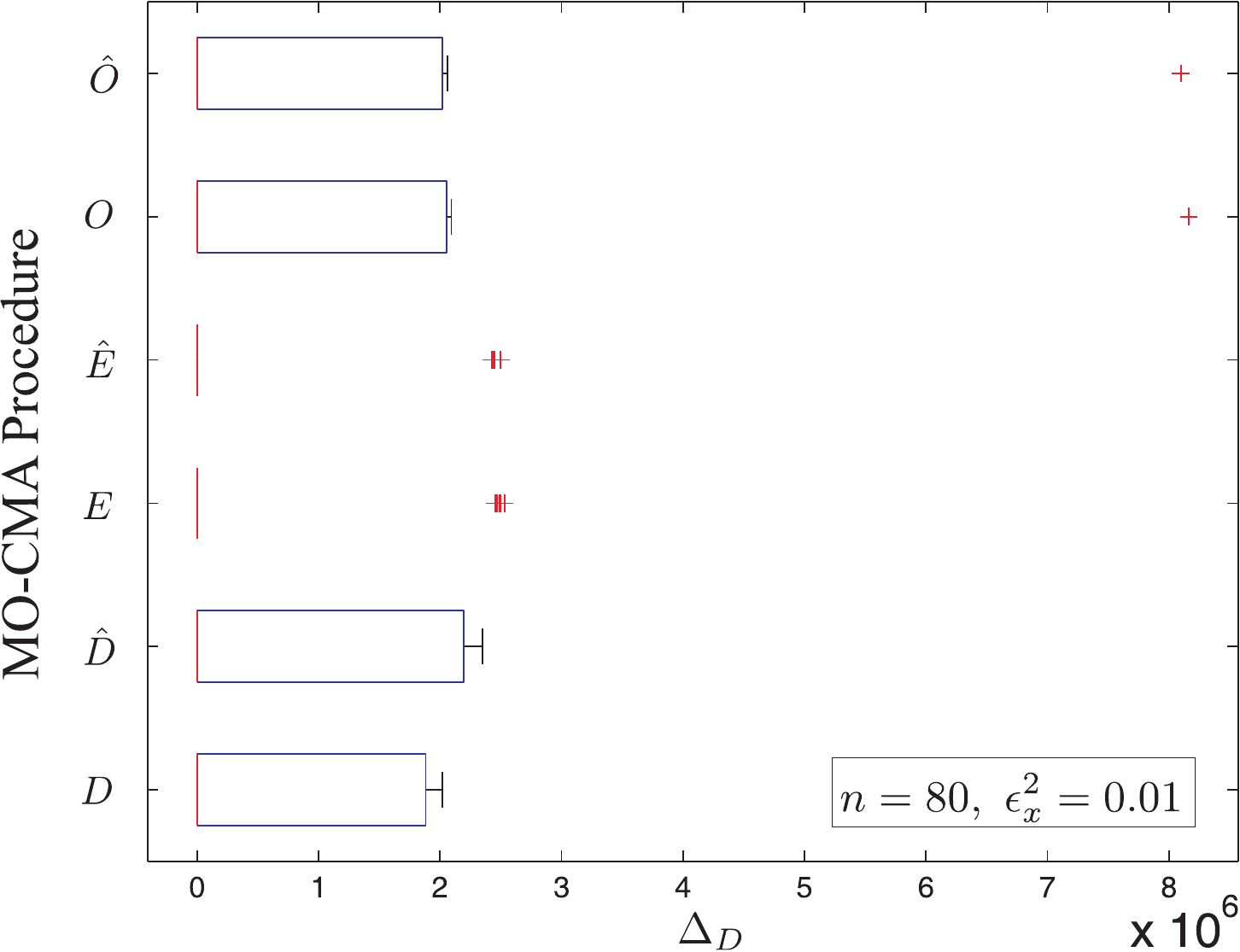}
\caption{Box-plots of $\Delta_D$ values (Eq.\ \ref{eq:objectivediversity}) over all converged Multi-Sphere
runs, of three test-cases:
$n=10$ with $\epsilon_x^2=0.05$ [top], $n=30$ with $\epsilon_x^2=0.02$ [middle], and $n=80$ with $\epsilon_x^2=0.01$ [bottom].
The \emph{perceived fronts} of the three optimization procedures, corresponding to the
three parental re-evaluation scenarios (Section \ref{sec:proposedscheme}), are noted as $\left\{D,~E,~O\right\}$.
The \emph{ideal fronts} (noise-free evaluation of the Pareto sets) are
noted as $\left\{\hat{D},~\hat{E},~\hat{O}\right\}$. Each case consists of 30 runs.
\label{fig:ODboxplots}}
\end{multicols}
\end{figure*}

\subsubsection{Discussion}
While the perceived fronts given as output by the MO-CMA provide fair Pareto
approximations, with some expected error due to the presence of noise, an
examination of the actual archived solutions reveals an entirely different
picture. When exposed to noise on the decision parameters, the default MO-CMA
is observed to lack population diversity in the objective space for all search
space dimensions under investigation. This effect becomes evident upon the
\emph{a posteriori} noise-free evaluation of the archived solutions: the
outcome is several clustered points along the perceived front, as depicted in
Fig.\ \ref{fig:clusteringeffect}. The lack of diversity continually worsens as
the expected disturbance increases, i.e., higher noise strength and higher
dimensionality lead to increased clustering. Fig.\ \ref{fig:ODboxplots} depicts
box-plots for the $\Delta_D$ values of 3 Bi-Sphere test-cases. While the raw
$\Delta_D$ values do not reflect the degree of discrepancy by themselves, it is
important to consider those values with respect to the perceived front of the
default MO-CMA, which typically obtains a fair approximation to the true front
given the disturbance. This effect may also be observed in Fig.\
\ref{fig:boxplots}, when noticing the considerable counter-intuitive
differences in the $\Delta_V$ values between the default MO-CMA ('$D$') and its
\emph{a posteriori} \textbf{noise-free} evaluation ('$\hat{D}$').

The proposed explanation for the observed lack of diversity is the following.
During the run, search points which lead in the progress towards the Pareto
front generate offspring by means of Gaussian sampling (Eq.\ \ref{eq:cma_gen}).
Offspring with good positions with respect to the front, especially whose
disturbed fitness values lie along the currently progressing front, are
selected, and their decision parameters are archived. While the perceived
offspring's point in the objective space may represent a promising coordinate
with respect to ranked domination as well as to hypervolume contribution, its
pre-image in the decision space is merely a small deviation from the original
parent. In practice, leading individuals take-over the population, since
generating offspring by means of small mutations in combination with the noise
disturbance is sufficient to span a fair distribution along the Pareto front.
This statement was numerically assessed by explicitly calculating the expected
distribution with the analytical forms of Eqs.\ \ref{eq:spheres_mean} and
\ref{eq:spheres_var}, and it was furthermore corroborated with the sampling of
the actual archived Pareto optimal set of an MO-CMA run. The aforementioned
calculations are depicted in Fig.\ \ref{fig:clusteringeffect}, which provides a
clear picture -- \textbf{the obtained clusters are the minimal configuration of
points for sampling the entire Pareto front with the current noise level, and
moreover, the perceived front can indeed be reconstructed by elitist selection
of the attained statistical sample.} It is also evident from further
calculations that the number of clusters increases with the reduction of noise
disturbance, as expected from Eqs.\ \ref{eq:spheres_mean} and
\ref{eq:spheres_var}. This clustering effect may be considered as a
multi-objective generalization to the systematic overvaluation effect, as
discussed by Arnold and Beyer for the single-objective case in
\cite{ArnoldB02}. We thus claim that fitness disturbance in multi-objective
optimization is responsible for the low objective space diversity in the
archiving mechanism of the MO-CMA.

As a second routine employed, parental re-evaluation every generation clearly
hampered the performance of the default MO-CMA. The attained solutions
constitute worst quality sets, when compared to the default procedure, for all
the different test-cases under investigation. This poor performance may be
clearly observed in Figs.\ \ref{fig:boxplots} and \ref{fig:ODboxplots} when
considering '$E$' / '$\hat{E}$'. The explanation for this behavior is a
stochastic disturbance to the archiving mechanism, which has a direct negative
impact on the consistency of the selection phase.

The third routine, MO-CMA with occasional parental re-evaluation (every 10
generations), seems empirically to be the best solution for the systematic
disturbance problem. While low population diversity, as assessed with
$\Delta_D$ values, is still observed upon the \emph{a posteriori} noise-free
evaluation of the archived solutions, the attained clusters are bigger in size,
and closer to the true Pareto front. Essentially, the archived solutions of
this procedure are of the highest quality when reconstructed \emph{a
posteriori} in comparison to the other procedures (see '$O$' / '$\hat{O}$' in
Figs.\ \ref{fig:boxplots} and \ref{fig:ODboxplots}). The perceived Pareto front
is typically not as good as the one attained by the default MO-CMA, but unlike
the default procedure, the \emph{a posteriori} noise-free evaluation yields a
better Pareto front in comparison to its perceived front, and especially better
than the post-default front. This effect is also visually apparent when
exploring the box-plots of both quantifiers and noting the inversion of roles:
while '$D$' is always of higher quality than '$\hat{D}$', '$O$' is of lower
quality than '$\hat{O}$'. We conclude that in line with the single-objective
scenario, occasional parental fitness re-evaluation seems to be superior with
respect to no re-evaluation at all and to re-evaluation in every generation.
\begin{figure}
\centering \includegraphics[scale=1.0]{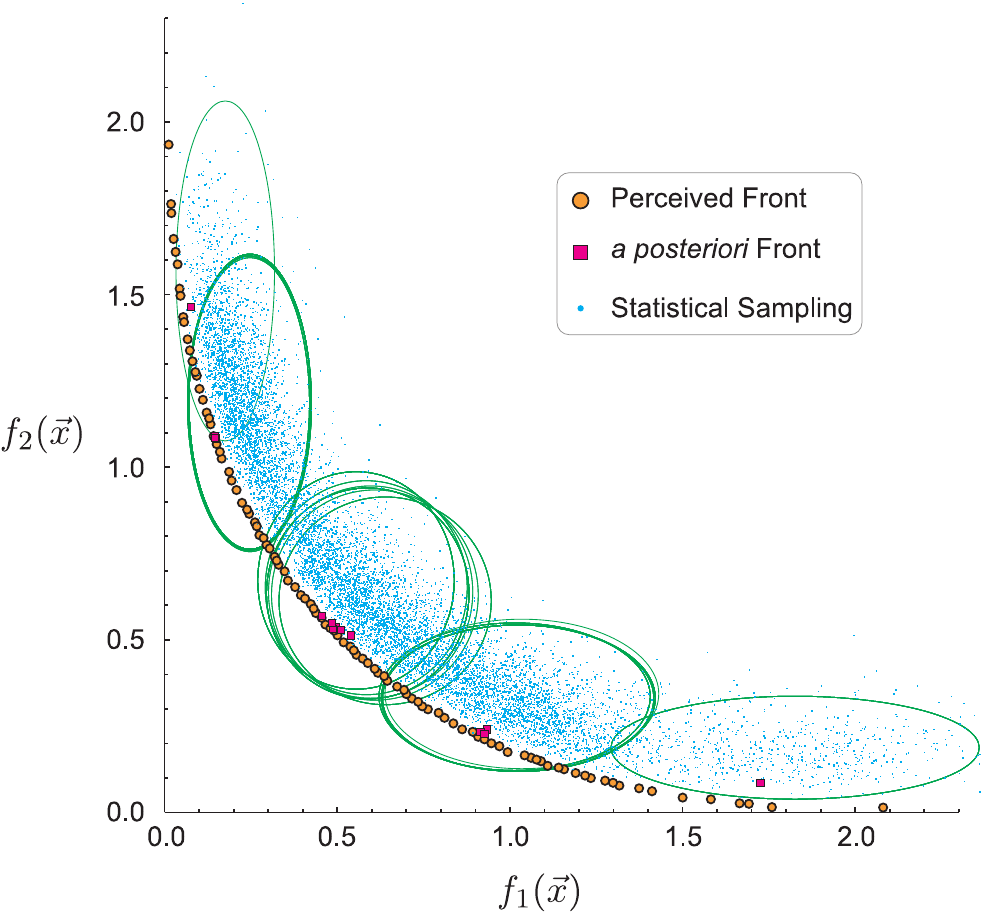}
\caption{Statistical sampling of the Pareto set attained by the MO-CMA on the
Multi-Sphere with $n=10$ at $\epsilon_x^2=0.01$. The perceived Pareto front
constitutes an excellent approximation to the true front, and the \emph{a
posteriori} noise-free evaluation of its pre-images yields clusters along the
front, whose sampling subject to the same noise level yields the depicted cloud
of points. The ellipses represent the disturbance distributions, centered about
the mean with twice the standard deviations as axis, based upon the analytical
forms of the perceived fitness in Eqs.\ \ref{eq:spheres_mean} and
\ref{eq:spheres_var}. It is clear from these results that the clusters are the
minimal configuration of points for sampling the entire Pareto front, subject
to elitism, with the current noise level. \label{fig:clusteringeffect}}
\end{figure}

\subsubsection{Reference Algorithms}
We considered additional standard EMOA as reference methods to the MO-CMA, in
order to observe their behavior on the Multi-Sphere model landscape, subject to
the current modeling of noise. We carried out simulations on similar test-cases
with the NSGA-II \cite{Deb01} as well as with the SMS-EMOA \cite{EBN05}. We
employ Deb's operators and his defaults settings for the
NSGA-II\footnote{Source code of the NSGA-II algorithm used in this work was
downloaded from the KanGAL homepage: http://www.iitk.ac.in/kangal/}. Regarding
the SMS-EMOA, we follow the settings described at \cite{SMS-EMOA_Journal}\footnote{Source code was provided by Michael Emmerich}. 
The population sizes are
similar to those employed by the MO-CMA. These settings hold for the
application of both NSGA-II and SMS-EMOA throughout the entire study. Typical
runs of both algorithms on the case of $n=10$ with $\epsilon_x^2=0.01$ are
depicted in Fig.\ \ref{fig:referenceEMOA}, presenting the perceived fronts
versus the \emph{a posteriori} noise-free evaluation of the attained Pareto
optimal sets.
\begin{figure*}
\begin{multicols}{2}
\centering \includegraphics[scale=0.4]{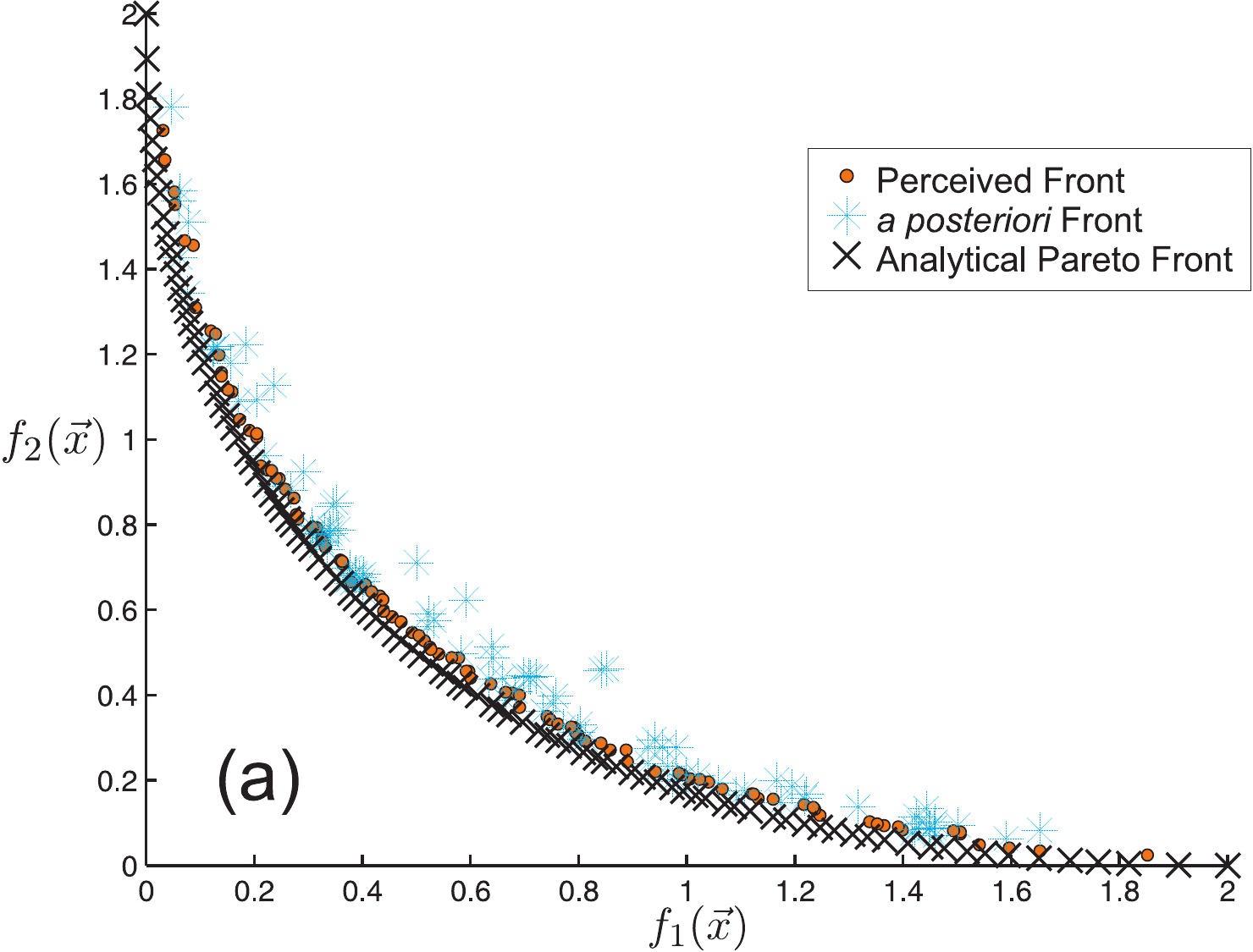}
\newpage
\centering \includegraphics[scale=0.4]{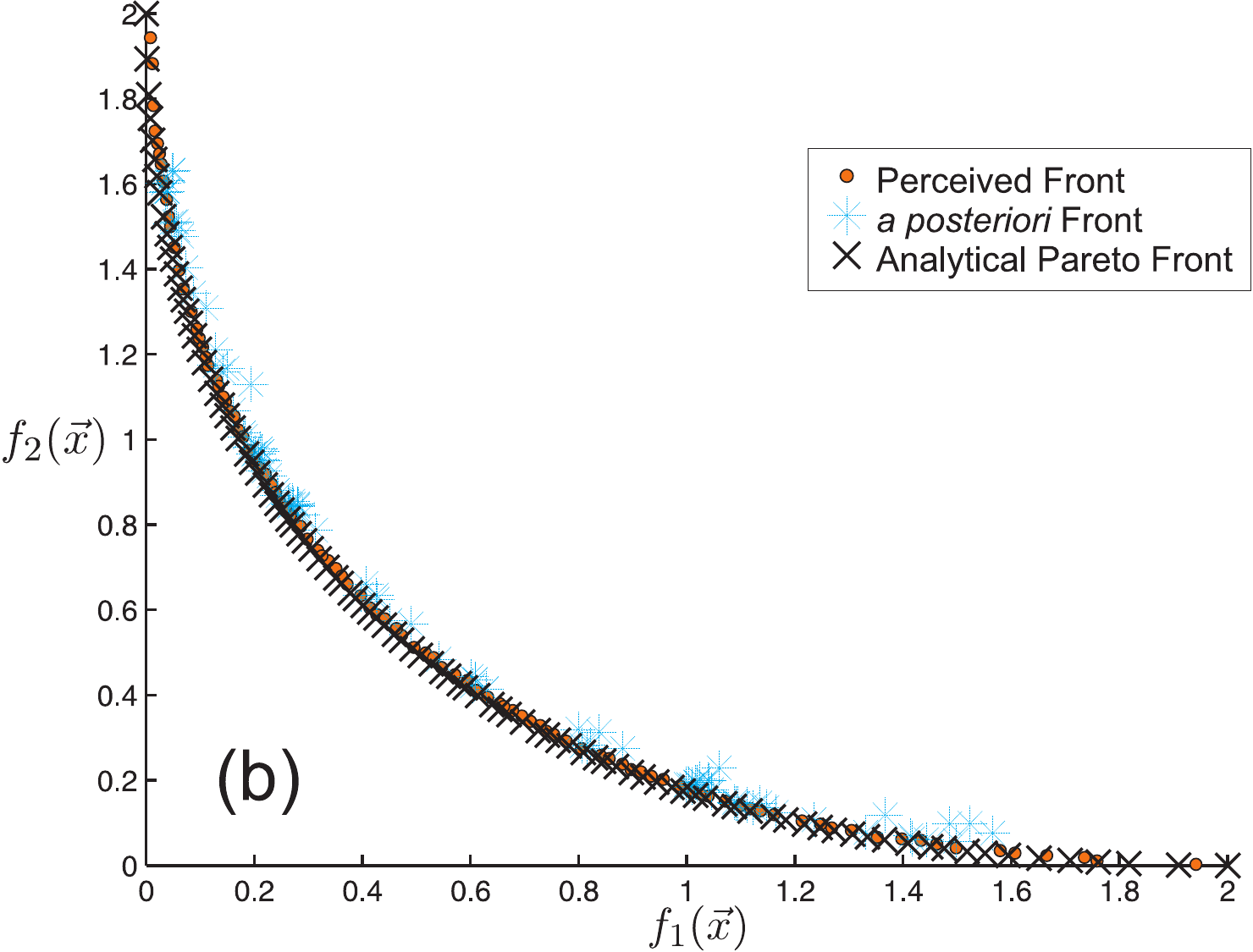}
\end{multicols}
\caption {Typical runs of the reference EMOA on the noisy Multi-Sphere case of
$n=10$ with $\epsilon_x^2=0.01$: [LEFT (a)] NSGA-II versus [RIGHT (b)]
SMS-EMOA. The figures depict the perceived fronts, the \emph{a posteriori}
noise-free evaluation of the Pareto optimal sets, and the analytical Pareto
front (Eq.\ \ref{eq:2d-sphere-front}).\label{fig:referenceEMOA}}
\end{figure*}
The NSGA-II attained a perceived Pareto front which constitutes a good
approximation to the true front, and at the same time, the noise-free
reconstruction of the Pareto optimal set provides a reasonable front. The
SMS-EMOA, on the other hand, attained a perceived Pareto front which offers an
excellent approximation to the true front, and upon the noise-free
re-evaluation of the Pareto optimal set the reconstructed front is observed to
lose its diversity to some extent. It should be stressed that the absolute
'clustering effect' within the archiving mechanism, which was typical of the
MO-CMA, was not observed for these reference algorithms. This might reflect the
difference between an algorithm which is clearly designed for learning
distributions (i.e., employing statistical learning), such as the MO-CMA,
versus EMOA with traditional evolutionary core mechanisms, which evidently
operate in a na\"ive way. Overall, in terms of the capacity to reconstruct
Pareto information out of the archived solutions, SMS-EMOA seems to perform
best on the Multi-Sphere noisy model landscape. A more comprehensive
performance comparison between these three EMOA will be carried out in the
following section with regard to the Diffraction Grating model landscape.

\subsubsection{Noisy Tri-Sphere Simulations}
Finally, we tested the behavior of the MO-CMA on the Tri-Sphere case (Eq.\ \ref{eq:multisphere} with $m=3$).
Toward this end, we employed a steady-state implementation which reduces the extensive complexity of the hypervolume calculations.
We provide here a brief qualitative description of our observations.
The MO-CMA obtained a good approximate Pareto surface for the noise-free problem. Upon consideration of systematic noise on the decision
parameters, as done in the Bi-Sphere case, the diversity loss effect in the archiving mechanism of the decision space is not observed to be
significant any longer, even at high noise levels of, e.g., $\epsilon_x^2=0.05$.
We propose the following explanation for this observation: given the selection mechanism of the MO-CMA,
treatment of an additional objective reduces the selection pressure. Lower pressure may thus reduce the probability
of take-over, which was our understanding of the mechanism for the 'clustering effect'.

\subsection{Diffraction Grating: Extensive Performance Comparison}
\label{sec:obs_diffraction}
Rather than considering the individual performances of the 3 MO-CMA schemes,
we present a comprehensive performance comparison between the default MO-CMA, SMS-EMOA,
and NSGA-II on the Diffraction Grating problem in several dimensions and at various noise levels.
As a secondary \emph{research question}, we aim at reporting on the MO-CMA behavior on this problem.

Let us begin by qualitatively describing the MO-CMA behavior on this search
problem, in light of the observation reported in Section \ref{sec:obs_spheres}.
Fig.\ \ref{fig:ellipses_diffraction} depicts typical results of the MO-CMA on
two variants of the Diffraction Grating problem with $n=10$ phase points at two
noise levels. Equivalent to Fig.\ \ref{fig:clusteringeffect}, the Pareto sets
are reconstructed \emph{a posteriori} in noise-free evaluations, then
statistically sampled at the same noise levels of the evolutionary run, and
compared to the perceived Pareto fronts, given as output by the algorithm. As a
reference, the ellipses representing the noise distribution are plotted,
according to Eq.\ \ref{eq:DG_mean} (mean) and Eq.\
\ref{eq:gratingvarconclusion}-\ref{eq:gratingvarfinal} (variance; see Appendix
\ref{app:perceivedcalculations}). It is straightforward to observe the
'clustering effect' in the archiving mechanism, similar to the one occurring in
the Multi-Sphere case.
\begin{figure*}
\begin{multicols}{2}
\centering \includegraphics[scale=0.4]{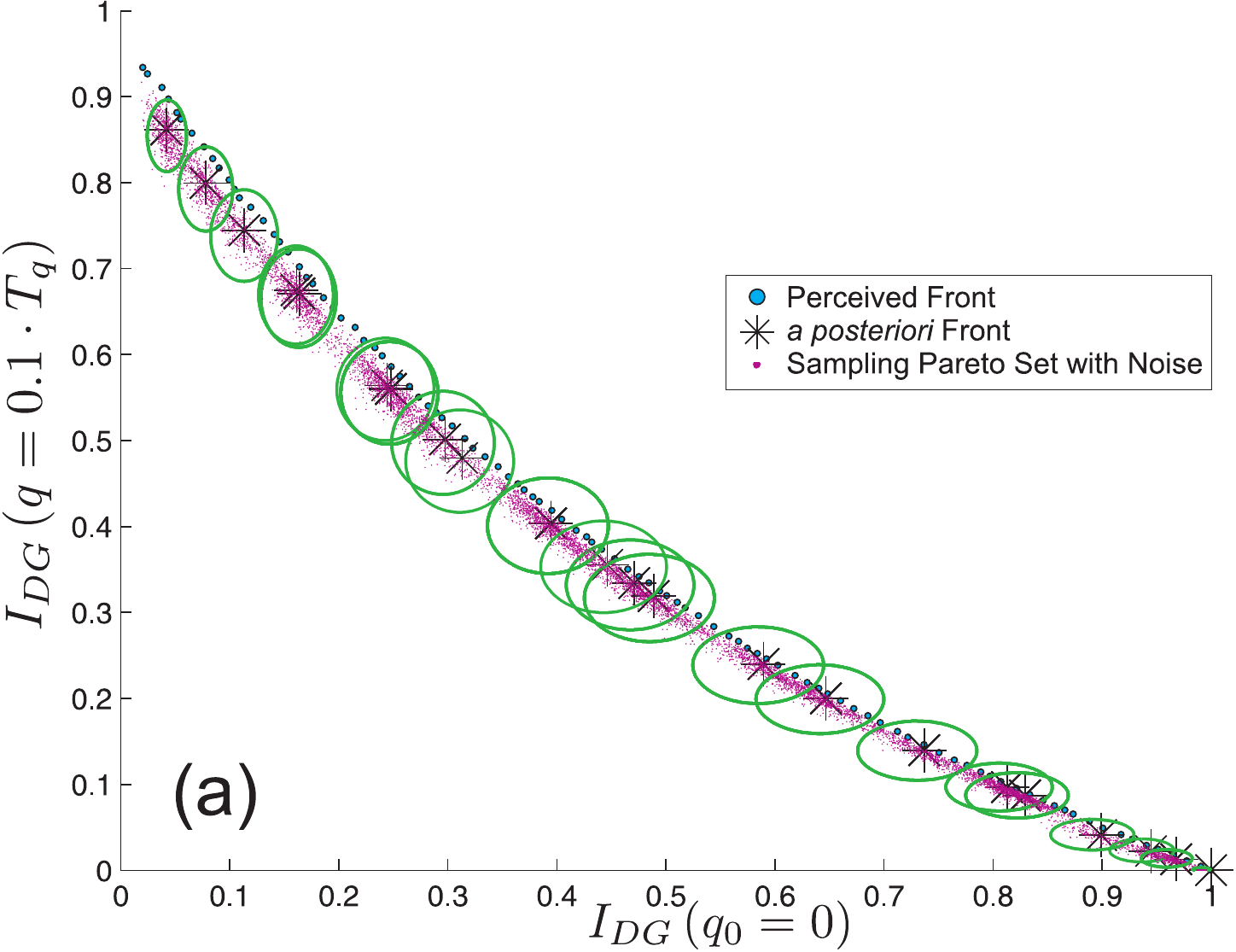}
\newpage
\centering \includegraphics[scale=0.4]{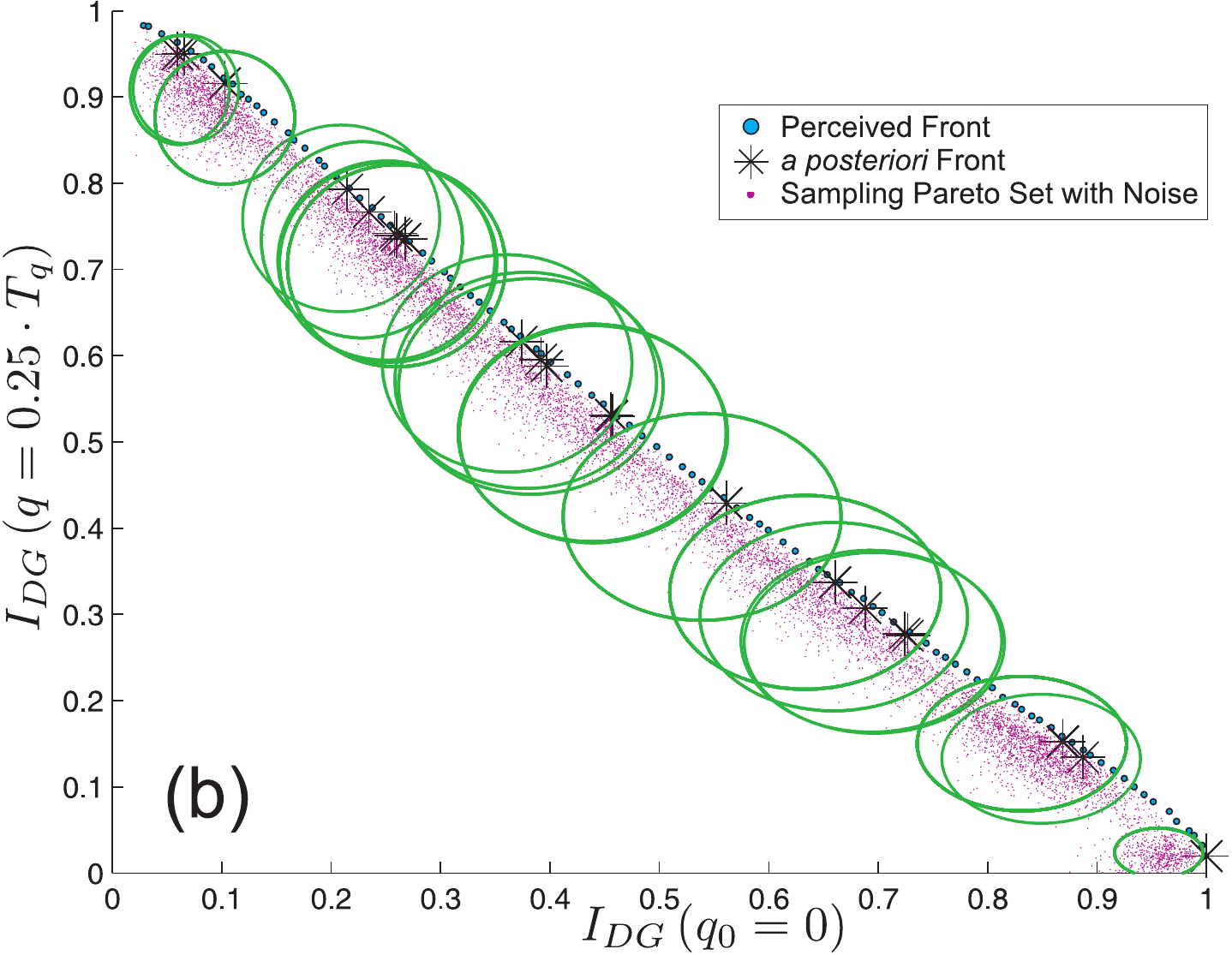}
\end{multicols}
\caption {Statistical sampling of the Pareto optimal sets attained by the
MO-CMA for two variants of the Diffraction Grating problem with $n=10$ phase
points. Equivalent to Fig.\ \ref{fig:clusteringeffect}, the Pareto sets are
reconstructed by means of noise-free evaluation, and compared to their
statistical sampling at the noise level of the evolutionary run, as well as to
the perceived Pareto fronts, given as output by the algorithm. As a reference,
the noise distributions are depicted according to the analytical results of
Eq.\ \ref{eq:DG_mean} (mean) and Eqs.\ \ref{eq:gratingvarconclusion} and
\ref{eq:gratingvarfinal} (variance; see Appendix
\ref{app:perceivedcalculations}). [LEFT, (a)] Maximization of $I_{DG}(q_0=0)$
versus $I_{DG}(q=0.1\cdot T_q)$ at $\epsilon_x^2=0.01$. [RIGHT, (b)]
Maximization of $I_{DG}(q_0=0)$ versus $I_{DG}(q=0.25\cdot T_q)$ at
$\epsilon_x^2=0.05$. \label{fig:ellipses_diffraction}}
\end{figure*}
\subsubsection{Numerical Results}
\noindent\textbf{Setup.} We consider here simulations on a specific case of the
Diffraction Grating problem (Eqs.\ \ref{eq:IntensityDG} and \ref{eq:I05Tq} set
up with $b=1,~h=4$), in search space dimensions of $n=\left\{10,~30,~80
\right\}$, and at noise levels given by Eq.\ \ref{eq:noisestrength}. We fix the
total number of function evaluations per search space dimensionality:
$num_{evals}=\left\{10^6,2\cdot 10^6,5\cdot 10^6 \right\}$ for
$n=\left\{10,30,80\right\}$, respectively. For the hypervolume calculations, a
reference point at $[0,0]$ is considered.

\noindent\textbf{Experimentation/Visualization.}
Next, we shall consider the performance of the three EMOA on the given Pareto problems,
considering the hypervolume indicator as the performance criterion.
Based on the analytical expressions of the Pareto front for this problem, given in Appendix \ref{app:ParetoProof},
the hypervolume of the true front is $\mho^{*}=0.47482$.
Table \ref{tab:diffractionhypervolume} presents the mean and standard-deviations of the hypervolume calculations over 30 runs
of the attained Pareto fronts for the various test-cases. The table contains the hypervolume values for the perceived fronts, as well as
for the noise-free \emph{a posteriori} fronts. Table \ref{tab:utest} provides the Mann-Whitney U-test calculations for the pairwise
algorithm comparisons corresponding to the test-cases of Table \ref{tab:diffractionhypervolume}.
\begin{table}
\begin{center}
\caption{Hypervolume Calculations: the Grating Diffraction Landscape ($b=1$, $h=4$; $q_0=0$, $q=0.5\cdot T_q$; $n=\left\{10,30,80 \right\}$).
Mean and standard-deviations over 30 runs, considering a reference point at $(0,0)$.Based on the analytical expressions of the Pareto front for this problem, given in Appendix \ref{app:ParetoProof}, the hypervolume of the true front is $\mho^{*}=0.47482$.\label{tab:diffractionhypervolume}}
\tiny
{\bf n=10}\\
\begin{tabular}{|c||c|c||c|c||c|c|}
\hline
Noise &\multicolumn{2}{c||}{MO-CMA-ES}&\multicolumn{2}{c||}{SMS-EMOA}&\multicolumn{2}{c|}{NSGA-II}\\
\cline{2-7}
Strength &perceived & \emph{a posteriori} & perceived & \emph{a posteriori} & perceived & \emph{a posteriori} \\
\hline \hline
 $\epsilon_S^2 = 0$ &  {\bf 0.47476}$\pm$0.0001 & 0.47476$\pm$0.0001 & 0.47443$\pm$0.0004 & 0.47443$\pm$0.0004 & 0.31258$\pm$0.0669 & 0.31258$\pm$0.0669 \\
\hline
$\epsilon_S^2 = 0.001$ & 0.{\bf 47420}$\pm$0.0001 & 0.47420 $\pm$ 0.0004 & 0.47339$\pm$0.0016 & 0.47219$\pm$0.0025 & 0.39245$\pm$0.0472 & 0.37274$\pm$0.0575 \\
\hline
$\epsilon_S^2 = 0.005$  & {\bf 0.47398}$\pm$0.0002 & 0.47158$\pm$0.0031 & 0.47128$\pm$0.0039 & 0.46473$\pm$0.0082 & 0.35833$\pm$0.0584 & 0.34559$\pm$0.0638 \\
\hline
 $\epsilon_S^2 = 0.01$ & {\bf 0.47362}$\pm$0.0007 & 0.46682$\pm$0.0069 & 0.46972$\pm$0.0042 & 0.45733$\pm$0.0116 & 0.39245$\pm$0.0472 & 0.37274$\pm$0.0575 \\
\hline
$\epsilon_S^2 = 0.02$ & {\bf 0.47316}$\pm$0.0006 & 0.46285$\pm$0.0091 & 0.47018$\pm$0.0024 & 0.45215$\pm$0.0104 & 0.40789$\pm$0.0357 & 0.37518$\pm$0.0504 \\
\hline
$\epsilon_S^2 = 0.05$ & {\bf 0.47168}$\pm$0.0008 & 0.44705$\pm$0.0158 & 0.46488$\pm$0.0049 & 0.42055$\pm$0.0259 & 0.42755$\pm$0.0317 & 0.38169$\pm$0.0433 \\
\hline
\end{tabular}
\\ {\bf n=30} \\
\begin{tabular}{|c||c|c||c|c||c|c|}
\hline
Noise &\multicolumn{2}{c||}{MO-CMA-ES}&\multicolumn{2}{c||}{SMS-EMOA}&\multicolumn{2}{c|}{NSGA-II}\\
\cline{2-7}
Strength &perceived & \emph{a posteriori} & perceived & \emph{a posteriori} & perceived & \emph{a posteriori} \\
\hline \hline
$\epsilon_S^2 = 0$ & 0.43685$\pm$0.0448 & 0.43685$\pm$0.0448 & {\bf 0.46864}$\pm$0.0034 & 0.46864$\pm$0.0034 & 0.24917$\pm$0.0291 & 0.24917$\pm$0.0291 \\
\hline
$\epsilon_S^2 = 0.001$ & 0.40817$\pm$0.0622 & 0.40395$\pm$0.0647 & {\bf 0.46080}$\pm$0.0123 & 0.45902$\pm$0.0134 & 0.30131$\pm$0.0347 & 0.28343$\pm$0.0372 \\
\hline
$\epsilon_S^2 = 0.005$ & 0.42474$\pm$0.0435 & 0.41412$\pm$0.0482 & {\bf 0.44955}$\pm$0.0117 & 0.44130$\pm$0.0152 & 0.28015$\pm$0.0428 & 0.26685$\pm$0.0435 \\
\hline
 $\epsilon_S^2 = 0.01$ & 0.40719$\pm$0.0529 & 0.38921$\pm$0.059 & {\bf 0.44051}$\pm$0.0128 & 0.42709$\pm$0.0177 & 0.30131$\pm$0.0347 & 0.28343$\pm$0.0372 \\
\hline
$\epsilon_S^2 = 0.02$ & 0.41905$\pm$0.0435 & 0.39627$\pm$0.0531 & {\bf 0.42513}$\pm$0.0207 & 0.39920$\pm$0.0316 & 0.32415$\pm$0.0362 & 0.29869$\pm$0.0377 \\
\hline
$\epsilon_S^2 = 0.05$ & 0.40997$\pm$0.0392 & 0.37579$\pm$0.0466 & {\bf 0.41614}$\pm$0.0139 & 0.37495$\pm$0.0240 & 0.35147$\pm$0.0252 & 0.31039$\pm$0.0335 \\
\hline
\end{tabular}
\\ {\bf n=80} \\
\begin{tabular}{|c||c|c||c|c||c|c|}
\hline
Noise &\multicolumn{2}{c||}{MO-CMA-ES}&\multicolumn{2}{c||}{SMS-EMOA}&\multicolumn{2}{c|}{NSGA-II}\\
\cline{2-7}
Strength &perceived & \emph{a posteriori} & perceived & \emph{a posteriori} & perceived & \emph{a posteriori} \\
\hline \hline
 $\epsilon_S^2 = 0$ & 0.35875$\pm$0.0515 & 0.35875$\pm$0.0515 & {\bf 0.45796}$\pm$0.0104 & 0.45796$\pm$0.0104 & 0.20782$\pm$0.0440 & 0.20782$\pm$0.0440 \\
\hline
$\epsilon_S^2 = 0.001$ & 0.27607$\pm$0.0451 & 0.27123$\pm$0.0452 & {\bf 0.44408}$\pm$0.0146 & 0.44289$\pm$0.0152 & 0.23867$\pm$0.0388 & 0.22717$\pm$0.0398 \\
\hline
$\epsilon_S^2 = 0.005$ & 0.26278$\pm$0.0458 & 0.25099$\pm$0.0460 & {\bf 0.42345}$\pm$0.0146 & 0.41800$\pm$0.0170 & 0.24902$\pm$0.0343 & 0.24310$\pm$0.0347 \\
\hline
 $\epsilon_S^2 = 0.01$ & 0.25462$\pm$0.0327 & 0.23735$\pm$0.0329 & {\bf 0.40138}$\pm$0.0198 & 0.39219$\pm$0.0222 & 0.23867$\pm$0.0388 & 0.22717$\pm$0.0398 \\
\hline
$\epsilon_S^2 = 0.02$ & 0.25478$\pm$0.0482 & 0.23163$\pm$0.0463 & {\bf 0.38329}$\pm$0.0245 & 0.36870$\pm$0.0292 & 0.25967$\pm$0.0350 & 0.24304$\pm$0.0378 \\
\hline
$\epsilon_S^2 = 0.05$ & 0.22611$\pm$0.0358 & 0.19538$\pm$0.0369 & {\bf 0.34807}$\pm$0.0250 & 0.32100$\pm$0.0308 & 0.28492$\pm$0.0333 & 0.25748$\pm$0.0386 \\
\hline
\end{tabular}
\end{center}
\normalsize
\end{table}
\begin{table}

\begin{center}
\caption{Mann-Whitney U-Test Calculations: the Grating Diffraction Problem ($b=1$, $h=4$; $q_0=0$, $q=0.5\cdot T_q$; $n=\left\{10,30,80 \right\}$).
A comparison is drawn from the numerical results of the 3 algorithms in the various test-cases,
considering a null hypothesis $H_0$ stating that there is no performance difference in terms of the attained hypervolume,
versus a hypothesis $H_1$ stating that two algorithms have significantly different performance.
Accordingly, a table symbol of $\pm$ indicates a rejection of the null hypothesis at the 5\% significance level,
whereas a symbol of $\approx$ indicates a failure to reject the null hypothesis at the 5\% significance level.
$+$ refers to a statistically significant outperformance of the left-side algorithm over the right-side algorithm, and $-$ indicates the reverse scenario.\label{tab:utest}}
\scriptsize
{\bf n=10}\\
\begin{tabular}{|c||c|c|c||c|c|c|}
\hline
Noise &\multicolumn{3}{c||}{perceived}&\multicolumn{3}{c|}{\emph{a posteriori}}\\
\cline{2-7}
Strength &CMA/SMS & CMA/NSGA-II & SMS/NSGA-II &  CMA/SMS & CMA/NSGA-II & SMS/NSGA-II\\
\hline \hline
 $\epsilon_S^2 = 0$ & $+$ & $+$ & $+$ & $+$ & $+$ & $+$ \\
\hline
$\epsilon_S^2 = 0.001$ & $\approx$ & $+$ & $+$ & $+$ & $+$ & $+$ \\
\hline
$\epsilon_S^2 = 0.005$ & $+$ & $+$ & $+$ & $+$ & $+$ & $+$ \\
\hline
 $\epsilon_S^2 = 0.01$ & $+$ & $+$ & $+$ & $+$ & $+$ & $+$ \\
\hline
$\epsilon_S^2 = 0.02$ & $+$ & $+$ & $+$ & $+$ & $+$ & $+$ \\
\hline
$\epsilon_S^2 = 0.05$ & $+$ & $+$ & $+$ & $+$ & $+$ & $+$ \\
\hline
\end{tabular}
\\{\bf n=30}\\
\begin{tabular}{|c||c|c|c||c|c|c|}
\hline
Noise &\multicolumn{3}{c||}{perceived}&\multicolumn{3}{c|}{\emph{a posteriori}}\\
\cline{2-7}
Strength &CMA/SMS & CMA/NSGA-II & SMS/NSGA-II &  CMA/SMS & CMA/NSGA-II & SMS/NSGA-II\\
\hline \hline
 $\epsilon_S^2 = 0$ & $\approx$ & $+$ & $+$ & $\approx$ & $+$ & $+$ \\
\hline
$\epsilon_S^2 = 0.001$ & $-$ & $+$ & $+$ & $-$ & $+$ & $+$ \\
\hline
$\epsilon_S^2 = 0.005$  & $\approx$ & $+$ & $+$ & $\approx$ & $+$ & $+$ \\
\hline
 $\epsilon_S^2 = 0.01$ & $-$ & $+$ & $+$ & $-$ & $+$ & $+$\\
\hline
$\epsilon_S^2 = 0.02$ & $\approx$ & $+$ & $+$ & $\approx$ & $+$ & $+$ \\
\hline
$\epsilon_S^2 = 0.05$ & $\approx$ & $+$ & $+$ & $\approx$ & $+$ & $+$ \\
\hline
\end{tabular}
\\{\bf n=80}\\
\begin{tabular}{|c||c|c|c||c|c|c|}
\hline
Noise &\multicolumn{3}{c||}{perceived}&\multicolumn{3}{c|}{\emph{a posteriori}}\\
\cline{2-7}
Strength &CMA/SMS & CMA/NSGA-II & SMS/NSGA-II &  CMA/SMS & CMA/NSGA-II & SMS/NSGA-II\\
\hline \hline
 $\epsilon_S^2 = 0$ & $-$ & $+$ & $+$ & $-$ & $+$ & $+$ \\
\hline
$\epsilon_S^2 = 0.001$ & $-$ & $+$ & $+$ & $-$ & $+$ & $+$ \\
\hline
$\epsilon_S^2 = 0.005$ & $-$ & $\approx$ & $+$ & $-$ & $\approx$ & $+$ \\
\hline
 $\epsilon_S^2 = 0.01$ & $-$ & $\approx$ & $+$ & $-$ & $\approx$ & $+$ \\
\hline
$\epsilon_S^2 = 0.02$ & $-$ & $\approx$ & $+$ & $-$ & $\approx$ & $+$ \\
\hline
$\epsilon_S^2 = 0.05$ & $-$ & $-$ & $+$ & $-$ & $-$ & $+$ \\
\hline
\end{tabular}
\end{center}
\normalsize
\end{table}

\subsubsection{Discussion}
Given the numerical results in Table \ref{tab:diffractionhypervolume} and the statistical tests in Table \ref{tab:utest},
we suggest the following observation: while the MO-CMA achieves superior hypervolume values on the 10-dimensional case,
there is no clear winner on the 30-dimensional case (see U-tests), and finally,
the SMS-EMOA is the winner on the 80-dimensional cases.
In the vast majority of the cases, the NSGA-II is outperformed by its competitors.

We speculate whether the poor performance of the MO-CMA in the high-dimensional cases in comparison to the SMS-EMOA
is due to an insufficient budget of function evaluations.
Upon granting the MO-CMA additional function evaluations for the high-dimensional cases this speculation
is indeed corroborated. We carried out $30$ independent runs for the noise-free test cases of $n=30$ and $n=80$,
with $10$ times the original budget of function evaluations, i.e., with $2\cdot 10^7$ and $5\cdot 10^7$ function evaluations, respectively.
For the $n=30$ test-case the MO-CMA obtained a mean hypervolume value of $0.47465$,
whereas for the $n=80$ test-case it obtained a mean hypervolume value of $0.46089$.
Fig.\ \ref{fig:30dres80dres_boxplots} depicts statistical box-plots describing the miscellaneous
runs granting the MO-CMA additional function evaluations for the high-dimensional noise-free cases,
presenting the attained hypervolume values at specific milestones along the runs.
As stated earlier, it is indeed shown that the MO-CMA is \emph{slower} than the SMS-EMOA for those problems,
but it is capable of eventually converging to a good approximate front, given sufficient function evaluations.

The empirically observed slow \emph{progress rate} may be attributed to the self-adaptation mechanism which is typically responsible
for the relatively long learning period of the CMA-ES when compared to other strategies, e.g., ES with fewer strategy parameters \cite{HansenDR2,HansenDR2PPSN08}.
Overall, it seems that employing the strong search-engine of the CMA does not pay off on the
Diffraction Grating problem upon consideration of the reduced convergence speed in comparison to the SMS-EMOA.

\begin{figure*}
\begin{multicols}{2}
\centering \includegraphics[scale=0.4]{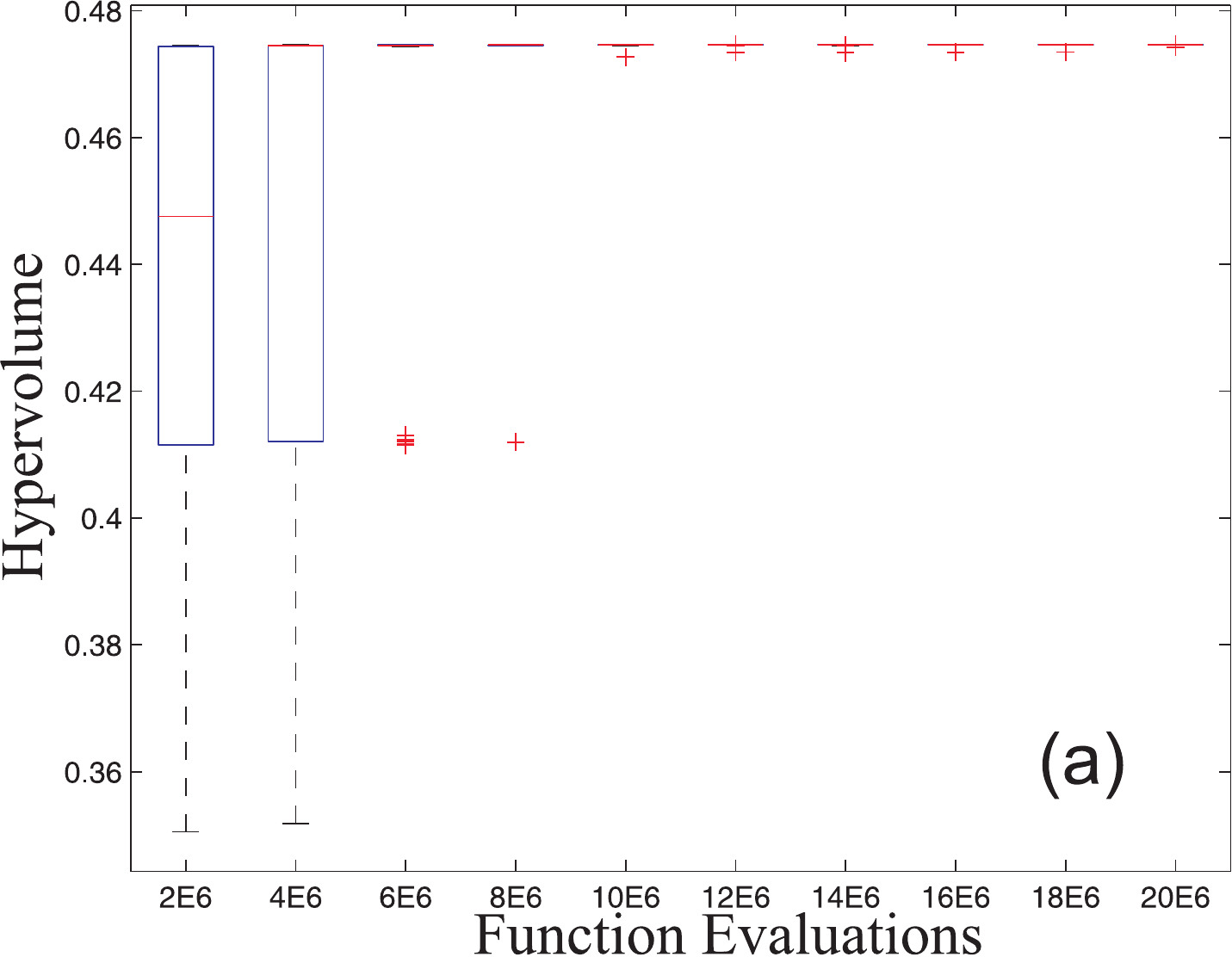}
\newpage
\centering \includegraphics[scale=0.4]{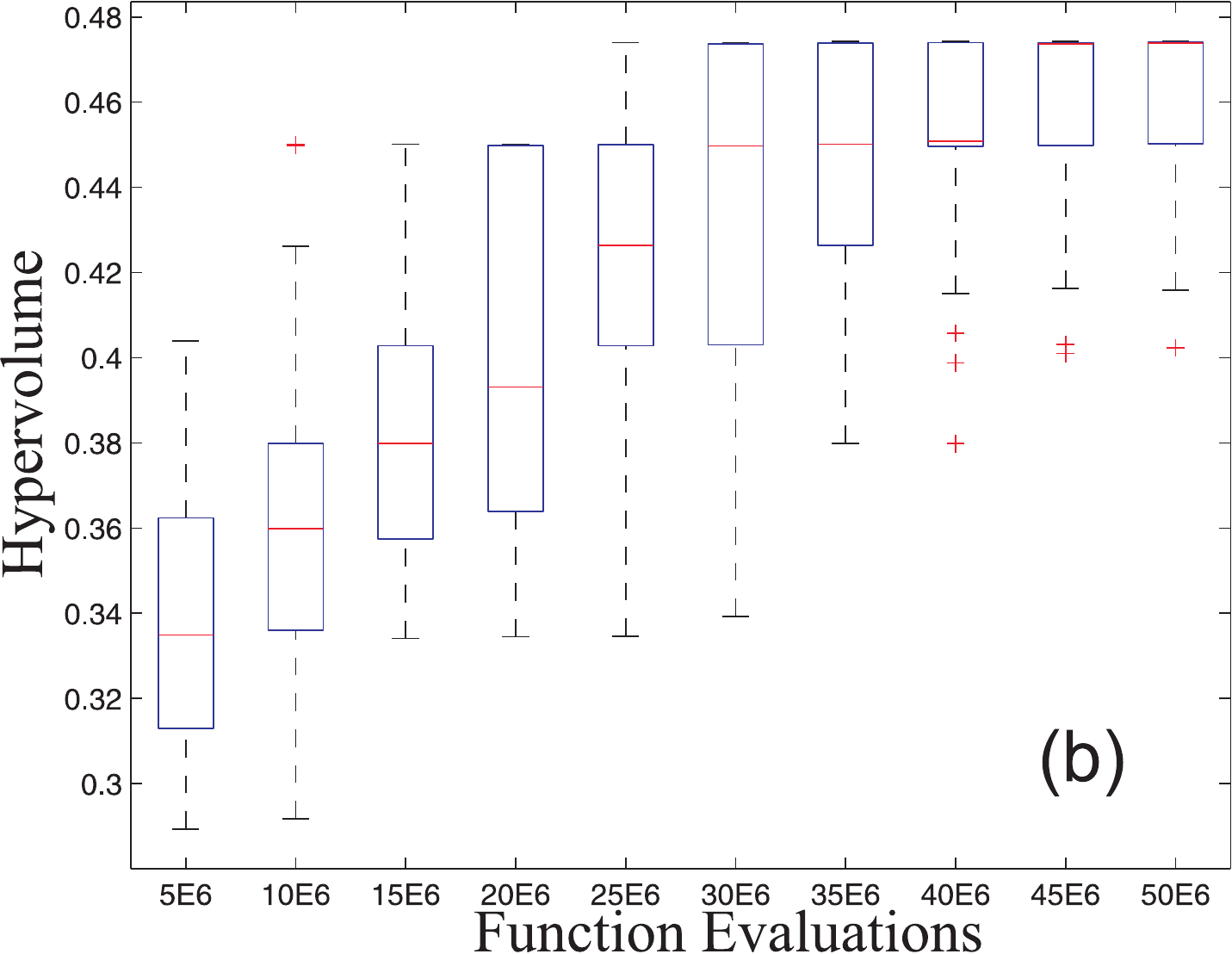}
\end{multicols}
\caption
{Granting the MO-CMA additional function evaluations for the noise-free Diffraction Grating problem.
Statistical box-plots of 30 independent runs, presenting the attained hypervolume values at specific milestones along
the run, with up to 10 times the original budget of function evaluations.
[LEFT, (a)] $n=30$, where the algorithm obtains the maximally attainable hypervolume in all runs, without exception,
after $2\cdot 10^7$ function evaluations;
[RIGHT, (b)] $n=80$, where the majority of the runs obtain the maximally attainable hypervolume after
$5\cdot 10^7$ function evaluations.\label{fig:30dres80dres_boxplots}}
\end{figure*}

\subsection{Molecular Alignment Simulations}
\label{sec:obs_align}
We consider the detailed effect of pixel noise on the quantum observables and the overall MO-CMA performance.

\noindent\textbf{Performance Assessment.}
In the context of molecular alignment (Eq.\ \ref{eq:alignment}), $f_1$ is of particular interest, and thus is considered as the primary objective.
The maximally attainable theoretical upper bound that can be supported by the utilized rotational states used here was found to be 0.9863
\cite{Shir-JPhysB}, but the best known single-objective $f_1$ yield within the current bandwidth discovered by an ES was
reported to be 0.962 \cite{Shir-JPhysB}, with a corresponding $f_2$ value of 0.154.
The nature of the conflict between $f_1$ to $f_2$ is generally unknown, and we shall use our noise-free runs
as a reference Pareto front for the runs on noisy systems.

\noindent\textbf{Setup.} Due to computational limitations, we set a limit of 10 runs per test-case.
Preliminary runs of MO-CMA, SMS-EMOA and NSGA-II on the noise-free simulation are carried out as an introductory comparison.
Furthermore, we will take into account systems with noisy controls subject to the noise strength values of
Eq.\ \ref{eq:noisestrength}. Each run is limited to 1000 iterations.

\begin{figure}
\centering \includegraphics[scale=0.5]{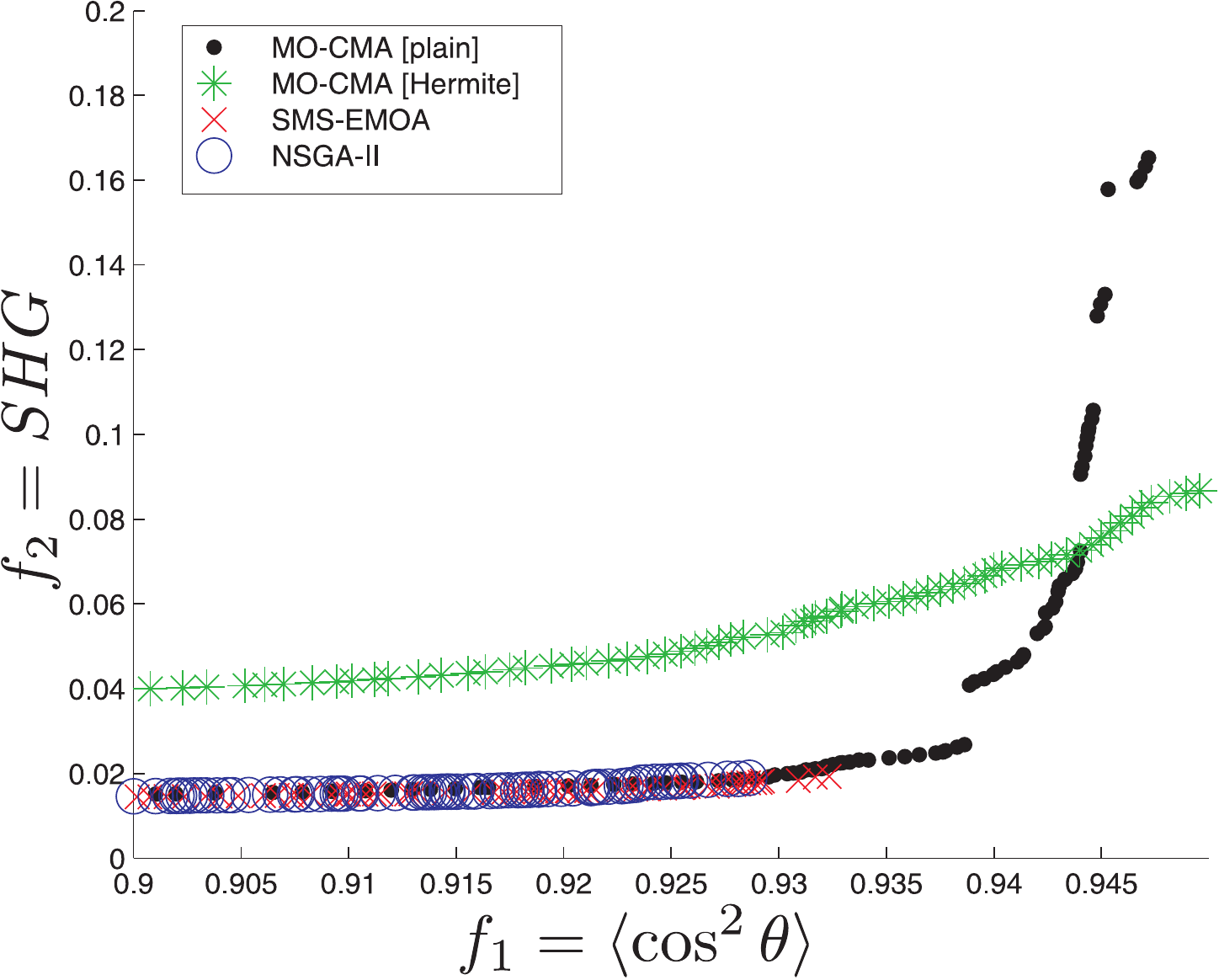}
\caption{Attained Pareto fronts on the noise-free molecular alignment simulation of 4 EMOA routines:
MO-CMA with 'plain' parametrization (decision variables are directly addressed as phase points),
MO-CMA with Hermite parametrization (decision variables correspond to coefficients of the first 40 Hermite polynomials,
spanning altogether the phase), SMS-EMOA, and NSGA-II.
Each front is reconstructed of 10 runs per routine.
\label{fig:plain-hermite}}
\end{figure}
\noindent\textbf{Preliminary: EMOA Noise-Free Comparison.} The noise-free runs
yielded disconnected local Pareto fronts, which offered limited coverage of the
objective space per run. This may suggest that the search space is broken into
separate regions, partitioned by barriers, possibly due to the inherent
constraints on the system, e.g., the bandwidth, the discretization, etc. We
reconstructed a single Pareto front from these runs, referred to here as the
\emph{best known Pareto front}. The shape of the attained front indicates that
the conflict is rather soft, as high $f_1$ values may be obtained while keeping
$f_2$ values extremely low. There seems to be no considerable pay-off in $f_1$
when unleashing $f_2$. Furthermore, from a practical perspective one may argue
that this conflict is irrelevant, as the observed $f_2$ values are sufficiently
low. It should also be noted that $f_1$ values of $\approx$0.96 could not be
attained in these runs; the best obtained value was $f_1^{*}=0.947$,
corresponding to $f_2^{*}=0.165$. This observation may be linked to previous
reports on the single-objective CMA-ES applied to this problem
\cite{Shir-JPhysB}, investigating its performance in maximizing $f_1$ subject
to various parametrizations. In particular, the so-called 'plain'
parametrization, where the decision variables correspond to the phase function
pixels, was observed to be inferior in comparison to specific polynomial-based
configurations, where the decision variables played the role of coefficients of
complete-basis functions. In \cite{Shir-JPhysB}, following an empirical
comparison, the Hermite polynomials were reported to perform best. Here, we
carried out additional calculations, employing the Hermite parametrization, in
order to assess the latter observation. The results, which are depicted in
Fig.\ \ref{fig:plain-hermite}, generalize the observation reported in
\cite{Shir-JPhysB} into the bi-criteria picture, confirming that the MO-CMA is
capable of attaining $f_1$ values of $\approx$0.96 when special configurations
are in use. Moreover, it confirms that the inherent advantage of the Hermite
parametrization in terms of $f_1$ values translates into a trade-off with
slightly higher $f_2$ values. Concerning the competing SMS-EMOA and NSGA-II
algorithms, it is clearly observed that they present inferior performance,
especially with respect to the coverage of $f_1$ values. In total, their
results are disappointing, but at the same time are in some consistency with
previous observations on a different variant of this problem (see, e.g.,
\cite{SHIR-AlignMO}).


\noindent\textbf{Observation: MO-CMA on Noisy Systems.} In what follows, we
consider the MO-CMA alone on the noisy alignment problem. When subject to
noise, the MO-CMA \emph{seems} to perform well, especially with its default
procedure, in obtaining fair Pareto fronts, in comparison to the noise-free
simulations. As in the noise-free case, the attained fronts were typically
broken, and we reconstructed them into a single front for their presentation.
In some cases, the perceived Pareto fronts of the noisy system dominated the
best known front, and the \emph{a posteriori} noise-free evaluation of the
archived phase functions introduced a local improvement to the best known
front. This is an example of a scenario in which fitness overvaluation has the
potential to enhance the search. However, the reproduction of the Pareto front
by evaluating the Pareto optimal set typically failed, suggesting that decision
space information was lost, as was observed on the model landscape. Fig.\
\ref{fig:alignmo_fronts}[a] depicts the attained front of the default MO-CMA
procedure in a noisy system of $\epsilon_S^2=0.01$. The plot contains the
reconstructed Pareto front of 10 runs, the best known front, the \emph{a
posteriori} noise-free evaluation of the Pareto optimal set, as well as the
noisy sampling of the Pareto optimal set. Close examination of the \emph{a
posteriori} sampled data and their grouping towards the perceived front reveals
interesting insight into the noise propagation through the two objective
functions (Fig.\ \ref{fig:alignmo_fronts}[b]). It is evident that noisy
sampling of a phase function corresponding to a point on the perceived front
results in an elliptic cloud of points, whose elitist outliers constitute the
points of the perceived front, as in the model landscapes (see, e.g., Fig.\
\ref{fig:clusteringeffect}). Also, it is clear that these clouds have a
dominant horizontal axis in the current scaling. This observation suggests that
the alignment observable ($f_1$) is sensitive to noise, unlike SHG ($f_2$),
which is hardly affected by it at the current noise level. Moreover, the shape
of these clouds seems to be dependent upon the two objective values through a
multiplicative relation: points with low $f_1$ values possess a longer
horizontal axis and a shorter vertical axis in comparison to points with higher
$f_1$ values.
\begin{figure}[t]
\centering \includegraphics[scale=0.6]{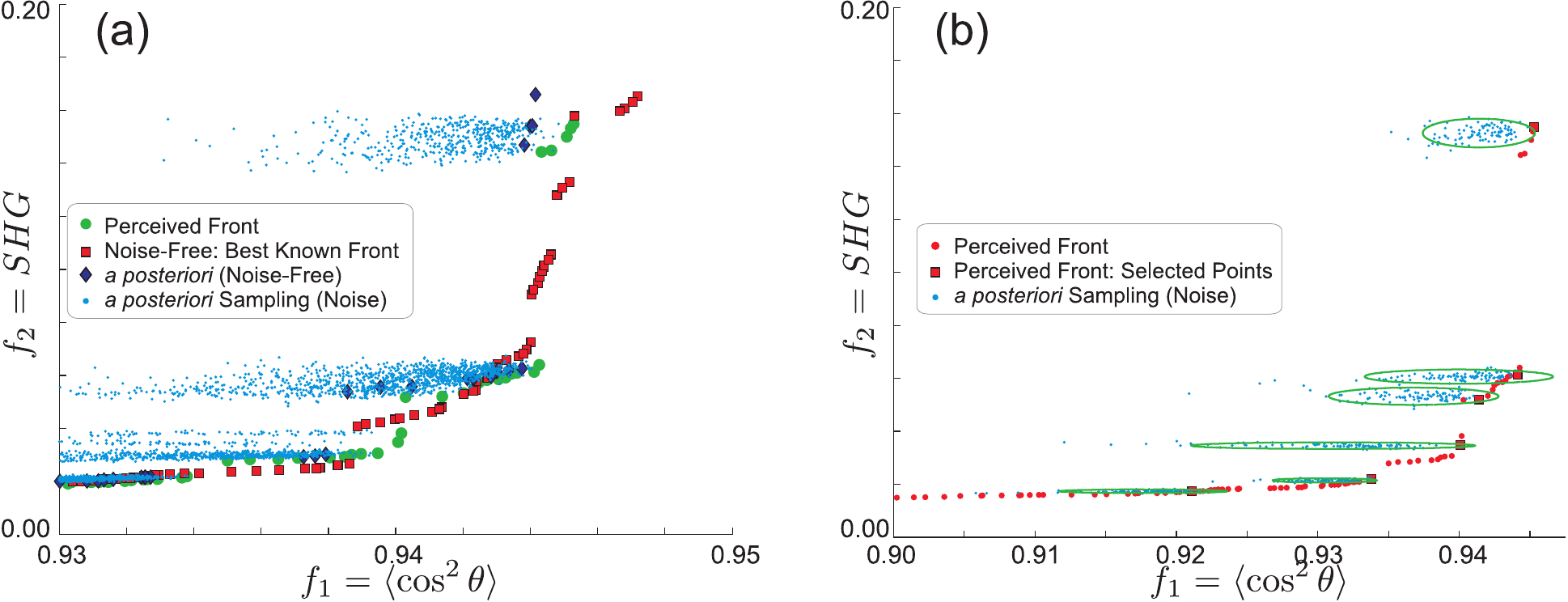} 
\caption{The attained reconstructed Pareto front of the default MO-CMA on the noisy
alignment simulation with $\epsilon_S^2=0.01$. [LEFT (a)] The
\emph{reconstructed} perceived front of 10 runs, the best known front, the
\emph{a posteriori} noise-free evaluation of the Pareto optimal set, and the
\emph{a posteriori} noisy sampling of the Pareto optimal set. [RIGHT (b)]
Statistical examination of the \emph{a posteriori} sampling of selected points
of the Pareto optimal set. Each sampling set comprises 100 evaluations at the
noise level of $\epsilon_S^2=0.01$. The ellipses represent the disturbance
distributions, centered about the mean with twice the standard deviations as
axes, based upon {\bf statistics} of the attained data. As in the model
landscapes (see, e.g., Fig.\ \ref{fig:clusteringeffect}), the perceived front
constitutes an elitist selection of these distributions. The reader should mind
the different horizontal scaling of the two panels.\label{fig:alignmo_fronts}}
\end{figure}

It should be noted that the simulations at higher noise levels obtained reasonable Pareto fronts in comparison
to the noise-free best known front, but their reproduction by means of evaluation with the attained Pareto optimal set failed,
as found on the Multi-Sphere model landscape.
The simulations also revealed that the two procedures with additional parental fitness re-evaluations
produced Pareto fronts of low quality, as they were typically dominated by the default MO-CMA procedure.
In some cases, however, it is evident that local \emph{a posteriori} Pareto fronts of the procedure with occasional parental re-evaluation
locally dominated the equivalent fronts of the default procedure.
Overall, there is no clear superior procedure in this test-case.

\subsection{Laboratory Experiment I: Molecular Ion Generation}
\label{sec:obs_lab} An experimental Pareto front for the \emph{molecular ion
generation} system is depicted in Fig.\ \ref{fig:ionmo_front}. The shape of the
front has been assessed with high confidence, based on numerous runs of the
single-objective $\left( \mu , \lambda \right)$-CMA-ES on the corresponding
tailored \emph{ratio} objective function, i.e.,
$\frac{\mathcal{J}_{Ion}}{SHG^{\alpha}}$. We therefore conclude that the MO-CMA
obtained a perceived Pareto front consistent with the repeated aforementioned
single-objective optimization results, but nevertheless, its reconstruction by
means of the attained Pareto set was not successful, as observed with both the
Multi-Sphere and the molecular alignment problems. It is evident in Fig.\
\ref{fig:ionmo_front} that while the perceived Pareto front dominates the
unshaped control reference front, the mean values of the \emph{a posteriori}
sampling of the Pareto set produces a dramatically worse front, which is
\emph{Pareto indifferent} to the reference front. In addition, the attractive
\emph{knee point} (roughly located around coordinate (0.425,0.2)) could not be
reconstructed, and its information was practically lost. Upon consideration of
the experimental data, the perceived point appears to be an experimental
outlier, which dominated a converging local Pareto front in that domain and
contributed to its loss. However, it is crucial to note that this specific
\emph{knee area} represents a real domain of solutions which has been
identified in repeated occasions, whose Pareto coverage is much needed.
Repeating runs by means of alternative strategies introduced an experimental
overhead, and therefore was not carried out. The second QCE system, Molecular
Plasma Generation, possesses higher experimental stability, granted by the
different experimental design. It has therefore been targeted as a platform for
testing the re-evaluation approach and thus to address the issues revealed with
the current experimental system. Moreover, it allowed for a comparison between
various strategies, as will be described in the following section.
\begin{figure}
\centering \includegraphics[scale=1.0]{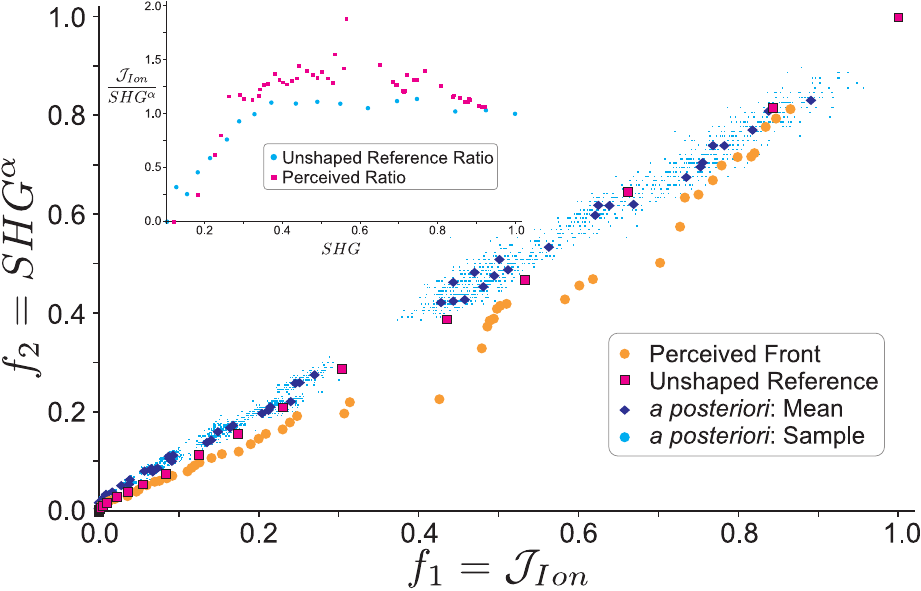}
\caption{Experimental Pareto front of the default MO-CMA on the total
ionization $\mathcal{J}_{Ion}$ versus SHG problem. The figure depicts the
perceived front of a single experiment, the reference front of the intensity
based non-shaped pulse, as well as a sampling of the Pareto set. Inset:
single-objective ratio picture. \label{fig:ionmo_front}}
\end{figure}

\subsection{Laboratory Experiment II: Molecular Plasma Generation}
\label{sec:obs_lab2} Taking advantage of the experimental stability of this
system, we carried out a Pareto optimization campaign by means of the EMOA
considered in the current study. In particular, we compared the experimental
performance of the MO-CMA (default and with occasional parental re-evaluation),
to the NSGA-II and the SMS-EMOA. The observation here is clear, as well as
consistent with the previous observations on the other systems: The default
MO-CMA produced highly-satisfying perceived fronts, but suffered from an
inability to reproduce them upon the termination of the runs. The NSGA-II, on
the other hand, performed poorly, and failed to obtain good approximations to
the Pareto front. The remaining strategies, MO-CMA with occasional
re-evaluation and the SMS-EMOA, both performed well -- the attained approximate
fronts were satisfying, and their post-reproduction was successful. Fig.\
\ref{fig:rodd_front} presents successful runs of both strategies, depicting the
perceived fronts, their reproduction, and the unshaped reference fronts
(measured upon scanning the amplitude of an unshaped pulse). Since the latter
represents a trivial reference to pulse shaping, and especially to any QC
optimization scheme, we argue that the QC optimization pay-off in the
multi-objective case may be assessed by the calculation of the hypervolume
ratio between the attained front to the unshaped reference front. Overall, the
MO-CMA with occasional parental re-evaluation performed best, introducing a
hypervolume improvement of 24.5\% with respect to the unshaped reference. The
SMS-EMOA, on the other hand, introduced an insignificant improvement of merely
3\%, due to bad coverage. The success of the occasional re-evaluation scheme
within the MO-CMA proved to be especially beneficial in this case, and thus
constitutes an experimental corroboration to the conclusions drawn on noisy
model landscapes (see Section \ref{sec:obs_spheres}). Fig.\ \ref{fig:hvol_rodd}
depicts the evolving hypervolume pay-off of the MO-CMA population -- presented
as the ratio between the raw MO-CMA hypervolume to the hypervolume of the
unshaped reference front -- corresponding to the run presented in Fig.\
\ref{fig:rodd_front}[(a)]. For the hypervolume calculations, a reference point
at $[0,1]$ is considered. The initial high values of the ratio around 0.9 are a
consequence of planting seed solutions in the initial population. It can be
clearly observed that the occasional re-evaluation (every 10 generations)
introduces corrections to fitness disturbances of the parental population that
translate into hypervolume declines. In particular, note the dramatic decline
following the re-evaluation of generation 30 -- had not this correction
occurred, the parental population would have been contaminated by extreme
outliers and the run would have been affected accordingly. At the same time,
the re-evaluation scheme does not hamper the general trend of hypervolume
increase, and thus offers an efficient solution to the previously reported
problem. We therefore conclude that this \emph{self-correcting} property of the
occasional re-evaluation scheme is essential for \emph{experimental} scenarios.
\begin{figure}[t]
\centering \includegraphics[scale=0.5]{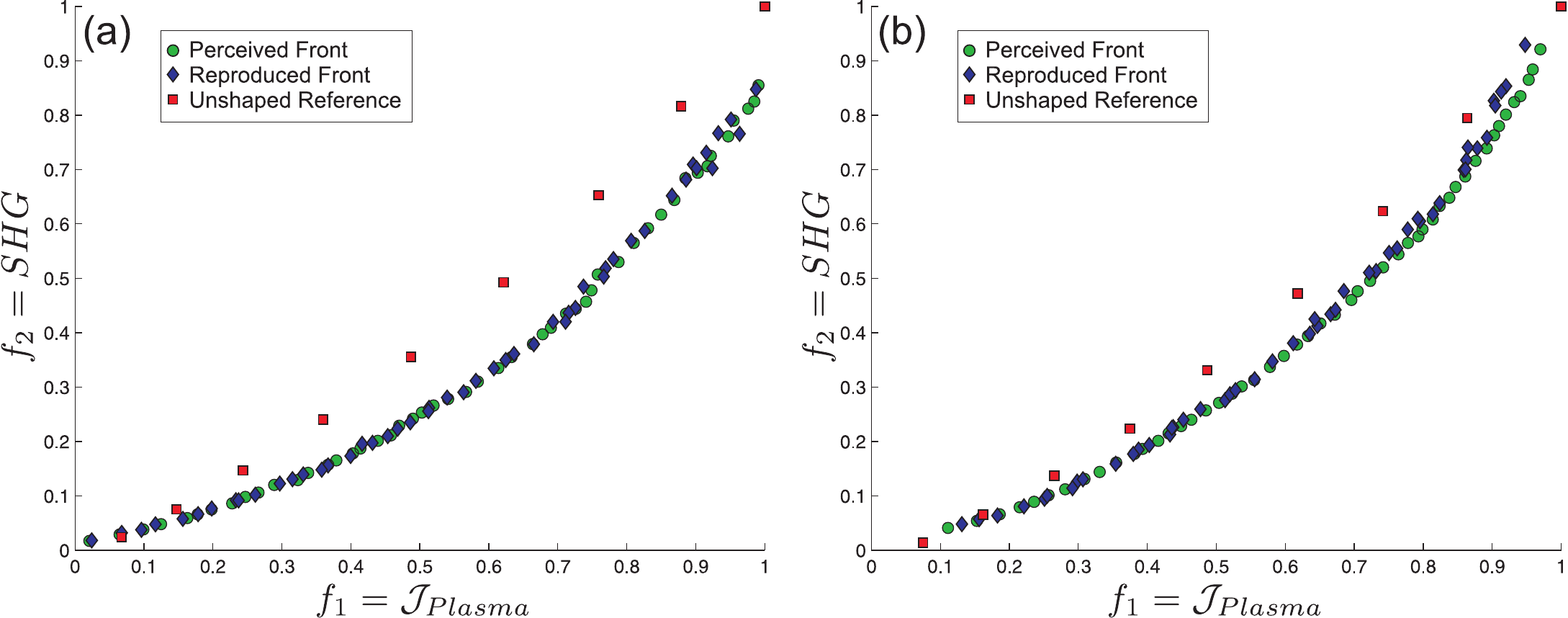}
\caption{Experimental Pareto fronts for the Molecular Plasma Generation problem
(maximizing free electron number $\mathcal{J}_{Plasma}$ versus minimizing
$SHG$), for the MO-CMA with occasional re-evaluation [LEFT, (a)] and for the
SMS-EMOA [RIGHT, (b)]. Each figure depicts the perceived front of a single
experiment, the reference front of the intensity based non-shaped pulse, as
well as the reproduction of the Pareto optimal set upon the termination of the
run. \label{fig:rodd_front}}
\end{figure}
\begin{figure}[t]
\centering \includegraphics[scale=0.5]{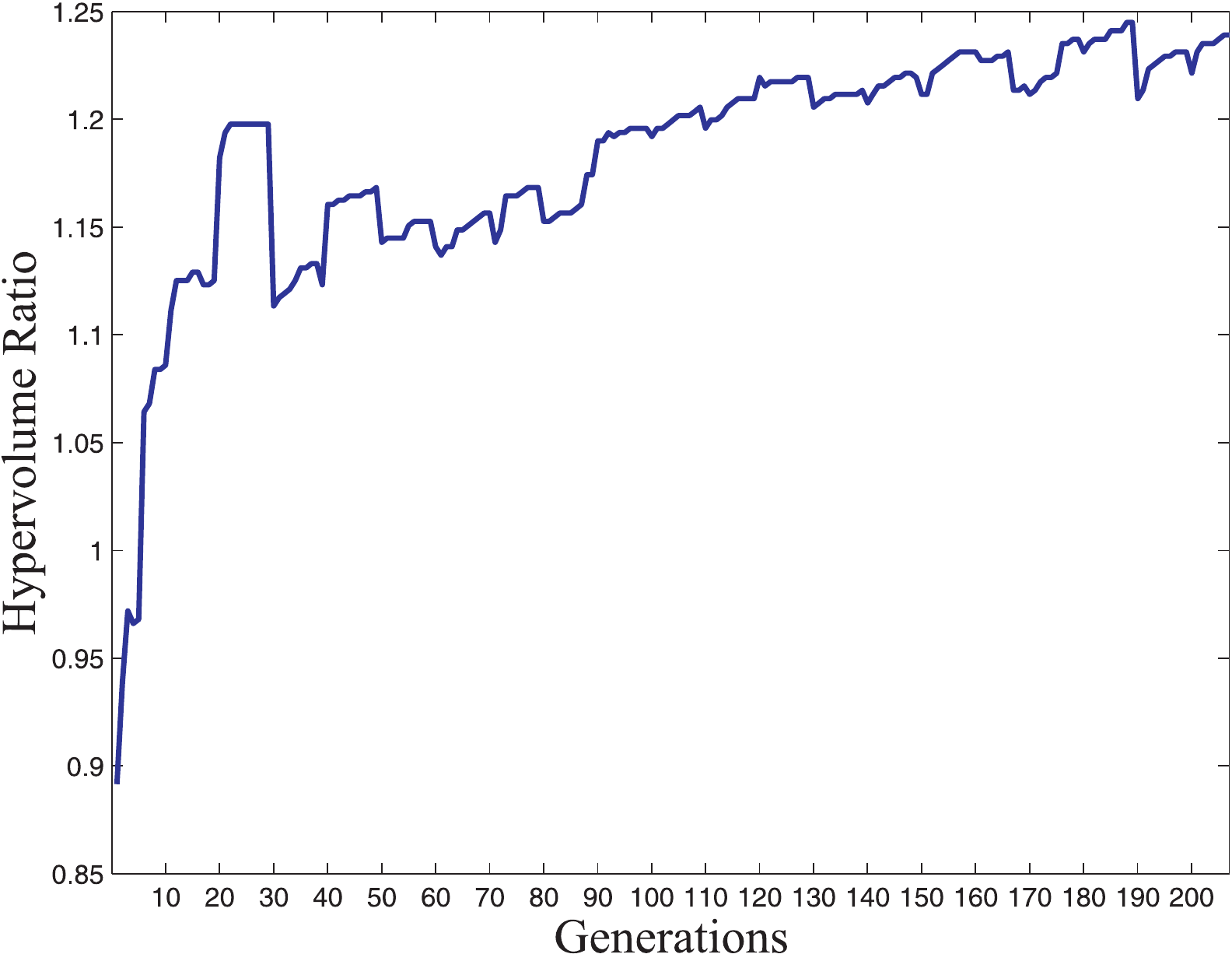} 
\caption{The evolving hypervolume pay-off of the parental population of the MO-CMA with
re-evaluation every 10 generations, with respect to the unshaped reference
front, corresponding to the run depicted in Fig.\ \ref{fig:rodd_front}[(a)].
The periodic re-evaluation corrects fitness disturbances within the parental
population, and causes the occasional hypervolume declines. It does not,
however, hamper the general trend of hypervolume increase in the course of the
entire run. For the hypervolume calculations, a reference point at $[0,1]$ was
considered.\label{fig:hvol_rodd}}
\end{figure}

\section{Summary}
\label{sec:discussion}
This paper introduced the topic of Multi-Observable Quantum Control and promoted its platform as a testbed for
evolutionary experimental multi-objective optimization. It discussed various practical issues concerning this experimental domain,
such as the sources of noise and uncertainty, and predominantly considered the MO-CMA as the optimization method.
Several frameworks were targeted for testing -- two noisy model landscapes, as well as multiple QC systems:
one simulated and two experimental.
Towards this end, we introduced here a family of test-functions,
originating from the \emph{optical} domain of Diffraction Grating problems,
which can provide model landscapes for Pareto optimization. Their attractiveness lies particularly within the simple,
yet full mathematical formulation as well as within the practical linkage to real-world experiments.
Overall, this effort constituted a broad study of the MO-CMA, subject to fitness disturbance of noisy decision parameters on simulated systems,
and its deployment in QC laboratory experiments.

While the MO-CMA excels in Pareto optimization of noise-free model landscapes, it has been observed in the current
study that there exists a considerable discrepancy between the perceived Pareto front, given as the output by the algorithm,
compared to the \emph{a posteriori} evaluation of its pre-images, on both model landscapes.
We proposed an explanation for this significant deviation,
stating that the MO-CMA optimally exploits the disturbance distribution and converges to the minimal number of search points
required to fully span the perceived front.
As we demonstrated on the Bi-Sphere case, occasional parental fitness re-evaluation improved the MO-CMA performance and thus constituted
a solution to the problem.

We set up a comparison between the MO-CMA and two conventional EMOA, namely NSGA-II and SMS-EMOA, on the Diffraction Grating test problem.
While the MO-CMA was the clear winner in low search space dimensions, it suffered from slow progress rates in higher-dimensions ($n=80$),
likely due to its self-adaptation mechanism, and required a significant increase in function evaluations in order to converge to the true Pareto front.
In those cases, the SMS-EMOA performed better, and provided a fair approximate front within the original budget of function evaluations.

The application of the MO-CMA to the simulated noisy QC alignment system was successful in terms of revealing the physics
conflict between the investigated objectives, and providing a reliable Pareto front considering the noise-free calculations.
The quality of the Pareto optimal set was questionable, since the perceived front could not be recovered to a satisfactory degree.
Concerning the reference algorithms, both SMS-EMOA and NSGA-II performed poorly in comparison to the MO-CMA, and failed to cover
an important area of the Pareto front.
The results here constitute an example of a scenario where there is clearly no best algorithm for a set of problems,
especially when practical experimental requirements, e.g. a fixed budget of function evaluations, are imposed on the search.
This observation can be considered as a practical interpretation to the so-called \emph{No Free Lunch} (NFL)
theorem (see, e.g., \cite{Wolpert97nofree}).

The laboratory experiments -- the practical climax of this work -- allowed us to examine the proposed algorithmic framework in real-world experimental scenarios.
We assessed the conflict between competing objectives for two experimental quantum systems, and provided interesting Pareto fronts which proved to be reliable with high confidence.
The first experimental case of \emph{molecular ion generation} considered only the default MO-CMA routine, due to instability and laboratory overhead.
The Pareto front in this case could not be recovered upon evaluation of the Pareto optimal set, consistent with the previous
observations of this work on model landscapes.
The second experimental case on the \emph{molecular plasma generation} system was extensively explored by means of various EMOA,
and the results led to important practical conclusions.
The MO-CMA with occasional parental re-evaluation performed best, obtained an excellent pay-off with respect to the standard unshaped reference, and the reproduction of its attained Pareto front was successful.
Examination of its evolving hypervolume revealed the self-correcting property of the re-evaluation scheme, which overall proved to be essential in this experimental scenario.
We therefore conclude that the MO-CMA with occasional re-evaluation, which introduces a basic yet effective extension to an existing EMOA,
constitutes a powerful and reliable routine for \emph{experimental} high-dimensional continuous Pareto optimization.

We would like to propose lines of future work. Given the conclusions drawn
here, the formulation of algorithmic solutions for the MO-CMA is needed. In
addition, sensitivity of auxiliary strategy parameters, including a parameter
that was introduced here (the parental re-evaluation interval) should be
investigated. In a different direction, future research may also incorporate
into multi-objective experimental optimization advanced features that have the
potential to capture various decision making preferences, such as
Pareto-compliant indicators \cite{zbt2007a}, or the enhancement of decision
space diversity \cite{SHIR_EMO2009,UlrichDiversity}.

\section*{Acknowledgments}
The authors acknowledge support from ARO, NSF, ONR, DHS and the Lockheed Martin Corporation.
\appendix
\section{The Diffraction Grating Problem: Analytical Expression of the 2-Dimensional Pareto Front}
\label{app:ParetoProof}
We consider here a specific instantiation of the Diffraction Grating problem, as formulated in Eq.\ \ref{eq:I05Tq} with $b=1,~h=4$.
Let $n\in\mathbb{N}$ and $n\ge 2$, we then define:
\begin{equation}
\begin{array}{l}
\medskip
 \displaystyle \jmath_1 = J_{DG}\left(0, \vec{\varphi}\right) = n+2\cdot \sum_{n>\ell > k\ge 0}\cos\left[\varphi_l-\varphi_k \right]\\
 \displaystyle \jmath_2 = J_{DG}\left(\frac{\pi}{4}, \vec{\varphi}\right) = 
n+2\cdot \sum_{n>\ell > k\ge 0}\cos\left[\pi \left(\ell-k\right)+\varphi_l-\varphi_k \right]
\end{array}
 \end{equation}
Let:
\begin{equation*}
D=\{(x,y)\in\mathbb{R}^2 ~ \backslash ~ \exists\vec\varphi\in[0~2\pi]^n: x=\jmath_1(\vec\varphi) \wedge y=\jmath_2(\vec\varphi)\} 
\end{equation*}

\begin{thm}
The Pareto front $PF(\vec{\jmath})$ of $\vec{\jmath}(\vec{\varphi})=(\jmath_1(\vec{\varphi}),\jmath_2(\vec{\varphi}))^T$ 
for $\vec\varphi\in[0~2\pi]^n$ is
\begin{equation}
PF(\vec{\jmath})=\{(\jmath_1,\jmath_2)^T\in[\delta~n^2]^2 ~ \backslash ~ \jmath_1+\jmath_2=n^2+\delta\}
\end{equation}
with 
\begin{equation*}
\displaystyle \delta = \left\{
\begin{array}{ll}
0 & \textrm{if}\: n=2\imath,~\imath \in\mathbb{N}\\
1 & \textrm{otherwise}
\end{array}
\right.\textrm{}
\end{equation*}
\end{thm}

\begin{proof}
Let us consider $n$ even (i.e., $\delta=0$), the proof for $n$ odd is similar. 
The proof is carried out in two steps:
\begin{itemize}
\item We prove that $D\subset F=\{(x,y)\in[0~n^2]^2 ~ \backslash ~ x+y\le n^2\}$
\item We prove that $\forall (x,y)\in F$ such that 
\begin{equation*}
x+y=n^2 \exists \vec\varphi\in[0~2\pi]^n 
\end{equation*}
with $x=\jmath_1(\vec\varphi)$ and $y=\jmath_2(\vec\varphi)$.
\end{itemize}
First, notice that:
\begin{equation*}
\jmath_1(\vec{\varphi})=\norm{\sum_{k=0}^{n-1}{e^{i\varphi_k}}}^2,~~~
\jmath_2(\vec{\varphi})=\norm{\sum_{k=0}^{n-1}{(-1)^k e^{i\varphi_k}}}^2
\end{equation*}
Hence, $\forall \vec{\varphi}\in[0~2\pi]^n$ $\jmath_1\ge 0$ and $\jmath_2\ge 0$.\\
We start by rewriting the functions $\jmath_1$ and $\jmath_2$ in order to eliminate the $\pi$ factor in the cosine arguments of $\jmath_2$:
\small
 \begin{eqnarray*}
 \jmath_1(\varphi_0,...,\varphi_{n-1})&=&n+2\cdot \sum_{k =0}^{n-2}{\sum_{l = k+1}^{n-1}{\cos\left[\varphi_l-\varphi_k \right]}}\\
 \jmath_2(\varphi_0,...,\varphi_{n-1})&=&n+2\cdot \sum_{k =0}^{n-2}{\sum_{l = k+1}^{n-1}{\cos\left[\pi(l-k)+\varphi_l-\varphi_k \right]}}
 \end{eqnarray*}
 \begin{eqnarray*}
 \jmath_1(\varphi_0,...,\varphi_{n-1})&=&n+2\cdot \sum_{k =0}^{n-2}{\sum_{l = 1}^{n-1-k}{\cos\left[\varphi_{k+l}-\varphi_k \right]}}\\
 \jmath_2(\varphi_0,...,\varphi_{n-1})&=&n+2\cdot \sum_{k =0}^{n-2}{\sum_{l = 1}^{n-1-k}{\cos\left[\pi l+ \varphi_{k+l}-\varphi_k \right]}}
 \end{eqnarray*}
\normalsize
Since $n$ is even and greater than $2$, $\exists m\in\mathbb{N}$ $n=2(m+1)$.
\scriptsize
\begin{equation}
\begin{array}{l}
 \displaystyle \frac{1}{2}\jmath_1(\varphi_0,\ldots,\varphi_{n-1})=(m+1)
+\sum_{p =0}^{m}\sum_{l = 1}^{2m+1-2p}{\cos\left[\varphi_{2p+l}-\varphi_{2p} \right]}
+\sum_{p =0}^{m-1}\sum_{l = 1}^{2m-2p}{\cos\left[\varphi_{2p+1+l}-\varphi_{2p+1} \right]}\\
 \displaystyle \frac{1}{2}\jmath_2(\varphi_0,\ldots,\varphi_{n-1})=(m+1)
+\sum_{p =0}^{m}\sum_{l = 1}^{2m+1-2p}{\cos\left[l\pi+\varphi_{2p+l}-\varphi_{2p} \right]}
+\sum_{p =0}^{m-1}\sum_{l = 1}^{2m-2p}{\cos\left[l\pi+\varphi_{2p+1+l}-\varphi_{2p+1} \right]}
 \end{array}
\end{equation}
\footnotesize
\begin{equation}
\begin{array}{l}
 \displaystyle \frac{1}{2}\jmath_1(\varphi_0,...,\varphi_{n-1})=(m+1)
\sum_{p =0}^{m}\sum_{q = 0}^{m-p}{\cos\left[\varphi_{2p+2q+1}-\varphi_{2p} \right]}
+\sum_{p =0}^{m-1}\sum_{q = 1}^{m-p}{\cos\left[\varphi_{2p+2q}-\varphi_{2p} \right]}\\
\displaystyle ~~~~~+\sum_{p =0}^{m-1}\sum_{q = 0}^{m-p-1}{\cos\left[\varphi_{2p+2+2q}-\varphi_{2p+1} \right]}
+\sum_{p =0}^{m-1}\sum_{q = 1}^{m-p}{\cos\left[\varphi_{2p+1+2q}-\varphi_{2p+1} \right]}\\
\displaystyle \frac{1}{2}\jmath_2(\varphi_0,...,\varphi_{n-1})=(m+1)
-\sum_{p =0}^{m}\sum_{q = 0}^{m-p}{\cos\left[\varphi_{2p+2q+1}-\varphi_{2p} \right]}
+\sum_{p =0}^{m-1}\sum_{q = 1}^{m-p}{\cos\left[\varphi_{2p+2q}-\varphi_{2p} \right]}\\
\displaystyle ~~~~~ -\sum_{p =0}^{m}\sum_{q = 0}^{m-p-1}{\cos\left[\varphi_{2p+2+2q}-\varphi_{2p+1} \right]}
+\sum_{p =0}^{m-1}\sum_{q = 1}^{m-p}{\cos\left[\varphi_{2p+1+2q}-\varphi_{2p+1} \right]}
 \end{array}
\end{equation}
\normalsize
Upon considering all the cosines having values of $\pm 1$, we may write:
$$
D\subset \left[0~n^2\right]^2
$$
Moreover, we have:
\footnotesize
\begin{equation}
 \displaystyle \frac{1}{2}(\jmath_1+\jmath_2)=2(m+1)+2\sum_{p =0}^{m-1}\sum_{q = 1}^{m-p}{\cos\left[\varphi_{2p+2q}-\varphi_{2p} \right]}
+2\sum_{p =0}^{m-1}\sum_{q = 1}^{m-p}{\cos\left[\varphi_{2p+1+2q}-\varphi_{2p+1} \right]}
\end{equation}
\normalsize
which leads to:
\begin{equation}
\jmath_1+\jmath_2\le n^2
\end{equation}

Hence,
 \begin{equation}
 \displaystyle  D\subset F=\left\{\left(x,y\right)\in\left[0~n^2\right]^2 ~ \backslash ~ x+y\le n^2\right\}
 \end{equation}

 In what follows, we shall show that this upper bound is indeed reached:\\
 Given $L=\frac{1}{2}(\jmath_1+\jmath_2)$, it reaches its global maximum if and only if, $\forall p\in[0~m]$ and $l$ such that $2p+2l\le n$ and $2p+1+2l\le n-1$ $\exists k_{ij}, k'_{ij} \in \mathbb{Z}$ such that:
 \begin{eqnarray}
\label{cdt1neven}
\varphi_{2p+2l}&=&\varphi_{2p}+2k_{lp}\pi\\
\label{cdt2neven}
\varphi_{2p+1+2l}&=&\varphi_{2p+1}+2k'_{lp}\pi
\end{eqnarray}
Let us consider $\vec{\varphi}$ satisfying Eq.\ \ref{cdt1neven} and Eq.\ \ref{cdt2neven}:
\begin{eqnarray*}
\jmath_1(\varphi_0,...,\varphi_{n-1})&=&\frac{1}{2}n^2(1+\cos(\varphi_1-\varphi_0))\\
\jmath_2(\varphi_0,...,\varphi_{n-1})&=&\frac{1}{2}n^2(1-\cos(\varphi_1-\varphi_0)),
 \end{eqnarray*}
where $\varphi_1-\varphi_0$ takes any value in $[0~2\pi]$. 
Since $\theta\in[0~2\pi]\rightarrow \cos(\theta)\in[-1~1]$ is a surjective function, we can conclude that for all $(x,y)\in[0~n^2]^2$ such that $x+y=n^2 ~ \exists \vec\varphi\in[0~2\pi]^n$ such that $x=\jmath_1(\vec\varphi)$ and $y=\jmath_2(\vec\varphi)$.
This concludes the proof.
\end{proof}

\section{Diffraction Grating: Noise Propagation}
\label{app:perceivedcalculations}
We provide here explicit calculations of the mean and variance for the perceived objective function of the Diffraction Grating model landscape,
described in Section \ref{sec:physics}.
\subsection{Diffraction Grating: Mean}
Consider the intensity function, $I_{DG}$, presented in Eq.\ \ref{eq:IntensityDG}, which may be written as
\begin{equation}
\label{eq:JDG_0}
\begin{array}{l}
 \medskip
\displaystyle I_{DG}\left(\zeta, \vec{\varphi} \right) = \frac{1}{n^2}\textrm{sinc}^2\left(\frac{\zeta b}{2} \right)
\cdot J_{DG}\left(\zeta, \vec{\varphi} \right) \\
\displaystyle J_{DG}\left(\zeta, \vec{\varphi} \right) = n+ 2\cdot \sum_{\ell > k}\cos\left[\zeta h \left(\ell-k\right)+\Delta\varphi_{\ell k} \right]
\end{array}
\end{equation}
where the compact double-sum notation is used for convenience.
Given a disturbed phase vector, $\tilde{\vec{\varphi}}$, following Eq.\ \ref{eq:varphinoise},
\begin{equation}
\tilde{\vec{\varphi}} = \left( \varphi_0 + \delta\varphi_0, \varphi_1 + \delta\varphi_1, \ldots, \varphi_{n-1} + \delta\varphi_{n-1} \right)^T
\end{equation}
it thus suffices to investigate the propagation of the noise through $J_{DG}$ only:
\begin{equation}
\label{eq:JDG_1}
\displaystyle \tilde{J}_{DG}\left(\zeta, \tilde{\vec{\varphi}} \right) = n+ 2\cdot \sum_{\ell > k}\cos\left[\zeta h \left(\ell-k\right)+\Delta\tilde{\varphi}_{\ell k} \right]
\end{equation}
Note that
\begin{equation}
\begin{array}{l}
 \medskip
\displaystyle \delta\varphi_{\ell} \sim \mathcal{N}\left(0,\epsilon^2 \right)\\
\displaystyle \delta\varphi_{\ell k} \equiv  \delta\varphi_{\ell} - \delta\varphi_{k} \sim \mathcal{N}\left(0,2\epsilon^2 \right)
\end{array}
\end{equation}
Given the probability density function of the normal distribution, denoted as $\Phi(z,\mu,\sigma^2)$, the expectation values
of the cosine and sine functions considering a distribution with zero mean read:
\begin{equation}
\label{eq:normalintegrals1}
 \begin{array}{l}
 \medskip
\displaystyle \int_{-\infty}^{\infty} \Phi(z,0,\sigma^2) \cos(z) dz = \exp \left(-\frac{\sigma^2}{2} \right)\\
 \displaystyle \int_{-\infty}^{\infty} \Phi(z,0,\sigma^2) \sin(z) dz = 0
\end{array}
\end{equation}
Eq.\ \ref{eq:JDG_1} can now be rewritten as:
\begin{equation}
\begin{array}{l}
 \medskip
 \displaystyle \tilde{J}_{DG} = n+ 2\cdot \sum_{\ell > k}\cos\left[\zeta h
\left(\ell-k\right)+\Delta\varphi_{\ell k} + \delta\varphi_{\ell k} \right]=\\
\displaystyle = n+ 2\cdot \sum_{\ell > k}\cos\left(a_{\ell k}\right)\cos\left(\delta\varphi_{\ell k}\right) -
\sin\left(a_{\ell k}\right)\sin\left(\delta\varphi_{\ell k}\right)
\end{array}
\end{equation}
where $a_{\ell k}\equiv \zeta h \left(\ell-k\right) + \Delta\varphi_{\ell k}$.
Upon calculating the expectation values, using Eq.\ \ref{eq:normalintegrals1}, one may write:
\begin{equation}
\label{eq:JDG_2}
\begin{array}{l}
\displaystyle \left< \tilde{J}_{DG} \right>= n+ 2\cdot \sum_{\ell > k}\cos\left(a_{\ell k}\right)\left<\cos\left(\delta\varphi_{\ell k}\right)\right>-
2\cdot \sum_{\ell > k}\sin\left(a_{\ell k}\right)\left<\sin\left(\delta\varphi_{\ell k}\right)\right>=\\
\displaystyle = n+ 2\cdot \sum_{\ell > k}\cos\left(a_{\ell k}\right)\cdot \exp\left(-\epsilon^2\right) 
= n+ 2\cdot\exp\left(-\epsilon^2\right)\cdot \sum_{\ell > k}\cos\left(a_{\ell k}\right),
\end{array}
\end{equation}
concluding with
\begin{equation}
\label{JDGmean}
\boxed{
\displaystyle \left< \tilde{J}_{DG} \right> = n\cdot \left(1- \exp\left(-\epsilon^2\right) \right) + \exp\left(-\epsilon^2\right)\cdot J_{DG}}
\end{equation}
The transition to $\left< \tilde{I}_{DG} \right>$ is trivial with Eq.\ \ref{eq:JDG_0}, yielding the result of Eq.\ \ref{eq:DG_mean}.

\subsection{Diffraction Grating: Variance}
\begin{equation*}
\textrm{VAR}\left[\tilde J_{DG}\right]=\left<{\tilde J_{DG}}^2\right>-\left<\tilde J_{DG}\right>^2
\end{equation*}
From Eq.\ \ref{JDGmean}, $\left<\tilde J_{DG}\right>^2$ is trivial. We now have to compute $\left<{\tilde J_{DG}}^2\right>$. 
In order to do so, let us first compute the mean of this easier term:
\small
\begin{equation*}
\begin{array}{l}
\left(\frac{\tilde J_{DG}-n}{2}\right)^2=\\
\sum_{l_1>k_1}\sum_{l_2>k_2}\left(\cos(a_{l_1k_1})\cos(\delta\varphi_{l_1k_1})-\sin(a_{l_1k_1})\sin(\delta\varphi_{l_1k_1})\right)\\
\left(\cos(a_{l_2k_2})\cos(\delta\varphi_{l_2k_2})-\sin(a_{l_2k_2})\sin\left(\delta\varphi_{l_2k_2}\right)\right)
\end{array}
\end{equation*}
\normalsize
Let (for $j=1,2$)
\small
\begin{equation*}
c_j^a=\cos(a_{l_jk_j})~, s_j^a=\sin(a_{l_jk_j})~,
c_j^d=\cos(\delta\varphi_{l_jk_j}) ~, c_j^d=\cos(\delta\varphi_{l_jk_j})
\end{equation*}
$$
\Gamma_{l_1l_2k_1k_2}=c_1^ac_2^ac_1^dc_2^d+s_1^as_2^as_1^ds_2^d-2c_1^as_2^ac_1^ds_2^d
$$
\normalsize
\begin{equation}
\label{eq:JDGsum}
\displaystyle\left(\frac{\tilde J_{DG}-n}{2}\right)^2=\sum_{l_1>k_1}\sum_{l_2>k_2} \Gamma_{l_1l_2k_1k_2}
\end{equation}

We divide the set $LK=\{(l_1,k_1,l_2,k_2)\in \left[0\ldots n-1\right]^4 / l_1>k_1 \wedge l_2>k_2\}$, to which belong $(l_1,k_1,l_2,k_2)$, 
into the six following subsets which form a partition:
$$
LK=LK_{indpt}\cup LK_{lklk} \cup LK_{l.l.}  \cup LK_{l..l} \cup LK_{.k.k} \cup LK_{.kk.}
$$
Consequently, the sum in Eq.\ \ref{eq:JDGsum} may be divided into six sums, and we note:
\small
\begin{equation*}
\begin{array}{l}
\left(\frac{\tilde J_{DG}-n}{2}\right)^2 = 
\sum_{LK_{indpt}}\Gamma_{l_1l_2k_1k_2}+\sum_{LK_{lklk}}\Gamma_{l_1l_2k_1k_2}+\sum_{LK_{l.l.}}\Gamma_{l_1l_2k_1k_2}\\
+\sum_{LK_{l..l}}\Gamma_{l_1l_2k_1k_2}+\sum_{LK_{.k.k}}\Gamma_{l_1l_2k_1k_2}+\sum_{LK_{.kk.}}\Gamma_{l_1l_2k_1k_2}
\end{array}
\end{equation*}
\normalsize
with
\scriptsize
\begin{equation*}
\begin{array}{l}
LK_{indpt}=\{(l_1,k_1,l_2,k_2)\in \left[0\ldots n-1\right]^4 / l_1>k_1
\wedge l_2>k_2 \wedge l_1\ne l_2 \wedge l_1\ne k_2 \wedge k_1\ne k_2 \wedge k_1\ne l_2\}\\
LK_{lklk}=\{(l_1,k_1,l_2,k_2)\in \left[0\ldots n-1\right]^4 / l_1>k_1 
\wedge l_2>k_2 \wedge l_1= l_2 \wedge k_1= k_2\}  \\
LK_{l.l.}=\{(l_1,k_1,l_2,k_2)\in \left[0\ldots n-1\right]^4 / l_1>k_1 
\wedge l_2>k_2 \wedge l_1= l_2 \wedge k_1\ne k_2 \}\\
LK_{l..l}=\{(l_1,k_1,l_2,k_2)\in \left[0\ldots n-1\right]^4 / l_1>k_1 
\wedge l_2>k_2 \wedge l_1= k_2\}\\
LK_{.k.k}=\{(l_1,k_1,l_2,k_2)\in \left[0\ldots n-1\right]^4 / l_1>k_1 
\wedge l_2>k_2 \wedge k_1= k_2 \wedge l_1\ne l_2 \} \\
LK_{.kk.}=\{(l_1,k_1,l_2,k_2)\in \left[0\ldots n-1\right]^4 / l_1>k_1 
\wedge l_2>k_2 \wedge k_1= l_2 \}
\end{array}
\end{equation*}
\normalsize
Additionally, we note that
\small
\begin{equation*}
\begin{array}{l}
\sharp(LK)=\left(\frac{n(n-1)}{2}\right)^2 \\
\sharp(LK_{lklk})=\frac{n(n-1)}{2}\\
\sharp(LK_{l.l.})=\sharp(LK_{.k.k})= \\
\frac{1}{6}n(n-1)(2n-1)+\frac{1}{2}(3-2n)n(n-1)+n(n-1)(n-2)\\
\sharp(LK_{l..l})=\sharp(LK_{.kk.})=-\frac{1}{6}n(n-1)(2n-1)+\frac{1}{2}n(n-1)^2\\
\sharp(LK_{indpt})=\frac{n(n-1)}{2}\left(\frac{n(n-1)}{2}-2n+3\right)
\end{array}
\end{equation*}
\normalsize
\subsection*{Explicit Summation}
First, consider the following useful results:
\small
\begin{equation*}
\begin{array}{l}
\left<\cos(\delta\varphi_{l})\right>=\exp\left(-\frac{\epsilon^2}{2}\right)\\
\left<\cos(\delta\varphi_{lk})\right>=\exp\left(-\epsilon^2\right)\\
\left<\sin(\delta\varphi_{lk})\right>=0\\
\left<\cos(2\delta\varphi_{lk})\right>=\exp\left(-4\epsilon^2\right)\\
\left<\cos(\delta\varphi_{lk})^2\right>=\frac{1}{2}\left(1+\exp\left(-4\epsilon^2\right)\right)\\
\left<\cos(\delta\varphi_{lk})\sin(\delta\varphi_{lr})\right>=0\\
\left<\sin(\delta\varphi_{lk})^2\right>=\frac{1}{2}\left(1-\exp\left(-4\epsilon^2\right)\right)\\
\left<\cos(\delta\varphi_{lk})\cos(\delta\varphi_{lr})\right>=\frac{1}{2}\left(1+\exp\left(-2\epsilon^2\right)\right)\exp\left(-\epsilon^2\right)\\
\left<\sin(\delta\varphi_{lk})\sin(\delta\varphi_{lr})\right>=\frac{1}{2}\left(1-\exp\left(-2\epsilon^2\right)\right)\exp\left(-\epsilon^2\right)
\end{array}
\end{equation*}
\normalsize
We then have:
\subsubsection{$LK_{indpdt}$}
\scriptsize
\begin{equation}
\label{LKindpdt}
\begin{array}{l}
\medskip
\displaystyle \left<\sum_{LK_{indpdt}}\Gamma_{l_1l_2k_1k_2}\right>=\\
\medskip
\displaystyle \sum_{LK_{indpdt}}{c_1^a c_2^a \left<c_1^d\right> \left<c_2^d\right>+s_1^a s_2^a \left<s_1^d\right> \left<s_2^d\right> - 2 c_1^a s_2^a \left<c_1^d\right> \left<s_2^d\right>}\\
\displaystyle =\exp\left(-2\epsilon^2\right)\sum_{LK_{indpdt}}{\cos(a_{l_1k_1})\cos(a_{l_2k_2})}
\end{array}
\end{equation}
\normalsize

\subsubsection{$LK_{lklk}$}
\scriptsize
\begin{equation}
\label{LKlklk}
\begin{array}{l}
\medskip
\displaystyle \left<\sum_{LK_{lklk}}\Gamma_{l_1l_2k_1k_2}\right> = \\
\medskip
\displaystyle \sum_{LK_{lklk}}{(c^a)^2\left<(c^d)^2\right>+(s^a)^2\left<(s^d)^2\right> - \frac{1}{2} s^{2a}\left<s^{2d}\right>}\\
\displaystyle =\frac{n(n-1)}{4}+\frac{\exp\left(-4\epsilon^2\right)}{2}\sum_{l>k}{\cos(2a_{lk})}
\end{array}
\end{equation}
\normalsize

\subsubsection{$LK_{l.l.}$}
\scriptsize
\begin{equation}
\label{LKlplp}
\begin{array}{l}
\medskip
\displaystyle \left<\sum_{LK_{l.l.}}\Gamma_{l_1l_2k_1k_2}\right> = \\
\displaystyle \sum_{LK_{l.l.}}c_1^a c_2^a \left<\cos(\delta\varphi_{lk})\cos(\delta\varphi_{lr})\right>+s_1^a s_2^a \left<\sin(\delta\varphi_{lk})\sin(\delta\varphi_{lr})\right>\\ 
\medskip
\displaystyle \quad\quad - 2c_1^a s_2^a \left<\cos(\delta\varphi_{lk})\sin(\delta\varphi_{lr})\right>\\
\displaystyle = \frac{1}{2}\exp\left(-\epsilon^2\right)\sum_{LK_{l.l.}}{\cos(a_{lk}-a_{lr})+\exp\left(-2\epsilon^2\right)\cos(a_{lk}+a_{lr})}
\end{array}
\end{equation}
\normalsize
\subsubsection{$LK_{l..l}$}
\scriptsize
\begin{equation}
\label{LKlppl} 
\begin{array}{l}
\medskip
\displaystyle\left<\sum_{LK_{l..l}}\Gamma_{l_1l_2k_1k_2}\right> = \\
\displaystyle \sum_{LK_{l..l}}c_1^a c_2^a \left<\cos(\delta\varphi_{lk})\cos(\delta\varphi_{sl})\right>+s_1^a s_2^a \left<\sin(\delta\varphi_{lk})\sin(\delta\varphi_{sl})\right>\\
\medskip
\displaystyle \quad\quad - 2 c_1^a s_2^a \left<\cos(\delta\varphi_{lk})\sin(\delta\varphi_{sl})\right>\\
\displaystyle=\frac{1}{2}\exp\left(-\epsilon^2\right)\sum_{LK_{l..l}}{\cos(a_{lk}+a_{lr})+\exp\left(-2\epsilon^2\right)\cos(a_{lk}-a_{lr})}
\end{array}
\end{equation}
\normalsize

\subsubsection{$LK_{.k.k}$}
\scriptsize
\begin{equation}
\label{LKpkpk} 
\begin{array}{l}
\medskip
\displaystyle \left<\sum_{LK_{.k.k}}\Gamma_{l_1l_2k_1k_2}\right> = \\
\displaystyle \sum_{LK_{.k.k}}c_1^a c_2^a \left<\cos(\delta\varphi_{lk})\cos(\delta\varphi_{sk})\right>+s_1^a s_2^a \left<\sin(\delta\varphi_{lk})\sin(\delta\varphi_{sk})\right>\\
\medskip
\displaystyle \quad\quad - 2 c_1^a s_2^a \left<\cos(\delta\varphi_{lk})\sin(\delta\varphi_{sk})\right>\\
\displaystyle = \frac{1}{2}\exp\left(-\epsilon^2\right)\sum_{LK_{.k.k}}{\cos(a_{lk}-a_{sk})+\exp\left(-2\epsilon^2\right)\cos(a_{lk}+a_{sk})}
\end{array}
\end{equation}
\normalsize

\subsubsection{$LK_{.kk.}$}
\scriptsize
\begin{equation}
\label{LKpkkp} 
\begin{array}{l}
\medskip
\displaystyle \left<\sum_{LK_{.kk.}}\Gamma_{l_1l_2k_1k_2}\right> = \\
\displaystyle \sum_{LK_{.kk.}}c_1^a c_2^a \left<\cos(\delta\varphi_{lk})\cos(\delta\varphi_{ks})\right>+s_1^a s_2^a \left<\sin(\delta\varphi_{lk})\sin(\delta\varphi_{ks})\right>\\
\medskip
\displaystyle \quad\quad - 2 c_1^a s_2^a \left<\cos(\delta\varphi_{lk})\sin(\delta\varphi_{ks})\right>\\
\displaystyle = \frac{1}{2}\exp\left(-\epsilon^2\right)\sum_{LK_{.kk.}}{\cos(a_{lk}+a_{ks})+\exp\left(-2\epsilon^2\right)\cos(a_{lk}-a_{ks})}
\end{array}
\end{equation}
\normalsize
\subsection*{Conclusion}
From Eqs.\ \ref{LKindpdt}, \ref{LKlklk}, \ref{LKlplp}, \ref{LKlppl}, \ref{LKpkpk}, \ref{LKpkkp} we may write:
\scriptsize
\begin{equation}
\begin{array}{l}
\medskip
\displaystyle \left<\left(\frac{\tilde J_{DG}-n}{2}\right)^2\right>=\frac{n(n-1)}{4}+\frac{\exp\left(-4\epsilon^2\right)}{2}\sum_{l>k}{\cos(2a_{lk})}
+\exp\left(-2\epsilon^2\right)\sum_{LK_{indpdt}}{\cos(a_{lk})\cos(a_{rs})}\\
\displaystyle + \frac{1}{2}\exp\left(-\epsilon^2\right)\left[\sum_{LK_{l.l.}\cup LK_{.k.k}}{\cos(a_{lk}-a_{rs})+\exp\left(-2\epsilon^2\right)\cos(a_{lk}+a_{rs})}\right.\\
\displaystyle \left.
+ \sum_{LK_{l..l}\cup LK_{.kk.}}{\cos(a_{lk}+a_{rs})+\exp\left(-2\epsilon^2\right)\cos(a_{lk}-a_{rs})}\right]
\end{array}
\end{equation}
\normalsize
concluding with:
\scriptsize
\begin{equation}
\label{eq:gratingvarconclusion}
\boxed{
\begin{array}{l}
\medskip
\displaystyle \textrm{VAR}\left[\tilde J_{DG}\right]=n(n-1)\left(1-\exp\left(-2\epsilon^2\right)\right)
-2\exp\left(-2\epsilon^2\right)\left(1-\exp\left(-2\epsilon^2\right)\right)\sum_{l>k}{\cos(2a_{lk})}\\
\displaystyle + 2\exp\left(-\epsilon^2\right)\left(1-\exp\left(-\epsilon^2\right)\right)\left[\sum_{LK_{l.l.}\cup LK_{.k.k}}{\cos(a_{lk}-a_{rs})}
+\sum_{LK_{l..l}\cup LK_{.kk.}}{\cos(a_{lk}+a_{rs})}\right]\\
\displaystyle - 2\exp\left(-2\epsilon^2\right)\left(1-\exp\left(-\epsilon^2\right)\right)\left[\sum_{LK_{l.l.}\cup LK_{.k.k}}{\cos(a_{lk}+a_{rs})}
+\sum_{LK_{l..l}\cup LK_{.kk.}}{\cos(a_{lk}-a_{rs})}\right]
\end{array} }
\end{equation}
\normalsize
\underline{An upper bound on the variance is given by:}
\small
\begin{equation}
\label{eq:gratingvarupperbound}
\displaystyle \textrm{VAR}\left[\tilde J_{DG}\right] \le n(n-1)\left[\left(1-\exp\left(-4\epsilon^2\right)\right)
+2\exp\left(-\epsilon^2\right)\left(1-\exp\left(-2\epsilon^2\right)\right)(n-2)\right]
\end{equation}
\normalsize
For a small $\epsilon$, the bound may be tightened:
\begin{equation}
\label{eq:gratingvarsmall}
\displaystyle \textrm{VAR}\left[\tilde J_{DG}\right] \le 4n(n-1)^2\epsilon^2
\end{equation}
\normalsize
Note that in order to obtain Eq.\ \ref{eq:gratingvarsmall} , all the cosine terms had to be majored by $1$.
Given a point on the screen with destructive interference (the sum and products of the cosines vanish), 
the upper bound in Eq.\ \ref{eq:gratingvarsmall} is strongly superior with respect to the actual variance. 
On the other hand, given a point with constructive interference, the upper bound is a fair estimation of the real variance. 
Notice also that the upper bound of the variance is proportional to the cube of the dimension and the variance of the stochastic noise.

Finally, the transition to $\tilde I_{DG}$ is obtained:
\begin{equation}
\label{eq:gratingvarfinal}
\textrm{VAR}\left[\tilde I_{DG}\right]=\frac{1}{n^4}\textrm{sinc}^4\left(\frac{\zeta b}{2} \right)\cdot \textrm{VAR}\left[\tilde J_{DG}\right]
\end{equation}

\end{document}